\documentclass[nohyperref]{article}

\usepackage{natbib}
\usepackage[top=1in, bottom=1in, left=1in, right=1in]{geometry}
\usepackage{etoc}
\usepackage{microtype}
\usepackage{graphicx}
\usepackage{xfrac}
\usepackage{subfigure}
\usepackage{booktabs}
\usepackage{multirow}
\usepackage{pifont}
\usepackage{wrapfig}
\usepackage{hyperref}
\usepackage{algorithm}
\usepackage{forloop}
\usepackage{tablefootnote}
\usepackage{algpseudocode}
\usepackage{enumitem}
\usepackage{amsmath}
\usepackage{amssymb}
\usepackage{mathtools}
\usepackage{amssymb}
\usepackage{pifont}
\newcommand{\cmark}{\ding{51}}%
\newcommand{\xmark}{\ding{55}}%
\usepackage{amsthm}
\usepackage[capitalize,noabbrev]{cleveref}
\usepackage{listings}
\usepackage{xcolor}
\definecolor{dkgreen}{rgb}{0,0.6,0}
\definecolor{dkred}{rgb}{0.8,0.0,0}
\definecolor{dkblue}{rgb}{0.0,0.0,0.9}
\definecolor{gray}{rgb}{0.5,0.5,0.5}
\definecolor{mauve}{rgb}{0.58,0,0.82}
\definecolor{lightgray}{HTML}{EEEEEE}
\usepackage[font=small,labelfont=bf]{caption}
\usepackage{tikz}
\usepackage{wrapfig}
\usepackage{thm-restate}

\hypersetup{
    colorlinks=true,
    linkcolor=cyan,
    citecolor=cyan,
    filecolor=cyan,
    urlcolor=cyan}
    
\everypar{\looseness=-1}
\setlist[itemize]{leftmargin=*}
\usepackage{amsmath,amsfonts,bm}

\def\<{ {\langle} }
\def\>{ {\rangle} }

\def\eqref#1{equation~\ref{#1}}
\def\1{\bm{1}}

\newcommand{\norm}[1]{\left\lVert#1\right\rVert}

\def\vtheta{{\bm{\theta}}}

\DeclareMathAlphabet{\mathsfit}{\encodingdefault}{\sfdefault}{m}{sl}
\SetMathAlphabet{\mathsfit}{bold}{\encodingdefault}{\sfdefault}{bx}{n}

\DeclareMathOperator*{\argmin}{arg\,min}

\newcommand{\loss}{\mathcal{L}}
\newcommand{\cost}{\mathcal{J}}

\newcommand{\pre}{\mathbf{P}}
\newcommand{\grad}{\nabla}

\newcommand{\hyper}{\boldsymbol{\phi}}
\newcommand{\lambfsd}{\lambda_{\text{FSD}}}
\newcommand{\lambwsd}{\lambda_{\text{WSD}}}

\newcommand{\boldu}{\mathbf{u}}
\newcommand{\boldF}{\mathbf{F}}
\newcommand{\boldI}{\mathbf{I}}
\newcommand{\boldt}{\mathbf{t}}
\newcommand{\boldH}{\mathbf{H}}
\newcommand{\boldS}{\mathbf{S}}

\newcommand{\boldG}{\mathbf{G}}
\newcommand{\boldg}{\mathbf{g}}
\newcommand{\boldP}{\mathbf{P}}

\newcommand{\boldW}{\mathbf{W}}

\newcommand{\boldz}{\mathbf{z}}

\newcommand{\bolds}{\mathbf{s}}

\newcommand{\boldx}{\mathbf{x}}
\newcommand{\boldy}{\mathbf{y}}
\newcommand{\boldA}{\mathbf{A}}
\newcommand{\boldB}{\mathbf{B}}
\newcommand{\boldC}{\mathbf{C}}
\newcommand{\boldD}{\mathbf{D}}
\newcommand{\boldJ}{\mathbf{J}}

\newcommand{\cD}{\mathcal{D}}
\newcommand{\Dtrain}{\mathcal{D}_{\text{train}}}

\newcommand{\bolda}{\mathbf{a}}

\newcommand{\boldxtil}{\tilde{\mathbf{x}}}

\newcommand{\lamwsd}{\lambda_{\text{WSD}}}
\newcommand{\lamfsd}{\lambda_{\text{FSD}}}

\newcommand{\boldtheta}{\boldsymbol{\theta}}

\newcommand{\disf}{D_{\text{F}}}
\newcommand{\disw}{D_{\text{W}}}

\newcommand{\batch}{\mathcal{B}}

\newcommand{\boldphi}{\boldsymbol{\phi}}

\ifdefined\nonewproofenvironments\else
\ifdefined\ispres\else
\newtheorem{theorem}{Theorem}
\newtheorem{corollary}[theorem]{Corollary}

\newtheorem*{theorem*}{Theorem}
\newtheorem*{corollary*}{Corollary}
\renewenvironment{proof}{\noindent\textbf{Proof.}\hspace*{.3em}}{\qed\\}
\newenvironment{proof-sketch}{\noindent\textbf{Proof Sketch}
  \hspace*{1em}}{\qed\bigskip\\}
\newenvironment{proof-idea}{\noindent\textbf{Proof Idea}
  \hspace*{1em}}{\qed\bigskip\\}
\newenvironment{proof-of-lemma}[1][{}]{\noindent\textbf{Proof of Lemma {#1}}
  \hspace*{1em}}{\qed\\}
\newenvironment{proof-of-theorem}[1][{}]{\noindent\textbf{Proof of Theorem {#1}}
  \hspace*{1em}}{\qed\\}
\newenvironment{proof-attempt}{\noindent\textbf{Proof Attempt}
  \hspace*{1em}}{\qed\bigskip\\}

\fi

\title{\vspace{-12mm}Amortized Proximal Optimization}

\newcommand*\samethanks[1][\value{footnote}]{\footnotemark[#1]}
\author{\thanks{Equal Contribution.} Juhan Bae\thanks{University of Toronto and Vector Institute for Artificial Intelligence. \texttt{\{jbae,pvicol,rgrosse\}@cs.toronto.edu}.} ,\,
*Paul Vicol\samethanks[2] ,\,
Jeff Z. HaoChen\samethanks[3]\thanks{Stanford University. \texttt{jhaochen@stanford.edu}.} ,\,
Roger Grosse\samethanks[2]
}

\begin{document}
\etocdepthtag.toc{mtchapter}
\etocsettagdepth{mtchapter}{subsection}
\etocsettagdepth{mtappendix}{none}

\maketitle

\begin{abstract}
We propose a framework for online meta-optimization of parameters that govern optimization, called Amortized Proximal Optimization (APO). We first interpret various existing neural network optimizers as approximate stochastic proximal point methods which trade off the current-batch loss with proximity terms in both function space and weight space. The idea behind APO is to amortize the minimization of the proximal point objective by meta-learning the parameters of an update rule. We show how APO can be used to adapt a learning rate or a structured preconditioning matrix. Under appropriate assumptions, APO can recover existing optimizers such as natural gradient descent and KFAC. It enjoys low computational overhead and avoids expensive and numerically sensitive operations required by some second-order optimizers, such as matrix inverses. We empirically test APO for online adaptation of learning rates and structured preconditioning matrices for regression, image reconstruction, image classification, and natural language translation tasks. Empirically, the learning rate schedules found by APO generally outperform optimal fixed learning rates and are competitive with manually tuned decay schedules. Using APO to adapt a structured preconditioning matrix generally results in optimization performance competitive with second-order methods. Moreover, the absence of matrix inversion provides numerical stability, making it effective for low precision training.
\end{abstract}
\section{Introduction}
\label{sec:introduction}

Many optimization algorithms widely used in machine learning can be seen as approximations to an idealized algorithm called the proximal point method (PPM). 
When training neural networks, the stochastic PPM iteratively minimizes a loss function $\cost_{\batch} \colon \mathbb{R}^m \to \mathbb{R}$ on a mini-batch $\mathcal{B}$, plus a proximity term that penalizes the discrepancy from the current iterate:
\begin{align}
    \vtheta^{(t+1)} \leftarrow \argmin_{\mathbf{u} \in \mathbb{R}^m} \cost_{\batch^{(t)}} (\mathbf{u}) + \lambda D (\mathbf{u}, \vtheta^{(t)}),
    \label{eq:proximal-point-method}
\end{align}
where $D(\cdot, \cdot)$ measures the discrepancy between two vectors and $\lambda > 0$ is a hyperparameter that controls the strength of the proximity term.
The proximity term discourages the update from excessively changing the parameters, hence preventing aggressive updates.
Moreover, the stochastic PPM has good convergence properties~\citep{asi2019stochastic}. While minimizing Eq.~\ref{eq:proximal-point-method} exactly is usually impractical (or at least uneconomical), solving it approximately (by taking first or second-order Taylor series approximations to the loss or the proximity term) has motivated important and widely used optimization algorithms such as natural gradient descent~\citep{amari1998natural} and mirror descent~\citep{beck2003mirror}. Stochastic gradient descent (SGD)~\citep{robbins1951stochastic} itself can be seen as an approximate PPM where the loss term is linearized and the discrepancy function is squared Euclidean distance. 

Inspired by the idea that the PPM is a useful algorithm to approximate, we propose to amortize the minimization of Eq.~\ref{eq:proximal-point-method} by defining a parametric form for an update rule which is likely to be good at minimizing it and adapting its parameters with gradient-based optimization. We consider adapting optimization hyperparameters (such as the learning rate) for existing optimizers such as SGD and RMSprop~\citep{Tieleman2012}, as well as learning structured preconditioning matrices. By choosing a structure for the update rule inspired by existing optimizers, we can take advantage of the insights that went into those optimizers while still being robust to cases where their assumptions (such as the use of linear or quadratic approximations) break down. By doing meta-descent on the optimization parameters, we can amortize the cost of minimizing the PPM objective, which would otherwise take many steps per parameter update.
Hence, we call our approach \emph{Amortized Proximal Optimization (APO)}.

Eq.~\ref{eq:proximal-point-method} leaves a lot of freedom for the proximity term. We argue that many of the most effective neural network optimizers can be seen as trading off two different proximity terms: a \emph{function space discrepancy (FSD)} term which penalizes the average change to the network's predictions, and a \emph{weight space discrepancy (WSD)} term which prevents the weights from moving too far, hence encouraging smoothness to the update and maintaining the accuracy of second-order approximations. Our meta-objective includes both terms.

Our formulation of APO is general, and can be applied to various settings, from optimizing a single optimization hyperparameter to learning a flexible update rule.
We consider two use cases that cover both ends of this spectrum.
At one end, we consider the problem of adapting learning rates of existing optimizers, specifically SGD, RMSprop, and Adam~\citep{duchi2011adaptive}. 
The learning rate is considered one of the most essential hyperparameters to tune~\citep{bengio2012practical}, and good learning rate schedules are often found by years of trial and error. Empirically, the learning rates found by APO outperformed the best fixed learning rates and were competitive with manual step decay schedules.

Our second use case is more ambitious. We use APO to learn a preconditioning matrix, giving the update rule the flexibility to represent second-order optimization updates such as Newton's method, Gauss-Newton, or natural gradient descent. We show that, under certain conditions, the optimum of our APO meta-objective with respect to a full preconditioning matrix coincides with damped versions of natural gradient descent or Gauss-Newton. While computing and storing a full preconditioning matrix for a large neural network is impractical, various practical approximations have been developed. We use APO to meta-learn a structured preconditioning matrix based on the EKFAC optimizer~\citep{george2018fast}. APO is more straightforward to implement in current-day deep learning frameworks than EKFAC and is also more computationally efficient per iteration because it avoids the need to compute eigendecompositions. Empirically, we evaluate APO for learning structured preconditioners on regression, image reconstruction, image classification, and neural machine translation tasks. The preconditioning matrix adapted by APO achieved competitive convergence to other second-order optimizers.
\section{Preliminaries}
\label{sec:preliminaries}

Consider a prediction problem from some input space $\mathcal{X}$ to an output space $\mathcal{T}$.
We are given a finite training set $\mathcal{D}_{\text{train}} = \{(\mathbf{x}^{(i)}, \mathbf{t}^{(i)})\}^{N}_{i=1}$. For a data point $(\mathbf{x}, \mathbf{t})$ and parameters $\vtheta \in \mathbb{R}^m$, let $\mathbf{y} = f(\mathbf{x}, \vtheta)$ be the prediction of a network parameterized by $\vtheta$ and $\mathcal{L} (\mathbf{y}, \mathbf{t})$ be the loss.
Our goal is to minimize the cost function:
\begin{align}
    \cost (\vtheta) = \frac{1}{N} \sum_{i=1}^N \mathcal{L} (f(\mathbf{x}^{(i)}, \vtheta), \mathbf{t}^{(i)}).
    \label{eq:cost}
\end{align}
We use $\cost_{\mathcal{B}} (\vtheta)$ to denote the mean loss on a mini-batch of examples $\mathcal{B} = \{(\mathbf{x}^{(i)}, \mathbf{t}^{(i)})\}_{i=1}^B$. We summarize the notation used in this paper in Appendix~\ref{app:table-of-notation}.

\section{Proximal Optimization and Second-Order Methods: \\A Unifying Framework}
\label{sec:prox-second-order}

We first motivate the proximal objective that we use as the meta-objective for APO, and relate it to existing neural network optimization methods.
Our framework is largely based on \citet{GrosseNNTDChapter4}, to which readers are referred for a more detailed discussion. 

\subsection{Proximal Optimization}
\label{sec:proximal-optimization}

\begin{figure*}[t]
    \centering
    \small
    \includegraphics[width=1\linewidth]{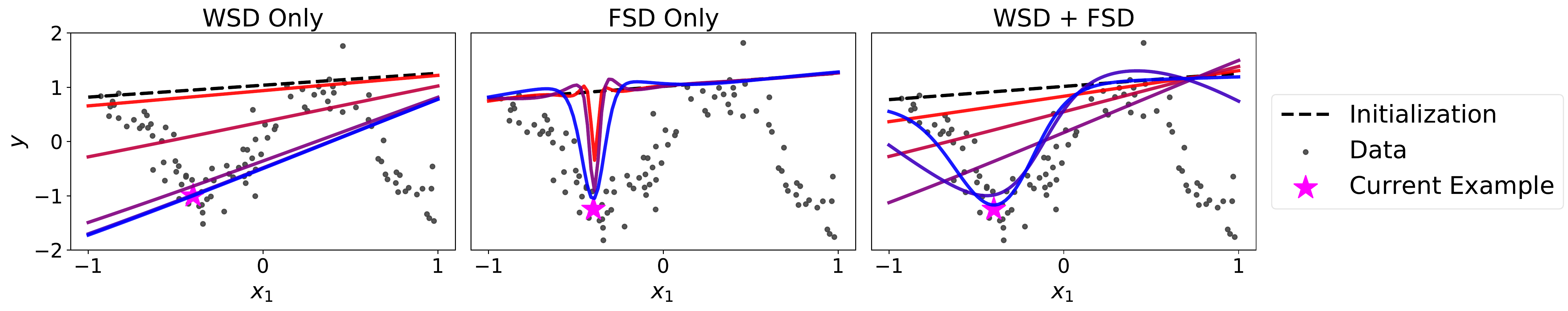}
    \vspace{-0.4cm}
    \caption{1D illustration of the exact proximal update on a regression problem with a batch size of 1, taking inspiration from~\citet{GrosseNNTDChapter4}.
    The weight of the discrepancy term(s) ($\lamwsd$ and $\lamfsd$) is decreased from {\color{red} \textbf{red}} to {\color{blue} \textbf{blue}}.
    }
    \label{fig:1d_regression}
    \vspace{-0.9cm}
\end{figure*}

When we update the parameters on a mini-batch of data, we would like to reduce the loss on that mini-batch, while not changing the predictions on previously visited examples or moving too far in weight space.
This motivates the following proximal point update:
\begin{align}
     \vtheta^{(t+1)} \leftarrow \argmin_{\mathbf{u} \in \mathbb{R}^m} \cost_{\batch^{(t)}} (\mathbf{u}) &+ \lamfsd \mathbb{E}_{\tilde{\mathbf{x}} \sim \mathcal{D}} [ \disf(\mathbf{u}, \vtheta^{(t)}, \tilde{\mathbf{x}})] + \lamwsd \disw ({\mathbf{u}, \vtheta^{(t)}}),
    \label{eq:nn-proximal-point-method}
\end{align}
where $\disf(\mathbf{u}, \vtheta, \mathbf{x}) = \rho (f(\mathbf{x}, \mathbf{u}), f(\mathbf{x}, \boldsymbol{\theta}))$ and $\disw(\mathbf{u}, \vtheta) = \sfrac{1}{2}\|\mathbf{u} - \vtheta\|^2_2$ are discrepancy functions, described in detail below.
Here, $\lamfsd$ and $\lamwsd$ are hyperparameters that control the strength of each discrepancy term, $\tilde{\mathbf{x}}$ is a random data point sampled from the data-generating distribution $\mathcal{D}$, and $\rho(\cdot, \cdot)$ is the output-space divergence. 

The proximal objective in Eq.~\ref{eq:nn-proximal-point-method} consists of three terms.
The first term is the loss on the current mini-batch.
The second term is the \emph{function space discrepancy (FSD)}, whose role is to prevent the update from substantially altering the predictions on other data points.
The general idea of the FSD term has been successful in alleviating catastrophic forgetting~\citep{benjamin2018measuring}, fine-tuning pre-trained models~\citep{jiang2019smart}, and training a student model from a teacher network~\citep{hinton2015distilling}.

The final term is the \emph{weight space discrepancy (WSD)}; this term encourages the update to move the parameters as little as possible. It can be used to motivate damping in the context of second-order optimization~\citep{martens2016second}. While weight space distance may appear counterproductive from an optimization standpoint because it depends on the model parameterization, modern analyses of neural network optimization and generalization often rely on neural network parameters staying close to their initialization in the Euclidean norm~\citep{jacot2018neural,zhang2019fast,belkin2019reconciling}. In fact, \citet{wadia2021whitening} showed that pure second-order optimizers (i.e.~ones without some form of WSD regularization) are unable to generalize in the overparameterized setting.

Figure~\ref{fig:1d_regression} illustrates the effects of the WSD and FSD terms on the exact PPM update for a 1D regression example with a batch size of 1.
If the proximal objective includes only the loss and WSD term (i.e.~$\lamfsd = 0$), the PPM update makes the minimal change to the weights which fits the current example, resulting in a global adjustment to the function which overwrites all the other predictions.
If only the loss and FSD terms are used (i.e.~$\lamwsd = 0$), the update carves a spike around the current data point, failing to improve predictions on nearby examples and running the risk of memorization. When both WSD and FSD are penalized, it makes a local adjustment to the predictions, but one which nonetheless improves performance on nearby examples.

\subsection{Connection Between Proximal Optimization and Second-Order Optimization}
\label{sec:connection}

\begin{wraptable}{r}{0.5\textwidth}
    \vspace{-0.9cm}
    \footnotesize
    \begin{tabular}{@{}lccc@{}}
    \toprule
    \textbf{Method} & \textbf{Loss Approx.} & \textbf{FSD} & \textbf{WSD} \\ \midrule
    \textbf{Gradient Descent}         & $1^{\text{st}}$-order  & - & \cmark \\
    \textbf{Hessian-Free} & $2^{\text{nd}}$-order & - & \cmark \\
    \textbf{Natural Gradient}  & $1^{\text{st}}$-order  & $2^{\text{nd}}$-order  & \xmark  \\
    \textbf{Proximal Point} & Exact & Exact & \cmark \\
    \bottomrule
    \end{tabular}
    \caption{Various classical first- and second-order optimization algorithms can be interpreted as minimizing approximations of the proximal objective in Eq.~\ref{eq:nn-proximal-point-method}, using 1$^{\text{st}}$ or 2$^{\text{nd}}$ order Taylor expansions of the loss or FSD terms.}
    \label{table:second-order-methods}
\end{wraptable}

We further motivate our proximal objective by relating it to existing neural network optimizers. Ordinary SGD can be viewed as an approximate PPM update with a first-order approximation to the loss term and no FSD term. Hessian-free optimization~\citep{martens2010deep}, a classic second-order optimization method for neural networks, approximately minimizes a second-order approximation to the loss on each batch using conjugate gradients. It can be seen as minimizing a quadratic approximation to Eq.~\ref{eq:nn-proximal-point-method} with no FSD term.

\citet{amari1998natural} motivated natural gradient descent (NGD) as a steepest descent method with an infinitesimal step size; this justifies a first-order approximation to the loss term and a second-order approximation to the proximity term. Optimizing over a manifold of probability distributions with KL divergence as the proximity term yields the familiar update involving the Fisher information matrix. Natural gradient optimization of neural networks~\citep{martens2014new,martens2015optimizing} can be interpreted as minimizing Eq.~\ref{eq:nn-proximal-point-method} with a linear approximation to the loss term and a quadratic approximation to the FSD term. While NGD traditionally omits the WSD term in order to achieve parameterization invariance, it is typically included in practical neural network optimizers for stability \citep{martens2015optimizing}. 

In a more general context, when taking a first-order approximation to the loss and a second-order approximation to the FSD, the update rule is given in closed form as:
\begin{align}
    \vtheta^{(t+1)} \approx \vtheta^{(t)} -  (\lambda_{\text{FSD}} \mathbf{G} + \lambda_{\text{WSD}} \mathbf{I})^{-1}  \grad_{\vtheta} \cost_{\mathcal{B}} (\vtheta^{(t)}),
    \label{eq:quadratic-proximal}
\end{align}
where $\mathbf{G}$ is the Hessian of the FSD term. The derivation is shown in Appendix~\ref{app:approximate-proximal-point}. All of these relationships are summarized in Table~\ref{table:second-order-methods}, and derivations of all of these claims are given in Appendix~\ref{app:opt-alg-derivations}.

\section{Amortized Proximal Optimization}
\label{sec:apo}

In this section, we introduce Amortized Proximal Optimization (APO), an approach for online meta-learning of optimization parameters. 
Then, we describe two use cases that we explore in the paper: (1) adapting learning rates of existing base optimizers such as SGD, RMSProp, and Adam, and (2) meta-learning a structured preconditioning matrix.

\begin{figure*}[t]
    \small
    \centering
    \includegraphics[width=0.8\linewidth]{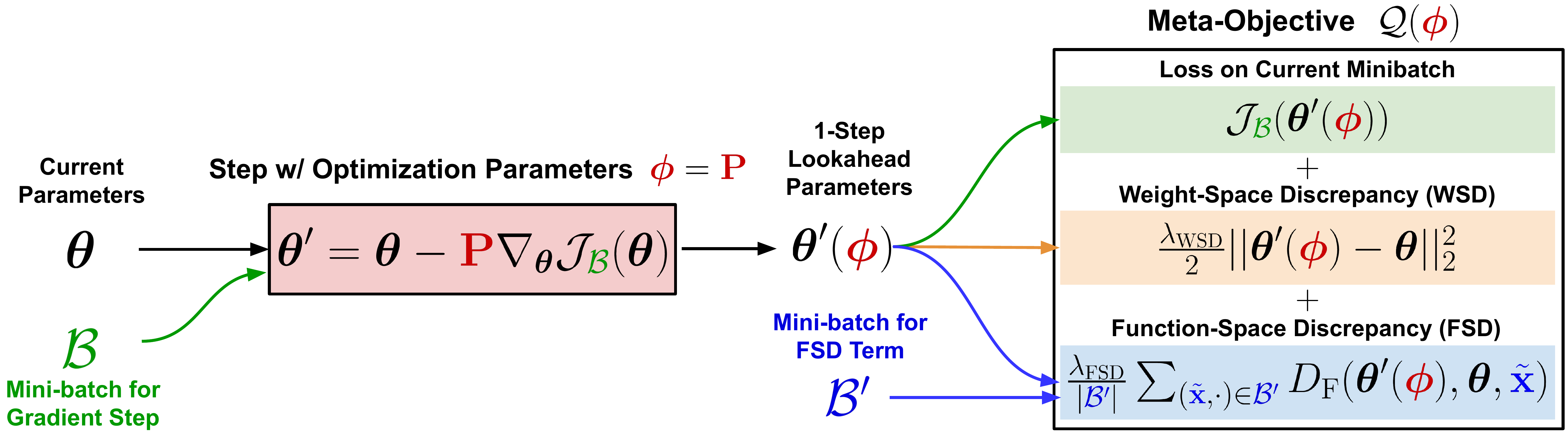}
    \caption{\textbf{Amortized Proximal Optimization (APO)}. 
    In each meta-optimization step, we perform a one-step lookahead from the current parameters $\boldtheta$ to obtain updated parameters $\boldtheta'({\color{dkred}\boldphi})$, where ${\color{dkred} \boldphi}$ denotes the optimization parameters (e.g.~learning rate ${\color{dkred}\eta}$ or preconditioner ${\color{dkred} \boldP}$).
    The meta-objective $\mathcal{Q}({\color{dkred}\boldphi})$ then evaluates the proximal point objective at $\boldtheta'({\color{dkred}\boldphi})$.
    Note that the loss term in $\mathcal{Q}({\color{dkred}\boldphi})$ is computed on the same data that was used to compute the gradient for the lookahead step, ${\color{dkgreen} \batch}$, while the FSD term is computed using a different batch ${\color{dkblue} \batch'}$.
    The optimization parameters ${\color{dkred} \boldphi}$ are updated via the meta-gradient $\nabla_{{\color{dkred} \boldphi}} \mathcal{Q}({\color{dkred} \boldphi})$.
    }
    \label{fig:apo-diagram}
    \vspace{-0.4cm}
\end{figure*}

\setlength{\textfloatsep}{5pt}
\begin{algorithm*}[h]
    \caption{Amortized Proximal Optimization (APO) --- Meta-Learning the Optimization Parameters ${\color{red}\boldphi}$}
    \begin{algorithmic}[l]
    \State \textbf{Require:} $\boldtheta$ (initial model parameters), $\boldphi$ (initial optimization parameters), $K$ (meta-update interval), $\alpha$ (meta-LR)
    \State \textbf{Require:} $\lamwsd$ (weight-space discrepancy term weighting), $\lamfsd$ (function-space discrepancy term weighting)
    \While{not converged, iteration $t$}
      	\State ${\color{dkgreen} \batch} \sim \Dtrain$  \algorithmiccomment{Sample mini-batch to compute the gradient and loss term}
        \If{\ $t \mod K = 0$\ }  \algorithmiccomment{Perform meta-update every $K$ iterations}
        \State ${\color{dkblue}\mathcal{B}'} \sim \Dtrain$ \algorithmiccomment{Sample additional mini-batch to compute the FSD term}
        \State $\boldtheta'({\color{dkred}\boldphi}) := u(\vtheta, {\color{dkred}\boldphi}, {\color{dkgreen} \batch})$  \algorithmiccomment{Compute the 1-step lookahead parameters}
        \State $\mathcal{Q}({\color{dkred}\boldphi}) := \mathcal{J}_{{\color{dkgreen} \batch}} \left(\boldtheta'({\color{dkred}\boldphi}\right)) + \sfrac{\lamfsd}{|{\color{dkblue}\mathcal{B}'}|} \sum_{({\color{dkblue}\tilde{\boldx}}, \cdot) \in {\color{dkblue}\mathcal{B}'}} \disf(\boldtheta'({\color{dkred}\boldphi}), \boldtheta, {\color{dkblue}\tilde{\boldx}}) + \sfrac{\lamwsd}{2} \norm{\boldtheta'({\color{dkred}\boldphi}) - \boldtheta}^2_2$ \algorithmiccomment{Meta-objective}
            \State ${\color{dkred}\boldphi} \gets {\color{dkred}\boldphi} - \alpha \nabla_{{\color{dkred}\boldphi}} \mathcal{Q}({\color{dkred}\boldphi})$ \algorithmiccomment{Update optimizer parameters (e.g.~LR or preconditioner)}
        \EndIf
        \State $\boldtheta \gets u(\boldtheta, {\color{dkred}\boldphi}, {\color{dkgreen}\mathcal{B}})$ \algorithmiccomment{Update model parameters}
    \EndWhile
    \end{algorithmic}
    \label{alg:APO}
\end{algorithm*}

\subsection{Proximal Meta-Optimization}
We assume an update rule $u$ parameterized by a vector $\hyper$ which updates the network weights $\vtheta$ on a batch $\mathcal{B}^{(t)}$:\footnote{The update rule may also depend on state maintained by the optimizer, such as the second moments in RMSprop~\citep{Tieleman2012}. This state is treated as fixed by APO, so we suppress it to avoid clutter.}
\begin{align}
    \vtheta^{(t+1)} \gets u(\vtheta^{(t)}, \hyper, \mathcal{B}^{(t)}),
\end{align}

One use case of APO is to tune the hyperparameters of an existing optimizer, in which case $\hyper$ denotes the hyperparameters.
For example, when tuning the SGD learning rate, we have $\hyper = \eta$ and the update is given by:
\begin{align}
    u_{\text{SGD}} (\vtheta, \eta, \mathcal{B}) &= \vtheta - \eta \grad_{\vtheta} \mathcal{J}_{\batch} (\vtheta).
\end{align}

More ambitiously, we could use APO to adapt a full preconditioning matrix $\mathbf{P}$. In this case, we define $\hyper = \boldP$ and the update is given by:
\begin{align}
    u_{\text{Precond}} (\vtheta, \boldP, \mathcal{B}) &= \vtheta - \boldP \grad_{\vtheta} \mathcal{J}_{\batch} (\vtheta).
\end{align}

In Section~\ref{sec:prox-second-order}, we introduced a general proximal objective for training neural networks and observed that many optimization techniques could be seen as an approximation of PPM. Motivated by this connection, we propose to directly minimize the proximal objective with respect to the optimization parameters. While still being able to take advantage of valuable properties of existing optimizers, direct minimization can be robust to cases when the assumptions (such as linear and quadratic approximation of the cost) do not hold. Another advantage of adapting a parametric update rule is that we can amortize the cost of minimizing the proximal objective throughout training.

We propose to use the following meta-objective, which evaluates the proximal objective at $u(\boldtheta, \boldphi, \mathcal{B})$:
\begin{align}
    \mathcal{Q}(\hyper) = & \mathbb{E}_{\mathcal{B} \sim \mathcal{D}} \Big[ \cost_{\mathcal{B}}(u(\boldtheta, \boldphi, \mathcal{B}))
    +
    \lambda_{\text{FSD}} \mathbb{E}_{(\tilde{\boldx}, \cdot) \sim \mathcal{D}} \left[ \disf ( u(\boldtheta, \boldphi, \mathcal{B}), \vtheta, \tilde{\mathbf{x}}) \right]
    +
    \frac{\lambda_{\text{WSD}}}{2} \|u(\boldtheta, \hyper, \mathcal{B}) -  \vtheta\|^2 \Big].
    \label{eq:meta-objective-exact}
\end{align}
In practice, we estimate the expectations in the meta-objective by sampling two different mini-batches, $\mathcal{B}$ and $\mathcal{B}'$, where $\mathcal{B}$ is used to compute the gradient and the loss term, and $\mathcal{B}'$ is used to compute the FSD term. Intuitively, this proximal meta-objective aims to find optimizer parameters $\boldphi$ that minimize the loss on the current mini-batch, while constraining the size of the step with the FSD and WSD terms so that it does not overfit the current mini-batch and undo progress that has been made by other mini-batches.

The optimization parameters $\boldphi$ are optimized with a stochastic gradient-based algorithm (the \textit{meta-optimizer}). The meta-gradient $\nabla_{\boldphi} \mathcal{Q}(\boldphi)$ can be computed via automatic differentiation through the one-step unrolled computation graph (see Figure~\ref{fig:apo-diagram}). We refer to our framework as Amortized Proximal Optimization (APO). 
An overview of APO is provided in Algorithm~\ref{alg:APO}.

\vspace{-0.1cm}
\subsection{APO for Learning Rate Adaptation}
\vspace{-0.1cm}
One use case of APO is to adapt the learning rate of an existing base optimizer such as SGD.
To do so, we let $u_{\text{SGD}} (\vtheta, \eta, \mathcal{B})$ be the 1-step lookahead of parameters and minimize the proximal meta-objective with respect to the learning rate $\eta$.
Although adaptive optimizers such as RMSProp and Adam use coordinate-wise adaptive learning rates, they still have a global learning rate parameter which is essential to tune.
APO can be applied to such global learning rates to find learning rate schedules (that depend on $\lamfsd$ or $\lamwsd$).

\vspace{-0.1cm}
\subsection{APO for Adaptive Preconditioning}
\vspace{-0.1cm}
More ambitiously, we can use the APO framework to adapt the preconditioning matrix, allowing the update rule to flexibly represent various second-order optimization updates.
We let $u_{\text{Precond}} (\vtheta, \mathbf{P}, \mathcal{B})$ denote the parameters after 1 preconditioned gradient step and adapt the preconditioning matrix $\boldP$ according to our framework.

If the assumptions made when deriving the second-order methods (detailed in Section~\ref{sec:connection}) hold, then the optimal preconditioning matrix of the proximal meta-objective can be equivalent to various second-order updates, depending on the choice of the FSD function.
\begin{theorem}
    Consider an approximation $\hat{\mathcal{Q}}(\pre)$ to the meta-objective (Eq.~\ref{eq:meta-objective-exact}) where the loss term is linearized around the current weights $\boldtheta$ and the FSD term is replaced by its second-order approximation around $\boldtheta$.
    Denote the gradient on a mini-batch as $\boldg = \grad_{\boldtheta} \cost_{\batch}(\boldtheta)$, and assume that the second moment matrix $\mathbb{E}_{\batch \sim \mathcal{D}} \left[ \boldg \boldg^\top \right]$ is non-singular.
    Then, the preconditioning matrix which minimizes $\hat{\mathcal{Q}}$ is given by $\pre^{\star} = (\lamfsd \mathbf{G} + \lamwsd \mathbf{I})^{-1}$, where $\boldG$ denotes the Hessian of the FSD evaluated at $\boldtheta$.
    \label{thm:optimal-precond}
\end{theorem}
The proof is provided in Appendix~\ref{app:precond-opt}. 
As an important special case, when the FSD term is derived from the KL divergence between distributions in output space, $\boldG$ coincides with the Fisher information matrix $\mathbf{F} = \mathbb{E}_{\boldx \sim \cD, \boldy \sim P_{\boldy \mid \boldx}(\boldtheta)}\left[ \nabla_{\boldtheta} \log p(\boldy | \boldx, \boldtheta) \nabla_{\boldtheta} \log p(\boldy | \boldx, \boldtheta)^\top \right]$, where $P_{\boldy \mid \boldx}(\boldtheta)$ denotes the model's predictive distribution over $\boldy$.
Therefore, the optimal preconditioner is the damped natural gradient preconditioner, $\mathbf{P}^{\star} = (\mathbf{F} + \lamwsd \mathbf{I})^{-1}$ when $\lamfsd = 1$. When APO is used to exactly minimize an approximate meta-objective, the update it yields coincides with classical second-order optimization algorithms, depending on the choice of the FSD function.

\subsection{Structured Preconditioner Adaptation}
In the previous sections, the discussion assumed a full preconditioning matrix for simplicity. However, the full preconditioner is impractical to represent for modern neural networks. Moreover, for practical stability of the learned preconditioned update, we would like to enforce the preconditioner to be positive semidefinite (PSD) so that the transformed gradient is a descent direction~\citep{wright1999numerical}.

To satisfy these requirements, we adopt a structured preconditioner analogous to that of the EKFAC optimizer~\citep{george2018fast}. Given a weight matrix $\mathbf{W} \in \mathbb{R}^{m_i \times m_{i+1}}$ of a layer, we construct the preconditioning matrix as a product of smaller matrices:
\begin{align}
    \pre_{\text{S}} = (\mathbf{A} \otimes \mathbf{B}) \text{diag}(\text{vec}(\mathbf{S}))^2 (\mathbf{A} \otimes \mathbf{B})^\top,
    \label{eq:ekfac-param}
\end{align}
where $\mathbf{A} \in \mathbb{R}^{m_{i+1} \times m_{i+1}}$, $\mathbf{B} \in \mathbb{R}^{m_{i} \times m_{i}}$, and $\mathbf{S} \in \mathbb{R}^{m_i \times m_{i+1}}$ are small block matrices. Here, $\otimes$ denotes the Kronecker product, $\text{diag}(\cdot)$ denotes the diagonalization operator, and $\text{vec}(\cdot)$ denotes the vectorization operator.
This parameterization is memory efficient: it requires $m_i^2 + m_{i+1}^2 + m_i m_{i+1}$ parameters to store, as opposed to $m_i^2 m_{i+1}^2$ parameters for the full preconditioning matrix.
It is straightforward to show that the structured preconditioner in Eq.~\ref{eq:ekfac-param} is PSD, as it takes the form $\boldC \boldD \boldC^\top$, where $\boldD$ is PSD. The preconditioned gradient can be computed efficiently by using the properties of the Kronecker product:
\begin{align}
\pre_{\text{S}} \text{vec}(\grad_{\mathbf{W}} &\cost_{\batch} (\vtheta)) = \text{vec}(\mathbf{B} (\mathbf{S}^2 \odot \mathbf{B}^{\top} \grad_{\mathbf{W}} \cost_{\batch} (\vtheta) \mathbf{A}) \mathbf{A}^{\top}),
\end{align}
where $\odot$ denotes elementwise multiplication. This is tractable to compute as it only requires four additional matrix multiplications and elementwise multiplication of small block matrices in each layer when updating the parameters. While EKFAC uses complicated covariance estimation and eigenvalue decomposition to construct the block matrices, in APO, we meta-learn these block matrices directly, where $\boldphi = [\text{vec}(\boldA)^{\top}, \text{vec}(\boldB)^{\top}, \text{vec}(\boldS)^{\top}]^{\top}$. As APO does not require inverting (or performing eigendecompositions of) the block matrices, our structured representation incurs less computation per iteration than KFAC and EKFAC.

While we defined an optimizer with the same functional form as EKFAC, it is not immediately obvious whether the preconditioner which is actually learned by APO will be at all similar. A Corollary of Theorem~\ref{thm:optimal-precond} shows that if certain conditions are satisfied, including the assumptions underlying KFAC~\citep{martens2015optimizing}, then the structured preconditioner minimizing $\hat{\mathcal{Q}}$ coincides with KFAC:
\begin{corollary}
\label{thm:kfac}
Suppose that (1) the assumptions for Theorem~\ref{thm:optimal-precond} are satisfied, (2) the FSD term measures the KL divergence, and (3) $\lamwsd = 0$ and $\lamfsd = 1$.
Moreover, suppose that the parameters $\vtheta$ satisfy the KFAC assumptions listed in Appendix~\ref{app:meta-opt-kfac}.
Then, the optimal solution to the approximate meta-objective $\hat{\mathcal{Q}}$ recovers the KFAC update, which can be represented using the structured preconditioner in Eq.~\ref{eq:ekfac-param}.
\end{corollary}
The proof is provided in Appendix~\ref{app:meta-opt-kfac}. If the KFAC assumptions are not satisfied, then APO will generally learn a different preconditioner. This may be desirable, especially if differing probabilistic assumptions lead to different update rules, as is the case for KFAC applied to convolutional networks~\citep{grosse2016kronecker,laurent2018an}. 

\subsection{Computation and Memory Cost}
\paragraph{Computation Cost.}
Computing the FSD term requires sampling an additional mini-batch from the training set and performing two additional forward passes for $f(\tilde{\boldx}, \boldtheta)$ and $f(\tilde{\boldx}, u(\vtheta, \hyper, \batch))$.
Combined with the loss term evaluated on the original mini-batch, one meta-optimization step requires three forward passes to compute the proximal objective. It additionally requires a backward pass through the 1-step unrolled computation graph (Figure~\ref{fig:apo-diagram}) to compute the gradient of the proximal meta-objective $\mathcal{Q}(\boldphi)$ with respect to $\boldphi$.
This overhead can be reduced by performing a meta-update only once every $K$ iterations: the overhead will consist of $3/K$ additional forward passes and $1/K$ additional backward passes per iteration, which is small for modest values of $K$ (e.g.~$K=10$).

\paragraph{Memory Cost.}
APO requires twice the model memory for the 1-step unrolling when computing the proximal meta-objective.
In the context of structured preconditioner adaptation, we further need to store block matrices $\mathbf{A}$, $\mathbf{B}$, and $\mathbf{S}$ (in Eq.~\ref{eq:ekfac-param}) for each layer, as in KFAC and EKFAC. 

\section{Related Work}
\label{sec:related-work}

\paragraph{Black-Box Hyperparameter Optimization.}
Finding good optimization hyperparameters is a longstanding problem~\citep{bengio2012practical}.
Black-box methods for hyperparameter optimization such as grid search, random search~\citep{bergstra2012random}, and Bayesian optimization~\citep{snoek2012practical,swersky2014freeze,snoek2015scalable} are expensive, as they require performing many complete training runs, and can only find fixed hyperparameter values (e.g.~a constant learning rate).
Hyperband~\citep{li2016hyperband} can reduce the cost by terminating poorly-performing runs early, but is still limited to finding fixed hyperparameters.
Population Based Training (PBT)~\citep{jaderberg2017population} trains a population of networks simultaneously, and throughout training it terminates poorly-performing networks, replaces their weights with a copy of the weights of a better-performing network, perturbs the hyperparameters, and continues training from that point.
PBT can find a coarse-grained learning rate schedule, but because it relies on random search, it is far less efficient than gradient-based meta-optimization.

\vspace{-0.2cm}
\paragraph{Gradient-Based Learning Rate Adaptation.}
\citet{maclaurin2015gradient} backpropagate through the full unrolled training procedure to find optimal learning rate schedules \textit{offline}. This is expensive, as it requires completing an entire training run to make a single hyperparameter update. A related approach is to unroll the computation graph for a small number of steps and perform truncated backpropagation~\citep{domke2012generic,franceschi2017forward}.
Hypergradient descent~\citep{baydin2017online} computes adapts the learning rate to minimize the expected loss in the next iteration.

\vspace{-0.2cm}
\paragraph{Loss-Based Learning Rate Adaptation.}
Several works have proposed modern variants of Polyak step size methods~\citep{hazan2019revisiting,vaswani2019painless,vaswani2020adaptive,loizou2021stochastic} for automatic learning rate adaptation.
In a similar vein, \citet{rolinek2018l4} proposed a method that tracks an estimated minimal loss and tunes a global learning rate to drive the loss on each mini-batch to the minimal loss.

\vspace{-0.2cm}
\paragraph{Short-Horizon Bias.} Short-horizon bias (SHB) has been identified as a core challenge in the meta-optimization of optimizer parameters such as learning rates.
\citet{wu2018understanding} analyzed the setting where the meta-objective is the expected loss after an optimization step $\mathbb{E}_{\batch \sim \Dtrain} [\cost_{\batch}(u(\boldtheta, \boldphi, \batch))]$.
Note that the expectation over the entire dataset is intractable and can be typically approximated using a single mini-batch.
This meta-objective yields undesirable behaviour where the learning rate decreases rapidly to reduce fluctuations of the loss caused by stochastic mini-batch evaluation. While our proximal meta-objective is greedy (e.g. it relies on a single lookahead step), APO empirically does not suffer from SHB.
We attribute this to the fact that the first term in our proximal meta-objective evaluates the loss on the \textit{same mini-batch used to compute the gradient} rather than a randomly-selected mini-batch.
We refer readers to Appendix~\ref{app:same-minibatch-first-term} for an ablation demonstrating the difference between using the same and a different mini-batch in computing the first term in our meta-objective. 
% We indeed found that using a random mini-batch in our objective leads to the same phenomenon observed by~\citet{wu2018understanding}.

\vspace{-0.2cm}
\paragraph{Forward-Mode Differentiation.}
In principle, short horizon meta-objectives cannot model long-term behavior.
Forward-mode autodiff like Real-Time Recurrent Learning (RTRL) and its approximations (UORO~\citep{tallec2017unbiased}, KF-RTRL~\citep{mujika2018approximating}, OK~\citep{benzing2019optimal}) can perform online optimization without suffering from SHB in principle (but still suffering from hysteresis).
However, they do not scale well to many parameters, and thus cannot be used to adapt preconditioner.
Other approaches like PES~\citep{vicol2021unbiased} use a finite-differences approximation to the gradient, which also does not scale to the dimensionality of the preconditioner.
RTRL has been applied to hyperparameter adaptation by~\citet{franceschi2017forward}, and a method for LR adaptation that interpolates between hypergradient descent and RTRL, called MARTHE, was introduced by~\citet{donini2019marthe}.

\vspace{-0.2cm}
\paragraph{Second-Order Optimization.}
Although preconditioning methods theoretically have significantly better convergence rates than first-order methods~\citep{becker1988improving,wright1999numerical}, storing and computing the inverse of large preconditioning matrices is often impractical for high-dimensional problems. To mitigate these computational issues, Hessian-free optimization~\citep{martens2010deep,martens2011learning} approximates Newton's update by only accessing the curvature matrix through Hessian-vector products. 
Other works impose a structure on the preconditioning matrix by representing it as a Kronecker product~\citep{martens2015optimizing,grosse2016kronecker,martens2018kronecker,george2018fast,gupta2018shampoo,tang2021skfac}, a diagonal matrix~\citep{duchi2011adaptive,kingma2014adam}, or a low-rank matrix~\citep{le2007topmoumoute,mishkin2018slang}. However, these approximate second-order methods may not be straightforward to implement in modern deep learning frameworks, and can still be computationally expensive as they often require matrix inversion or eigendecomposition.

\vspace{-0.2cm}
\paragraph{Gradient-Based Preconditioner Adaptation.} There has been some prior work on meta-learning the preconditioning matrix. \citet{moskovitz2019first} learn the preconditioning matrix with hypergradient descent. Meta-curvature~\citep{park2019meta} and warped gradient descent~\citep{lee2018gradient,flennerhag2019meta} adapt the preconditioning matrix that yields effective parameter updates across diverse tasks in the context of few-shot learning.

\vspace{-0.2cm}
\paragraph{Learned Optimization.}
Some authors have proposed learning entire optimization algorithms~\citep{li2016learning,li2017learning,andrychowicz2016learning,wichrowska2017learned,metz2018learned,wichrowska2017learned}. \citet{li2016learning} view this problem from a reinforcement learning perspective, where the state consists of the objective function $\cost$ and the sequence of prior iterates $\{ \vtheta^{(t)} \}$ and gradients $\{ \nabla_\vtheta \cost(\vtheta^{(t)}) \}$, and the action is the step $\Delta \vtheta$. In this setting, the update rule $\alpha$ is a policy that can be found via policy gradient methods~\citep{sutton2000policy}.
Approaches that learn optimizers must be trained on a \textit{set} of objective functions $\{ f_1, \dots, f_n \}$ drawn from a distribution $\mathcal{F}$.
This setup can be restrictive if we only have one instance of an objective function.
In addition, the initial phase of training the optimizer on a distribution of functions can be expensive. Learned optimizers require offline training, and typically must use long unrolls to ameliorate short horizon bias, making training expensive.
APO requires only the objective function of interest as it does not have trainable parameters.
\section{Experiments}
\label{sec:experiments}

Our experiments investigate the following questions: (1) How does the structured preconditioning matrix adapted by APO perform in comparison to existing first- and second-order optimizers? (2) How does the learning rate adapted by APO perform compared to optimal fixed learning rates and manual decay schedules commonly used in the literature? 

We used APO to meta-learn the preconditioning matrices for a broad range of tasks, including several regression datasets, autoencoder training, image classification on CIFAR-10 and CIFAR-100 using several network architectures, neural machine translation using transformers, and low-precision (16-bit) training. Several of these tasks are particularly challenging for first-order optimizers. In addition, we used APO to tune the global learning rate for multiple base optimizers -- SGD, SGD with momentum (denoted SGDm), RMSprop, and Adam -- on CIFAR-10 and CIFAR-100 classification tasks with several network architectures.

We denote our method that adapts the structured preconditioning matrix as ``APO-Precond''. The method that tunes the global learning rate of a base optimizer is denoted as ``Base-APO'' (e.g. SGDm-APO). Experiment details and additional experiments, including ablation studies, are given in Appendix~\ref{app:experiment-details} and ~\ref{app:labmda-ablation}, respectively.

\subsection{Meta-Optimization Setup}
\label{sec:meta-opt-setup}

In all experiments, we perform 1 meta-update every 10 updates of the base optimizer (e.g.~$K = 10$ in Algorithm~\ref{alg:APO}). Hence, the computational overhead of APO is just a small fraction of the original training procedure. We perform grid searches over $\lambfsd$ and $\lambwsd$ for both learning rate and preconditioner adaptation.

\paragraph{Learning Rate Adaptation Setup.} For our learning rate adaptation experiments, we used RMSProp as the meta-optimizer with a meta-learning rate of 0.1. For SGD-APO and SGDm-APO, we set the initial learning rate to 0.1, while for RMSProp-APO and Adam-APO, we set the initial learning rate to $3\text{e-}4$. These specific values are not important for good performance; we show in Appendix~\ref{app:robustmeta} that APO is robust to the choice of initial learning rate.

\paragraph{Preconditioner Adaptation Setup.} For preconditioner adaption experiments, we used Adam as the meta-optimizer with a meta-learning rate chosen from $\{ 1\text{e-}4,  3\text{e-}5 \}$. We initialized the block matrices such that the preconditioner is the identity, so that our update is equivalent to SGD at initialization. Furthermore, we used a warm-up phase of 3000 iterations at the beginning of training that updates the parameters with SGDm (or Adam for Transformer) but still meta-learns the preconditioning matrix (e.g.~performing meta-gradient steps on $\boldP$). When updating the parameters, we scaled the preconditioned gradient by a fixed constant of 0.9 for stability.

\subsection{Meta-Learning Preconditioners}

\subsubsection{Regression}
\paragraph{Poorly-Conditioned Regression.} We first considered a regression task traditionally used to illustrate failures of neural network optimization~\citep{rahimi_recht_2017}. The targets are given by $\mathbf{t} = \mathbf{A} \mathbf{x}$, where $\mathbf{A}$ is an ill-conditioned matrix with $\kappa(\mathbf{A}) = 10^{10}$.
We trained a two-layer linear network $f(\mathbf{x}, \vtheta) = \mathbf{W}_2 \mathbf{W}_1 \mathbf{x}$ and minimized the following objective:
\begin{align}
    \cost (\vtheta) = \mathbb{E}_{\mathbf{x} \sim \mathcal{N}(\mathbf{0}, \mathbf{I})} \left[ \| \mathbf{A} \mathbf{x} - \mathbf{W}_2 \mathbf{W}_1 \mathbf{x} \|^2 \right].
\end{align}
In Figure~\ref{fig:challenging-opt} (left), we show the training curves for SGDm, Adam, Shampoo~\citep{gupta2018shampoo}, KFAC, and APO-Precond. As the problem is ill-conditioned, second-order optimizers such as Shampoo and KFAC converge faster than first-order methods. APO-Precond shows performance comparable with second-order optimizers and achieves lower loss than KFAC.

\paragraph{UCI Tasks.}
We further validated APO-Precond on the Slice, Protein, and Parkinsons datasets from the UCI collection~\citep{Dua:2019}.
We trained a 2-layer multilayer perceptron with 100 hidden units per layer and ReLU activations for 500 epochs. The training curves for each optimizer are shown in Figure~\ref{fig:uci-regression}.
By automatically adjusting the preconditioning matrix during training, APO-Precond consistently achieved competitive convergence compared to other second-order optimizers and reached lower training loss than all baselines.

\begin{figure*}[!h]
    \vspace{-0.4cm}
    \centering
    \begin{tabular}{ccc}
    \includegraphics[width=0.31\linewidth]{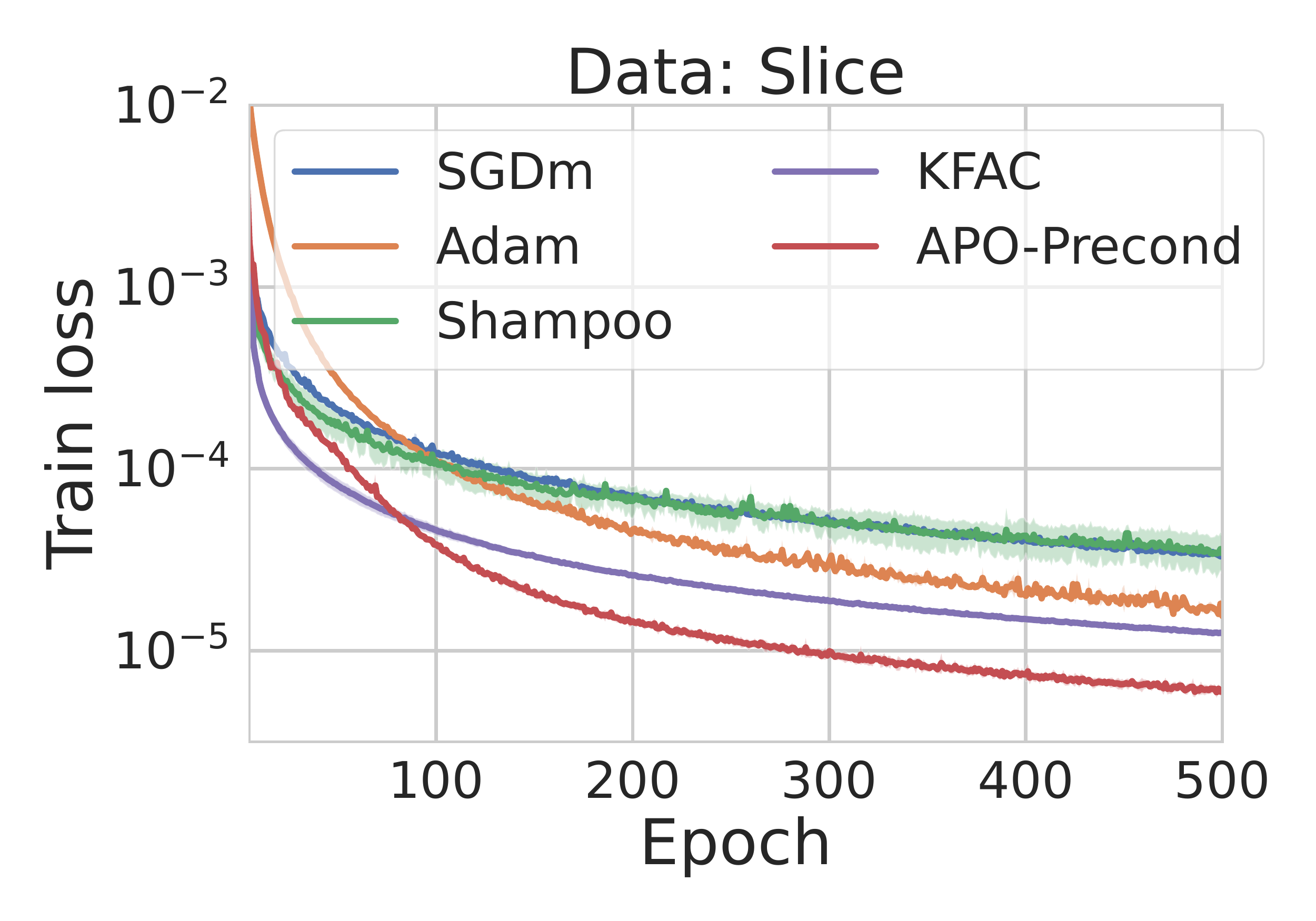}
    \includegraphics[width=0.31\linewidth]{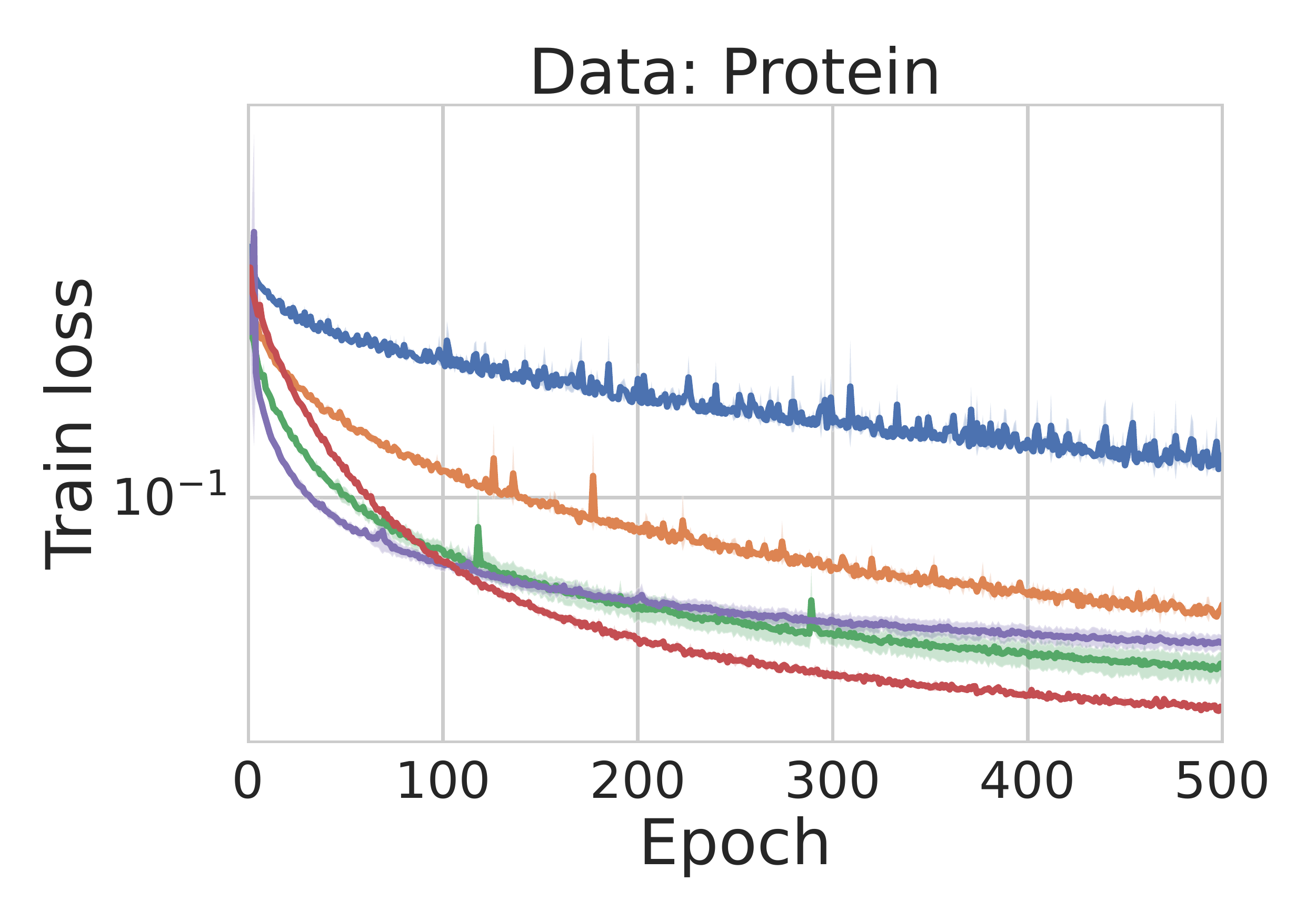}
    \includegraphics[width=0.31\linewidth]{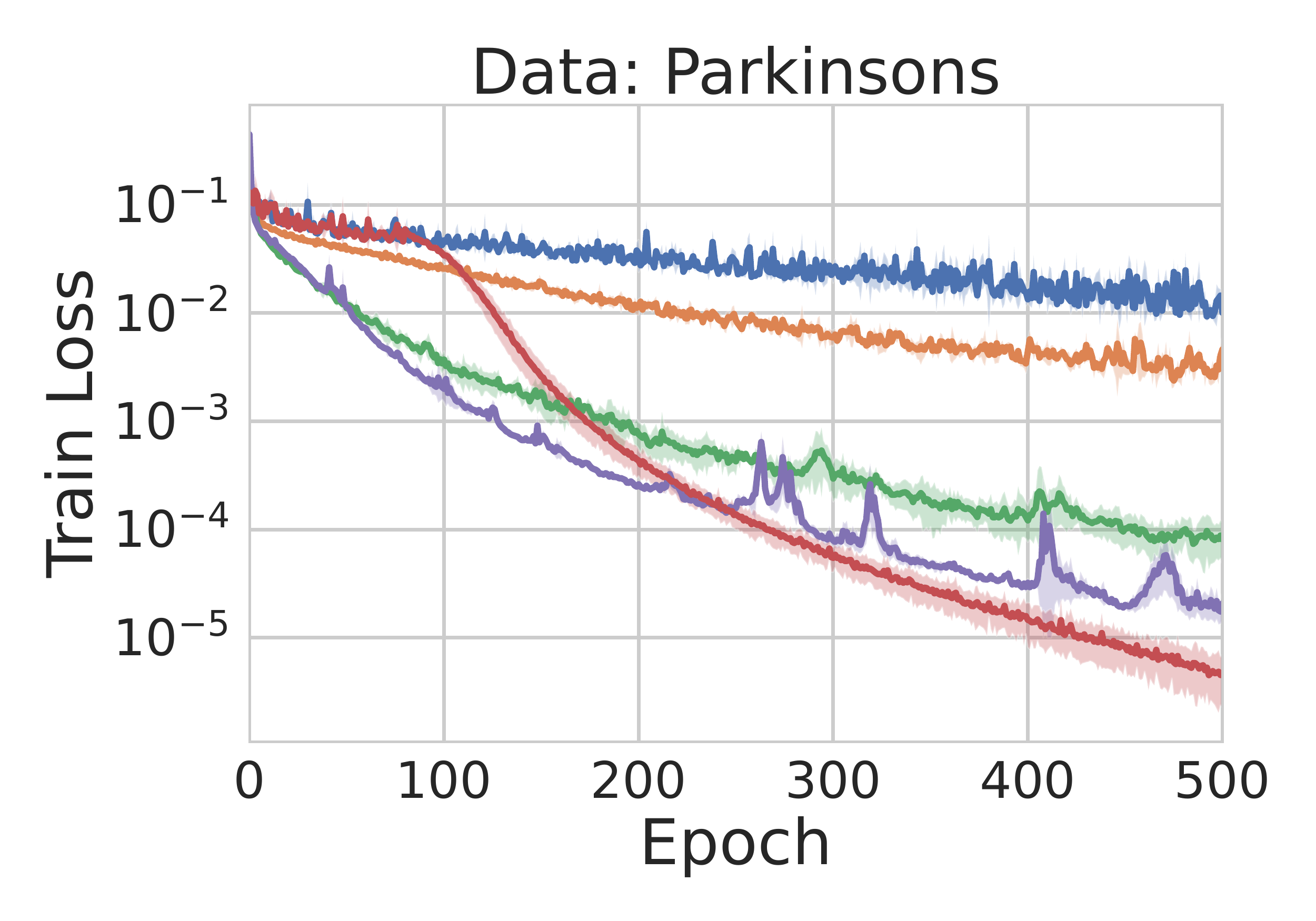}
    \end{tabular}
    \vspace{-4mm}
    \caption{A comparison of SGDm, Adam, KFAC, Shampoo, and APO-Precond on UCI regression tasks. Across all tasks, APO-Precond achieves lower loss with competitive convergence compared to second-order optimizers.}
    \vspace{-0.4cm}
    \label{fig:uci-regression}
\end{figure*}

\begin{figure*}[!htb]
    \centering
    \begin{tabular}{ccc}
    \includegraphics[width=0.35\linewidth]{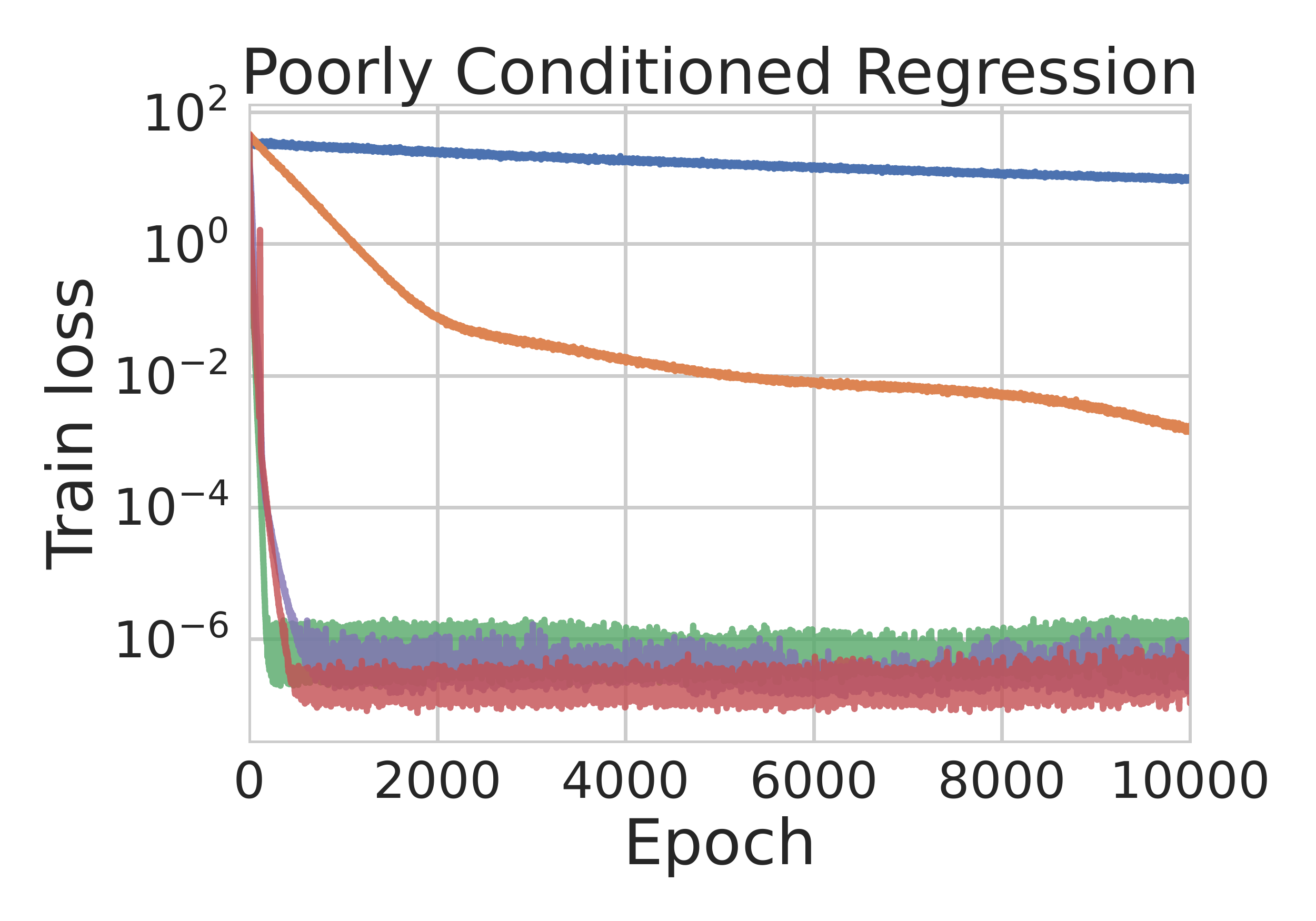}
    \quad
    \includegraphics[width=0.35\linewidth]{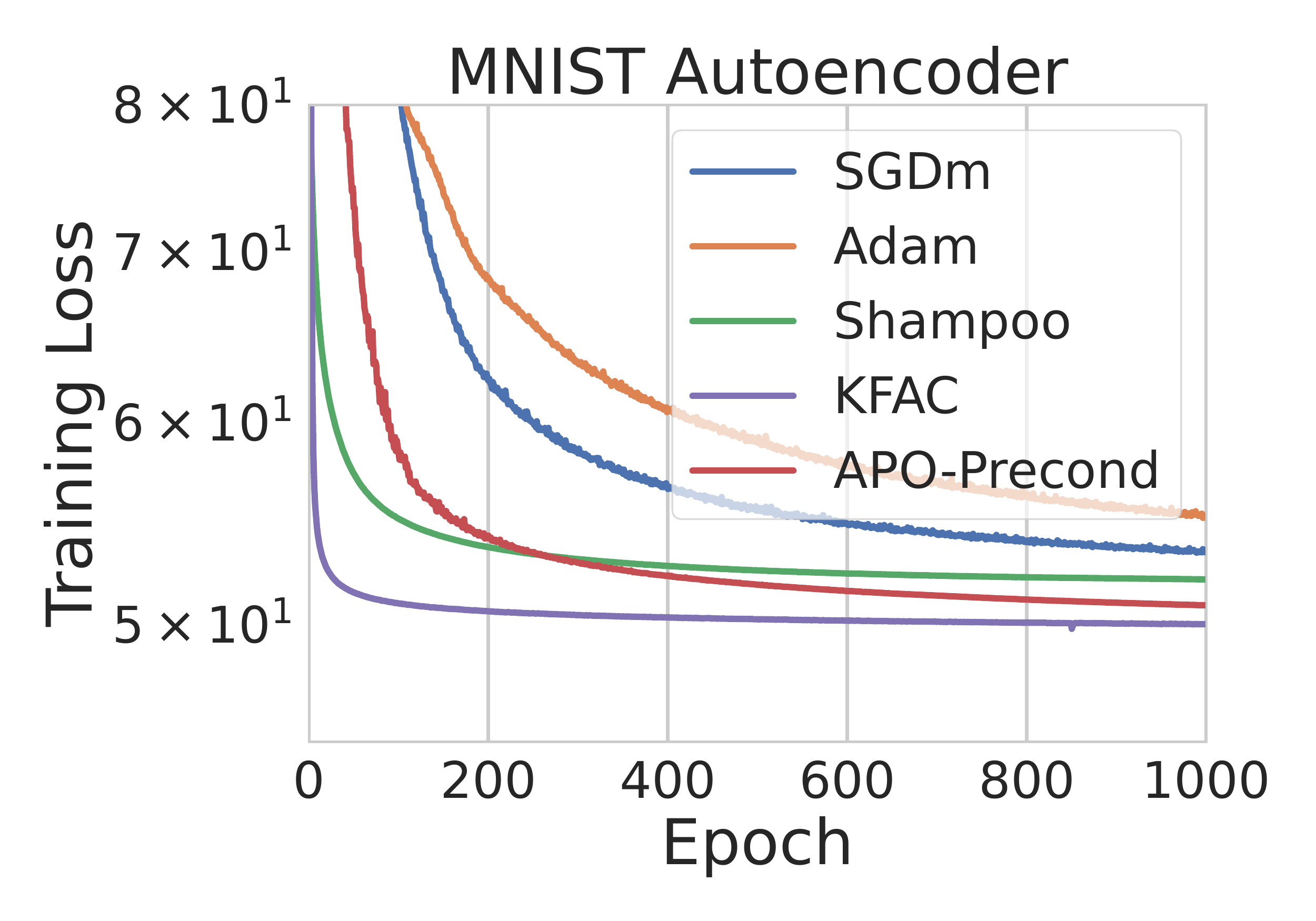}
    \end{tabular}
    \vspace{-4mm}
    \caption{
    A comparison of various optimizers on optimization tasks known to be difficult for first-order methods. \textbf{Left:} synthetic data for poorly-conditioned regression; \textbf{Right:} deep autoencoder on MNIST.
    }
    \vspace{-1.2cm}
    \label{fig:challenging-opt}
\end{figure*}

% \vspace{-0.1cm}
\subsubsection{Image Reconstruction}
% \vspace{-0.1cm}
We trained an 8-layer autoencoder on MNIST~\citep{lecun1998mnist}; this is known to be a challenging optimization task for first-order optimizers~\citep{hinton2006reducing,martens2011learning,martens2015optimizing}. We followed the experimental set-up from~\citet{martens2015optimizing}, where the encoder and decoder consist of 4 fully-connected layers with sigmoid activation.
The decoder structure is symmetric to that of the encoder, and they do not have tied weights.
The logistic activation function and the presence of a bottleneck layer make this a challenging optimization problem compared with most current-day architectures. We compare APO-Precond with SGDm, Adam, Shampoo, and KFAC optimizers and show the training losses for each optimizer in Figure~\ref{fig:challenging-opt} (right).
APO-Precond converges faster than first-order methods and achieves competitive training loss to other second-order methods (although there remains a performance gap compared with KFAC).

% VERSION OF TABLE 2 WITHOUT SGD, RMSPROP, ADAMW, EKFAC
% AND WITH A DIFFERENT WD RANGE
% =====================================================
\begin{table}[t]
  \setlength{\tabcolsep}{4pt}
  \centering
  \small
  \begin{tabular}{ccccccc}
  \toprule
  \textbf{Task} & \textbf{Model} & \textbf{SGDm} & \textbf{Adam} & \textbf{KFAC} & \textbf{APO-Precond}  \\ \midrule
  \textbf{CIFAR-10}   & \textbf{LeNet}     & 75.73 & 73.41 & 76.63 &  \textbf{77.42}  \\
  \textbf{CIFAR-10}   & \textbf{AlexNet}   & 76.27 & 76.09  & 78.33 & \textbf{81.14}  \\
  \textbf{CIFAR-10}   & \textbf{VGG16}  & 91.82 & 90.19  & 92.05 &  \textbf{92.13}  \\
  \textbf{CIFAR-10}   & \textbf{ResNet-18}  & 93.69 & 93.27  & 94.60 & \textbf{94.75}   \\
  \textbf{CIFAR-10}   & \textbf{ResNet-32}  & 94.40 & 93.30   & 94.49 & \textbf{94.83} \\
  \midrule
  \textbf{CIFAR-100}  & \textbf{AlexNet}   & 43.95 & 41.82  & 46.24 & \textbf{52.35} \\
  \textbf{CIFAR-100}  & \textbf{VGG16}  & 65.98 & 60.61 & 61.84 & \textbf{67.95}  \\
  \textbf{CIFAR-100}  & \textbf{ResNet-18}  & 76.85 & 70.87 & 76.48 & \textbf{76.88}   \\
  \textbf{CIFAR-100}  & \textbf{ResNet-32}  & \textbf{77.47} & 68.67 & 75.70 & 77.41   \\
  \midrule
  \textbf{SVHN}   & \textbf{ResNet-18}  & 96.19 & 95.59  & 96.08 & \textbf{96.89}   \\
  \midrule
  \textbf{IWSLT14} & \textbf{Transformer} & 31.43  & 34.60 \tablefootnote{We used AdamW optimizer~\citep{loshchilov2017decoupled} for training Transformer model.} &  -  & \textbf{34.62} \\
  \bottomrule
  \end{tabular}
  \caption{Test accuracy on CIFAR-10 and CIFAR-100, and BLEU score on the IWSLT'14 German-to-English translation dataset for various optimizers.}
  \label{table:testaccuracy}
\end{table}

\subsubsection{Image Classification}

To investigate whether adapting the preconditioner with APO leads to better generalization, we conducted image classification experiments on CIFAR-10 and CIFAR-100. We trained LeNet~\citep{lecun1998mnist}, AlexNet~\citep{krizhevsky2012imagenet}, VGG-16~\citep{simonyan2014very} (w/o batch normalization~\citep{ioffe2015batch}), ResNet-18, and ResNet-32~\citep{he2016deep} architectures for 200 epochs with batches of 128 images. For all optimizers, we decayed the learning rate by a factor of 5 at epochs 60, 120, and 160, following common practice from~\citet{zagoruyko2016wide}.
The test accuracies for SGDm, Adam, KFAC, and APO-Precond are summarized in Table~\ref{table:testaccuracy}.
By adapting the preconditioning matrix with our proximal meta-objective, APO-Precond achieved competitive generalization performance to SGDm.
Moreover, for architectures without batch normalization (e.g.~LeNet and AlexNet), APO-Precond improved the final test accuracy significantly. 

\subsubsection{Neural Machine Translation}

To verify the effectiveness of APO on various tasks, we applied APO-Precond on the IWSLT'14 German-to-English translation task~\citep{cettolo2014report}.
We used a Transformer~\citep{vaswani2017attention} composed of 6 encoder and decoder layers with a word embedding and hidden vector dimension of 512.
We compared APO-Precond with SGDm and AdamW~\citep{loshchilov2017decoupled}.
For AdamW, we used a warmup-then-decay learning rate schedule widely used in practice, and for SGD and APO-Precond, we kept the learning rate fixed after the warmup.
In Table~\ref{table:testaccuracy}, we show the final test BLEU score for SGDm, AdamW, and APO-Precond. While keeping the learning rate fixed, APO-Precond achieved a BLEU score competitive with AdamW.

\subsubsection{Low Precision Training}
\label{sec:low-prec-network}

\begin{wraptable}[6]{r}{0.55\textwidth}
    \vspace{-0.9cm}
    \centering
    \footnotesize
    \begin{tabular}{ccccc}
    \toprule
     \textbf{Task} & \textbf{Model} & \textbf{SGDm} & \textbf{KFAC} & \textbf{APO-P}  \\ \midrule
    \textbf{CIFAR-10}   & \textbf{LeNet}  & 75.65  & 74.95 & \textbf{77.25}  \\
    \textbf{CIFAR-10}   & \textbf{ResNet-18}  & 94.15  & 92.72 &  \textbf{94.79}  \\
    \textbf{CIFAR-100}  & \textbf{ResNet-18}  & 73.53
     & 73.12 & \textbf{75.47}  \\
    \bottomrule
    \end{tabular}
    \vspace{-0.1cm}
    \caption{Test accuracy of 16-bit networks on CIFAR-10/100.}
    \label{table:low-precision}
\end{wraptable}

Low precision training presents a challenge for second-order optimizers such as KFAC which rely on matrix inverses that may be sensitive to quantization noise. We trained LeNet and ResNet-18 with a 16-bit floating-point arithmetic to examine if APO-Precond is applicable in training the networks in lower precision. We used the same experimental setup from the image classification experiments but stored parameters, activations, and gradients in 16-bit floating-point. We found that KFAC sometimes required a large damping in order to maintain stability, and this prevented it from fully utilizing the curvature information. By contrast, as APO-Precond does not require matrix inversion, it remained stable with the same choice of FSD and WSD weights we used in the full precision experiments. The final test accuracies on ResNet-18 for SGDm, KFAC, and APO-Precond are shown in Table~\ref{table:low-precision}. APO-Precond achieved the highest accuracy on both CIFAR-10 and CIFAR-100, whereas KFAC had a $\sim 2\%$ decrease in accuracy. The training curves for ResNet-18 are provided in Appendix~\ref{app:16-bit}.

\begin{wrapfigure}[12]{r}{0.4\linewidth}
\centering
\vspace{-1.8cm}
\includegraphics[width=\linewidth]{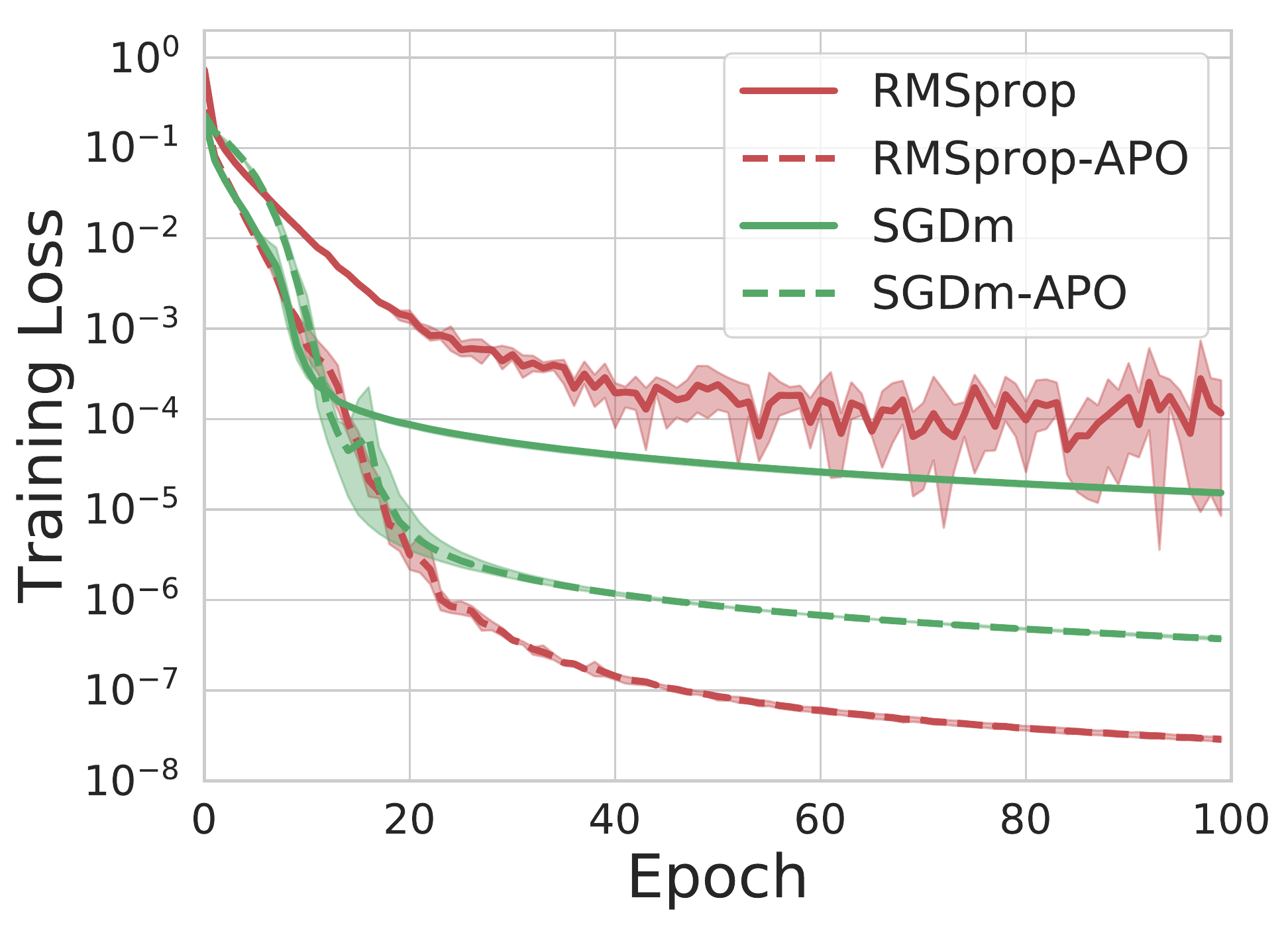}
\vspace{-0.7cm}
\caption{Training loss for a two-layer MLP trained on MNIST, using SGDm and RMSprop (solid lines), and their APO-tuned variants (dashed lines).}
\label{fig:mlp-mnist}
\end{wrapfigure}

\subsection{Meta-Learning Learning Rates}
\paragraph{MNIST.} First, we compared SGDm and RMSprop to their APO-tuned variants to train an MLP on MNIST. We used a two-layer MLP with 1000 hidden units per layer and ReLU activation, and trained on mini-batches of size 100 for 100 epochs. Figure~\ref{fig:mlp-mnist} shows the training loss achieved by each approach; we found that for both base optimizers, APO improved convergence speed and obtained substantially lower loss than the baselines.

\paragraph{CIFAR-10 \& CIFAR-100.}
For learning rate adaptation on CIFAR-10, we experimented with three network architectures: ResNet32~\citep{he2016deep}, ResNet34, and WideResNet (WRN-28-10)~\citep{zagoruyko2016wide}. For ResNet32, we trained for 400 epochs, and the decayed baseline used a step schedule with $10\times$ decay at epochs 150 and 250, following~\citet{lucas2018aggregated}.
For ResNet34 and WRN-28-10, we trained for 200 epochs, and the decayed baseline used a step schedule with $5\times$ decay at epochs 60, 120, and 160, following~\citet{zagoruyko2016wide}.
For CIFAR-100, we used WRN-28-10 with the same schedule as for CIFAR-10.
For each of the base optimizers, we compared APO to (1) the optimal fixed learning rate and (2) a manual step learning rate decay schedule. The test accuracies for each base optimizer and their APO-tuned variants are shown in Table~\ref{table:lr-accuracies}.
In addition, Figure~\ref{fig:wrn-cifar10-sgdm} shows the training loss, test accuracy and learning rate adaptation for WRN-28-10 on CIFAR-10, using SGDm as the base optimizer.
We found that using APO to tune the global learning rate yields higher test accuracy (or lower training loss) than the best fixed learning rate, and is competitive with the manual decay schedule.

\begin{figure*}[t]
\centering
\begin{tabular}{ccc}
\includegraphics[width=0.3\linewidth]{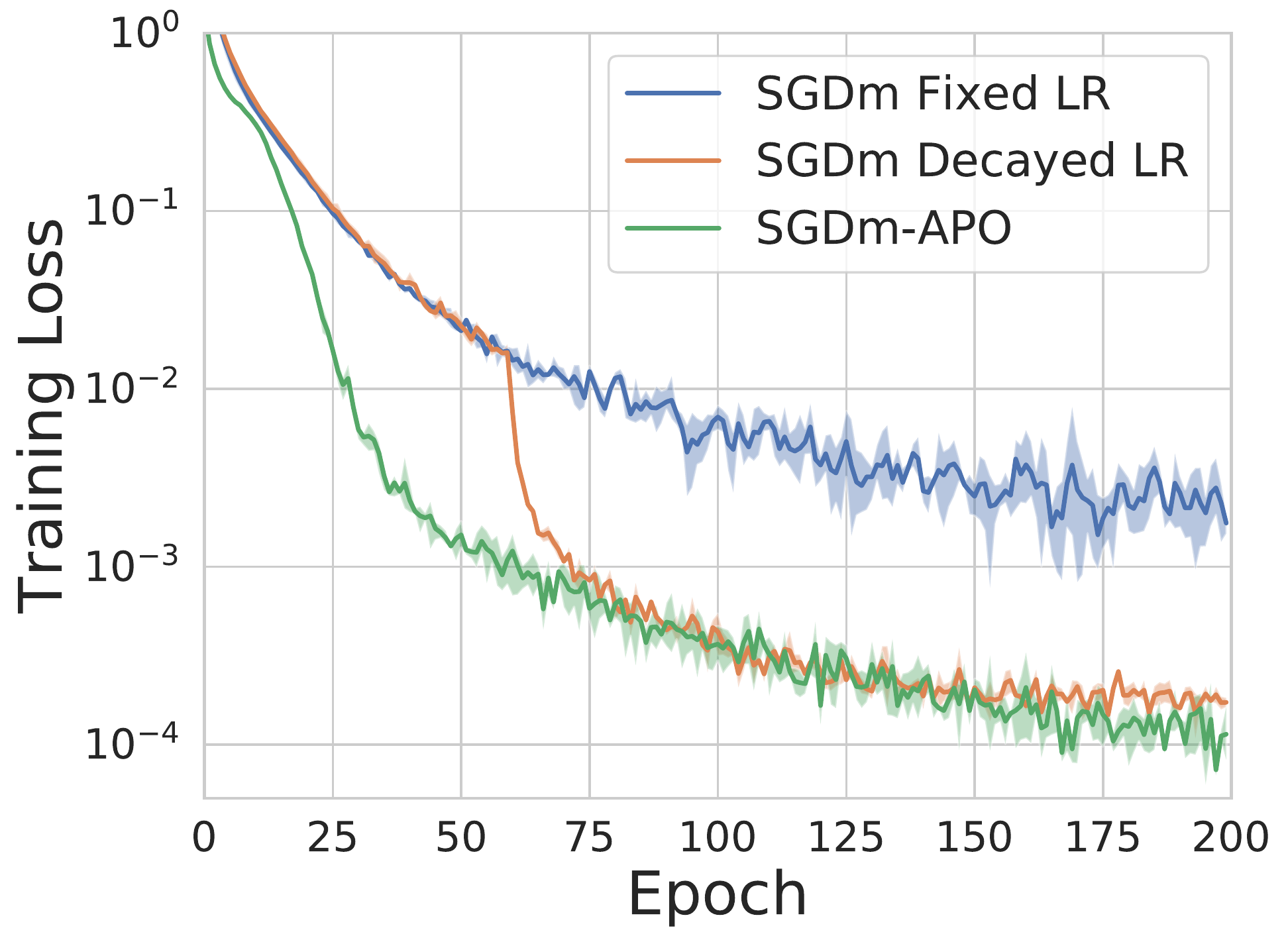}
\includegraphics[width=0.3\linewidth]{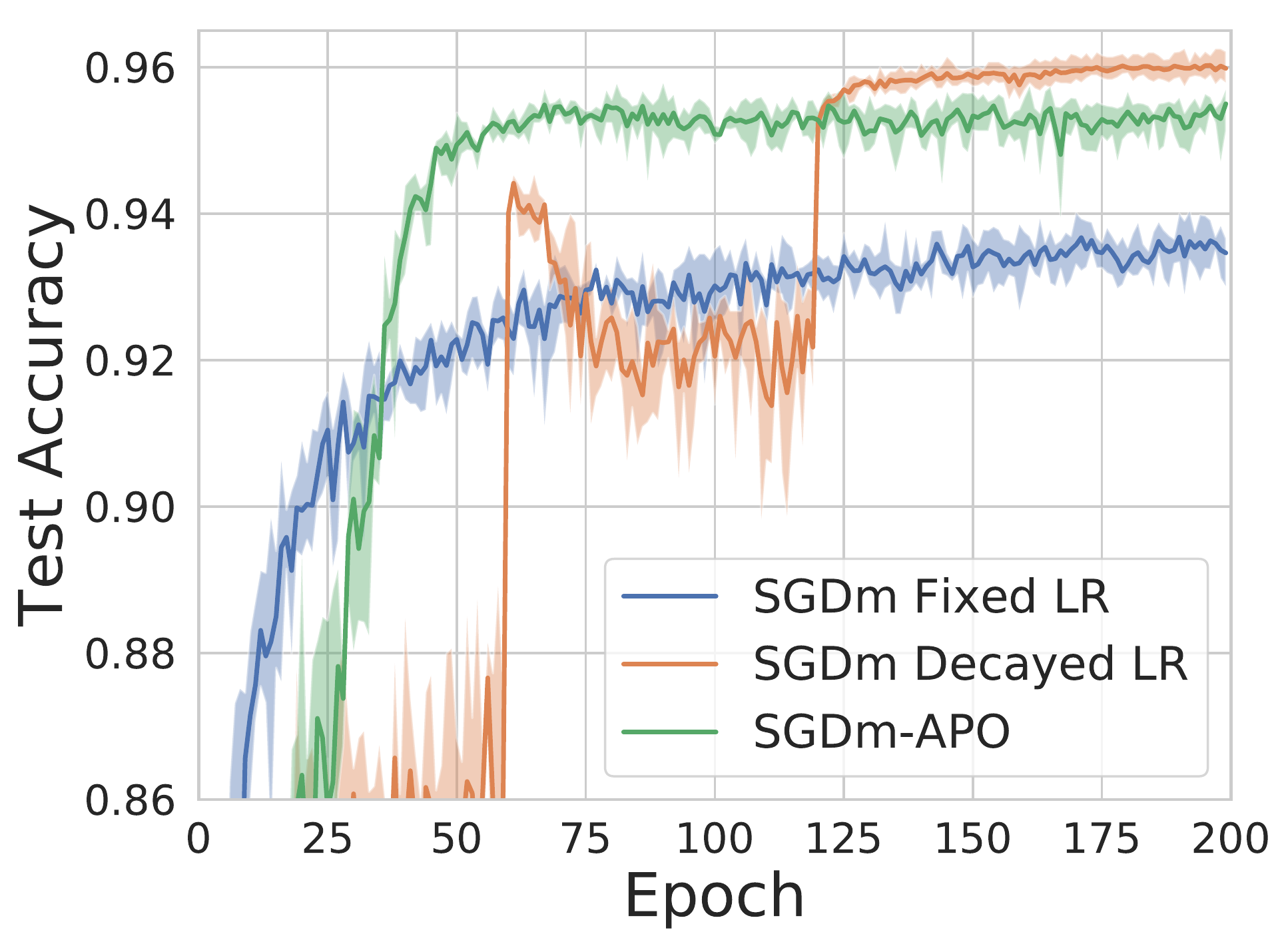}
\includegraphics[width=0.3\linewidth]{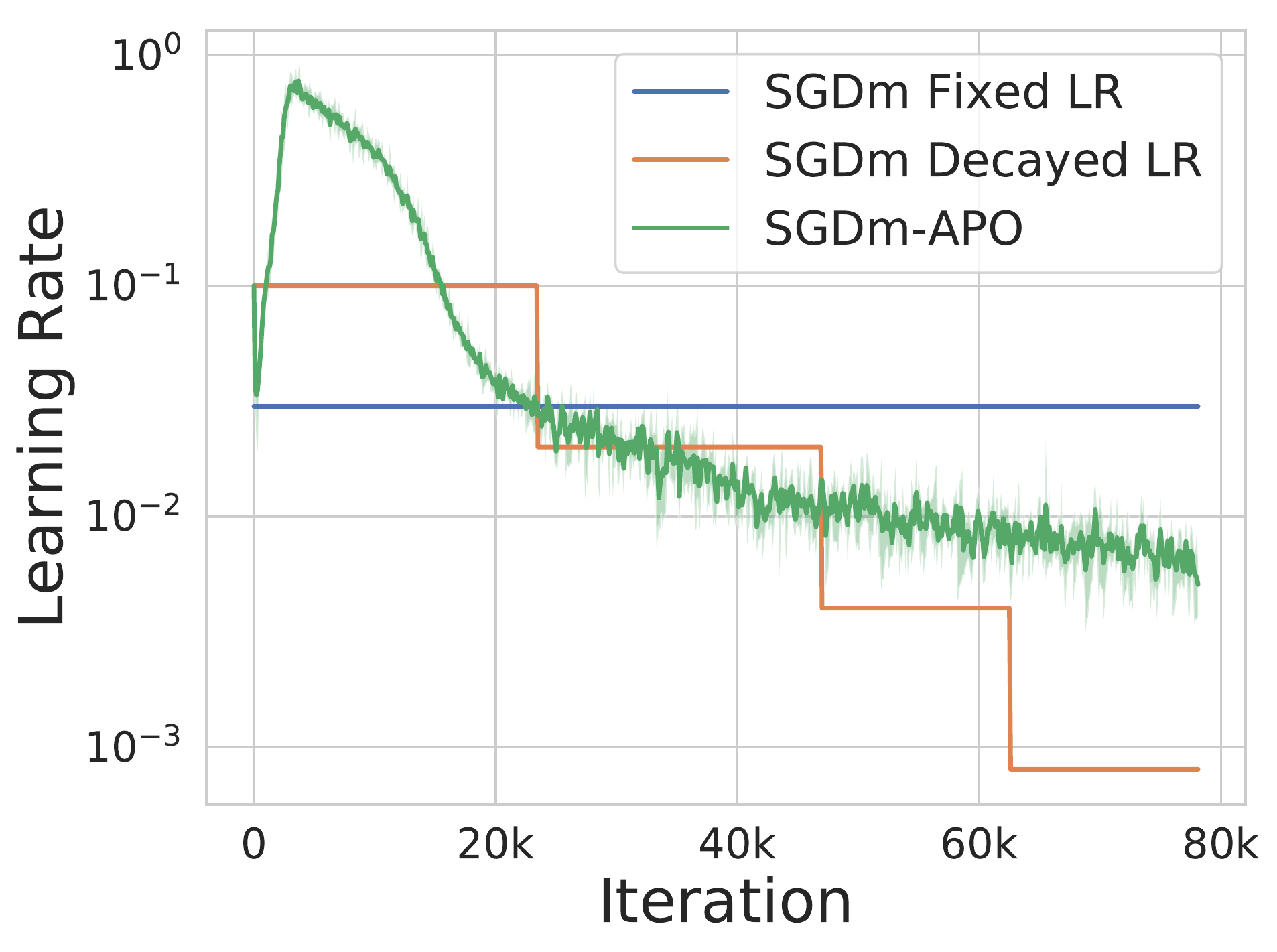}
\end{tabular}
\vspace{-0.2cm}
\caption{A comparison of the training loss (\textbf{left}), test accuracy (\textbf{middle}), and learning rate schedules (\textbf{right}) of baselines with fixed and decayed learning rates trained on WideResNet 28-10 with CIFAR-10. The training loss plot is shown using hyperparameters chosen based on training loss, while the middle and right plots use hyperparameters chosen based on validation accuracy.}
\label{fig:wrn-cifar10-sgdm}
\vspace{-0.5cm}
\end{figure*}

\begin{table*}[t]
    % \vspace{4mm}
    \setlength{\tabcolsep}{4pt}
    \centering
    \footnotesize
    \begin{tabular}{c|ccc|ccc|ccc|ccc}
    \toprule
        & \multicolumn{3}{c|}{\begin{tabular}{c}\textbf{CIFAR-10} \\ \textbf{ResNet-32}\end{tabular}} &
        \multicolumn{3}{c|}{\begin{tabular}{c}\textbf{CIFAR-10} \\ \textbf{ResNet-34}\end{tabular}}  &
        \multicolumn{3}{c|}{\begin{tabular}{c}\textbf{CIFAR-10} \\ \textbf{WRN 28-10}\end{tabular}} & \multicolumn{3}{c}{\begin{tabular}{c}\textbf{CIFAR-100} \\ \textbf{WRN 28-10}\end{tabular}} \\ \midrule
        & \textbf{Fixed} & \textbf{Decay}      & \textbf{APO} & \textbf{Fixed} & \textbf{Decay}      & \textbf{APO}      & \textbf{Fixed}     & \textbf{Decay}    & \textbf{APO}    & \textbf{Fixed}  & \textbf{Decay}  & \textbf{APO}  \\ \midrule
    \textbf{SGD}     & 90.07 & 93.30 & 92.71   & 93.00 & 93.54 & 94.27  & 93.38 & 94.86 & 94.85 & 76.29 & 77.92 &  76.87   \\
    \textbf{SGDm}    & 89.40 & 93.34 & 92.75   & 92.99 & 95.08 & 94.47  & 93.46 & 95.98 & 95.50 & 74.81 & 81.01 & 79.33     \\
    \textbf{RMSprop} & 89.84 & 91.94 & 91.28   & 92.87 & 93.87 & 93.97  & 92.91 & 93.60 & 94.22 & 72.06 & 76.06 &  74.17     \\
    \textbf{Adam}    & 90.45 & 92.26 & 91.81   & 93.23 & 94.12 & 93.80  & 92.81 & 94.04 & 93.83 & 72.01 & 75.53 & 76.33   \\
    \bottomrule
    \end{tabular}
    \vspace{-0.1cm}
    \caption{A comparison of the test accuracies achieved by the optimal fixed learning rate, the manual step decay schedule, and the APO-discovered schedule, using ResNet-32~\citep{he2016deep} and WideResNet 28-10~\citep{zagoruyko2016wide}.
    APO outperforms the optimal fixed learning rate, and is competitive with manual schedule.
    Note that APO generally achieves test accuracy similar to manual schedules in fewer iterations (shown in App.~\ref{app:additional-results}).}
    \label{table:lr-accuracies}
\end{table*}
\vspace{-0.1cm}
\section{Conclusion}
\vspace{-0.1cm}

We introduced Amortized Proximal Optimization (APO), a framework for online meta-learning of optimization parameters which approximates the proximal point method by learning a parametric update rule.
The update rule is adapted to better minimize a combination of the current-batch loss, function space discrepancy, and weight space discrepancy.
% We showed that under certain assumptions, exactly solving the meta-objective recovers classic second-order methods.
As the meta-parameters are updated only once per $K$ steps of optimization, APO incurs minimal computational overhead.
We applied the APO framework to two settings: (1) meta-learning the global learning rate for existing base optimizers such as SGD, RMSprop, and Adam and (2) meta-learning structured preconditioning matrices, which provides a new approach to second-order optimization.
For preconditioning adaptation problem, we showed that, under certain assumptions, exactly solving the meta-objective recovers classic second-order methods.
Compared to other second-order methods such as KFAC, APO eliminates the need to compute matrix inverses, enabling greater efficiency and numerical stability. Adapting update rules to a single meta-objective also opens up the possibility of coming up with even more clever update rules which are better able to exploit the structure of the optimization landscape.
%We evaluated APO on a range of tasks, and showed that it can perform competitively with second-order methods, and improve generalization compared to baseline first- and second-order optimizers.

% \vspace{-0.1cm}
\section*{Acknowledgements}
% \vspace{-0.1cm}
We thank Michael Zhang for valuable feedback on this paper and thank Alston Lo for helping setting up the experiments. We would also like to thank Saminul Haque, Jonathan Lorraine, Denny Wu, and Guodong Zhang, and our many other colleagues for their helpful discussions throughout this research. Resources used in this research were provided, in part, by the Province of Ontario, the Government of Canada through CIFAR, and companies sponsoring the Vector Institute (\url{www.vectorinstitute.ai/partners}).

\bibliography{main}
\bibliographystyle{icml2022}

\newpage
\appendix

\section*{Appendix}

This appendix is structured as follows:
\begin{itemize}
    \item In Section~\ref{app:table-of-notation}, we provide an overview of the notation we use throughout the paper.
    % \item In Section~\ref{app:extended-related-work}, we provide a discussion of additional related work.
    \item In Section~\ref{app:experiment-details}, we provide experiment details.
    \item In Section~\ref{app:additional-results}, we provide additional experimental results.
    \item In Section~\ref{app:same-minibatch-first-term}, we provide an ablation investigating the difference between evaluating the loss term on the same minibatch used to compute the base optimizer step vs a separate randomly-sampled minibatch. Using a separate random minibatch yields a similar meta-objective to that studied in~\citet{wu2018understanding}, which suffers from short-horizon bias.
    % \item In Section~\ref{app:base-optimizers-examples}, we provide the update equations prescribed by different base optimizers, including SGD, SGD with momentum, preconditioned SGD, RMSprop, and Adam.
    \item In Section~\ref{app:approximate-proximal-point}, we provide a derivation of Eq.~\ref{eq:quadratic-proximal}.
    \item In Section~\ref{app:opt-alg-derivations}, we derive classic first- and second-order optimization algorithms (gradient descent, Newton's method, and natural gradient) starting from the proximal objective for the PPM (Eq.~\ref{eq:nn-proximal-point-method}).
    \item In Section~\ref{app:precond-opt}, we provide the proof for Theorem~\ref{thm:optimal-precond}, which shows that exactly optimizing the approximate proximal meta-objective with respect to the preconditioner $\boldP$ can recover various first- and second-order optimization methods.
    \item In Section~\ref{app:meta-opt-kfac}, we list the KFAC assumptions and prove Corollary~\ref{thm:kfac}, which shows that under the assumptions of Theorem~\ref{thm:optimal-precond} and the KFAC assumptions, exactly optimizing the approximate proximal meta-objective yields the KFAC update.
    \item In Section~\ref{app:labmda-ablation}, we show an ablation over $\lamwsd$ and $\lamfsd$, and evaluate how well APO performs with different meta-update intervals.
\end{itemize}

\clearpage

\section{Table of Notation}
\label{app:table-of-notation}

\begin{table}[h!]
    \begin{center}
        \begin{tabular}{c  c}
            \toprule
            Notation         & Description \\
            \midrule
            % Variables
            $\mathcal{D}$ & Data-generating distribution \\
            $\mathcal{D}_{\text{train}}$ & Finite training dataset \\
            $(\boldx, \boldt) \sim \mathcal{D}$ & An input/target pair \\
            $\mathcal{B}$ & Mini-batch of data \\
            $N$ &  Number of training data points, $N = |\mathcal{D}_{\text{train}}|$ \\
            $\vtheta$ & Network parameters \\
            $\hyper$ & Optimization parameters (e.g.~learning rate or preconditioning matrix)\\
            $\pre$ & Preconditioning matrix \\
            $\mathbf{G}$ & Hessian of the function-space discrepancy \\
            $\mathbf{W}$ & Weight matrix of some particular layer of the network \\
            $\pre_{\text{S}}$ & Structured preconditioner, $\pre_{\text{S}} = (\mathbf{A} \otimes \mathbf{B}) \text{diag}(\text{vec}(\mathbf{S}))^2 (\mathbf{A} \otimes \mathbf{B})^\top$ \\
            $\mathbf{A}, \mathbf{B}, \mathbf{S}$ & Block matrices for EKFAC parameterization \\
            $\lambda$      & Weighting of the discrepancy term in the PPM \\
            $\lamfsd$      & Weighting of the function-space discrepancy term \\
            $\lamwsd$      & Weighting of the weight-space discrepancy term \\
            $\eta$      & Base optimizer learning rate \\
            $\alpha$      & Meta optimizer learning rate \\
            $K$ & Meta update interval \\

            % Dimensions
            $p$ & Number of optimization hyperparameters\\
            $m$ & Number of parameters\\

            % Functions
            $f(\boldx, \boldtheta)$ & Network function with parameters $\boldtheta$ applied to input $\boldx$ \\
            $\mathcal{L}(\mathbf{y}, \mathbf{t})$ & Loss function (e.g.~mean-squared error or cross-entropy) \\
            $\mathcal{J}(\vtheta)$ & Cost function, $\cost (\vtheta) = \frac{1}{N} \sum_{i=1}^N \mathcal{L} (f(\mathbf{x}^{(i)}, \vtheta), \mathbf{t}^{(i)})$ \\
            $\mathcal{J}_{\batch}(\vtheta)$ & Loss on a mini-batch of data, $\cost_{\batch}(\vtheta) = \frac{1}{|\batch|} \sum_{(\boldx, \boldt) \sim \batch} \loss(f(\boldx, \vtheta), \boldt)$ \\
            $\mathcal{Q}(\hyper)$ & Proximal meta-objective function \\
            $\hat{\mathcal{Q}}(\hyper)$ & Approximate proximal meta-objective function \\
            $\rho(\cdot, \cdot)$ & Output-space divergence (e.g.~KL divergence) \\
            $D(\cdot, \cdot)$ & Discrepancy function for the PPM \\
            $\disf(\cdot, \cdot, \cdot)$ & Function-space discrepancy function \\
            $\disw(\cdot, \cdot)$ & Weight-space discrepancy function \\
            % $\rho(\cdot, \cdot)$ & Dissimilarity function \\
            $\text{vec}(\cdot)$ & Vectorization function \\
            $\text{diag}(\cdot)$ & Diagonalization function \\
            $u(\vtheta, \hyper, \mathcal{B})$ & Base optimizer update (e.g.~stochastic gradient descent)\\
            $\vtheta'(\hyper)$ & Shorthand for $u(\vtheta, \hyper, \batch)$ where the data $\batch$ is implicit \\

            % Operations
            $\otimes$ & Kronecker product \\
            $\odot$ & Elementwise product \\
            \bottomrule
        \end{tabular}
    \end{center}
    \vspace{-0.2cm}
    \caption{A summary of notation used in this paper.}
    \label{tab:table-of-notation}
\end{table}

\section{Experimental Details}
\label{app:experiment-details}

\subsection{Computing Environment}
All experiments were implemented using the PyTorch~\citep{NEURIPS2019_9015} and JAX~\citep{jax2018github} libraries, and we ran all experiments on NVIDIA P100 GPUs.

\subsection{Regression on UCI Collection}
\label{appendix:experiment-details:mlp_uci}
We used Slice, Protein, and Parkinson data from the UCI collection. In training, we normalized the input features and targets to have a zero mean and unit variance. We used batch size of 128 and trained the network for 500 epochs without weight decay.

We conducted hyperparameter searches over learning rates for all baseline models, making choices based on the final training loss. With the chosen set of hyperparameters, we repeated the experiments 3 times with different random seeds. For SGDm, we set the momentum to 0.9 and swept over the learning rates \{1, 0.3, 0.1, 0.03, 0.01, 0.003, 0.001, 0.0003, 0.0001\}. For Adam, we performed a grid search on learning rate of \{1e-2, 3e-3, 1e-3, 3e-4, 3e-4, 1e-4\}. For Shampoo, we swept over the learning rates \{10, 5, 1, 0.3, 0.1, 0.03, 0.01, 0.003, 0.001\} as suggested by~\citet{gupta2018shampoo}. For KFAC, we did a grid search on learning rates of \{0.03, 0.01, 0.003, 0.001, 0.0003, 0.0001\} and damping values of \{3e-2, 1e-2, 3e-3, 1e-3, 3e-4, 1e-4\}. For APO-Precond, we set $\lambfsd = 0$ and grid searched over $\lambwsd = \{3, 1, 0.1, 0.01\}$.

\subsection{Image Reconstruction}
We used the same experimental set-up from~\citet{martens2015optimizing} for the deep autoencoder experiment. The loss function was defined to be the binary entropy and we further added a regularization term $\frac{\lambda_{\text{WD}}}{2} \|\vtheta\|^2$ to the loss function, where $\lambda_{\text{WD}} = 10^{-5}$. The layer widths for the autoencoder were set to be [784, 1000, 500, 250, 30, 250, 500, 1000, 784] and we used the sigmoid activation function. We trained the network for 1000 epochs with the batch size of 512.

We conducted extensive hyperparameter searches for all baseline models, making choices based on the final training loss. With the chosen set of hyperparameters, we repeated the experiments 3 times with different random seeds. For SGDm, we set the momentum to 0.9 and performed a grid search over the learning rates \{1, 0.3, 0.1, 0.03, 0.01, 0.003, 0.001, 0.0003, 0.0001\}. For Adam, we swept over learning rate of \{3e-2, 1e-2, 3e-3, 1e-3, 3e-4, 3e-4, 1e-4, 3e-5, 1e-5\}. For Shampoo, we tried setting learning rates in range of \{10, 5, 1, 0.3, 0.1, 0.03, 0.01, 0.003, 0.001\}. For KFAC, we did a grid search on learning rates of \{0.03, 0.01, 0.003, 0.001, 0.0003, 0.0001, 0.00003, 0.00001\} and damping values of \{3e-1, 1e-1, 3e-2, 1e-2, 3e-3, 1e-3, 3e-4, 1e-4\}. For APO-Precond, we performed a grid search over 
$\lambfsd = \{0.3, 0.1\}$ and $\lambwsd = \{3, 1, 0.3, 0.1\}$. We set the FSD to measure the KL divergence. 

\subsection{Image Classification}

\paragraph{CIFAR-10 \& CIFAR-100.} We used the standard procedure from~\citet{zagoruyko2016wide} for training convolutional neural networks. The images are zero-padded with 4 pixels on each side and then a random $32 \times 32$ crop is extracted from the image, horizontally flipped with the probability of 0.5. We further normalized the inputs with per-channel mean and standard deviations. We performed the extensive search over the hyperparameters for all networks. We held out 5k examples from the CIFAR-10 and CIFAR-100 to form the validation set following~\citet{zagoruyko2016wide} and selected the hyperparameters with the highest validation accuracy. With the chosen hyperparameters, we re-trained the network with the full training dataset and reported the final test accuracy.

\paragraph{Preconditioning Adaptation.} Across all experiments, we used the batch size of 128 and trained the network for 200 epochs. 
With the chosen set of hyperparameters, we repeated the experiments 3 times with different random seeds and reported the mean accuracy on the test dataset. For SGDm, we set the momentum to 0.9 and performed a grid search over the learning rates \{0.3, 0.1, 0.03, 0.01, 0.003, 0.001, 0.0003, 0.0001\}. For Adam, we swept over learning rate of \{3e-2, 1e-2, 3e-3, 1e-3, 3e-4, 3e-4, 1e-4\}. For KFAC, we did a grid search on learning rates of \{0.03, 0.01, 0.003, 0.001, 0.0003, 0.0001, 0.00003, 0.00001\} and damping values of \{3e-2, 1e-2, 3e-3, 1e-3, 3e-4, 1e-4\}. For APO-Precond, we performed a grid search over 
$\lambfsd = \{0.3, 0.1\}$ and $\lambwsd = \{3, 1, 0.3, 0.1\}$ for architectures without batch normalization. As batch normalization makes the parameters scale-invariant~\citep{arora2018theoretical}, we imposed a higher regularization in the function space and lower regularization in the weight space, searching over $\lambfsd = \{3, 1\}$ and $\lambwsd = \{0.3, 0.1, 0.03, 0.01\}$. We set the FSD term to measure the KL divergence. We also searched over the weight decay in range of \{5e-4, 1e-4, 5e-5\} for all optimizers. 

\paragraph{Learning Rate Adaptation.}
For ResNet32, we trained for 400 epochs, and used a manual schedule that decays the learning rate by a factor of 10 at epochs 150 and 250, following~\citet{lucas2018aggregated}.
For ResNet34 and WideResNet 28-10, we trained for 200 epochs, with a manual schedule that decays the learning rate by a factor of 5 at epochs 60, 120, and 160, following~\citet{zagoruyko2016wide}.

For the baseline optimizers, we performed grid searches over the fixed learning rate or initial learning rate for a fixed step schedule, as well as the weight decay.
For all base optimizers and APO-tuned variants, we searched over weight decay values in $\{ 0.01, 0.003, 0.001, 0.0003, 0.0001, 0.0 \}$.
For SGD and SGDm, we searched over learning rates in $\{ 1.0, 0.3, 0.1, 0.03, 0.01, 0.003, 0.001, 0.0003, 0.0001 \}$.
For RMSprop and Adam, we searched over learning rates in $\{ 0.1, 0.03, 0.01, 0.003, 0.001, 0.0003, 0.0001, 0.00003, 0.00001 \}$.
For each of the APO-tuned variants (e.g.~SGDm-APO), we kept $\lamwsd = 0$ and searched over $\lamfsd \in \{ 0.3, 0.1, 0.03, 0.01, 0.003, 0.001, 0.0003,  0.0001 \}$.

\subsection{Neural Machine Translation}
We trained a Transformer~\citep{vaswani2017attention} composed of 6 encoder and decoder layers, with a word embedding and hidden vector dimension of 512. The architecture has feed-forward size of 1024, 4 attention heads, dropout value 0.3, and weight decay value 0.0001. For APO-Precond, following the practice from~\citet{zhang2019algorithmic}, we used a diagonal block matrix in the structured preconditioner of the embedding weight matrix to reduce memory overhead. We used the Fairseq toolkit~\citep{ott2019fairseq} to conduct all experiments.

For SGDm, we grid searched over the learning rates in \{10, 3, 1, 0.3, 0.1, 0.03, 0.01, 0.003, 0.001\} and tried both the fixed learning rate and inverse sqrt learning rate schedules. For AdamW, we set $\beta_1 = 0.9$, $\beta_2 = 0.98$, and $\epsilon = 10^{-1}$ and searched over the learning rates in range \{0.05, 0.001, 0.0005, 0.0001\}. AdamW also used the inverse sqrt learning rate schedule with 4000 warmup steps. For APO-Precond, we searched over $\lambfsd = \{0.3, 0.1\}$ and $\lambwsd = \{3, 1, 0.3, 0.1\}$ and let FSD term to measure the KL divergence. As we found a fixed learning rate schedule to work best for SGDm, we also used a fixed learning rate schedule for APO-Precond. We selected the hyperparameters based on the BLEU score on the validation set and reported on the final test BLEU score on the checkpoint, which achieved the highest validation BLEU score.

\clearpage

\section{Additional Results}
\label{app:additional-results}

\subsection{Rosenbrock Function}
\label{app:rosenbrock}

\begin{wrapfigure}[12]{r}{0.36\linewidth}
    \vspace{-2.8cm}
    \centering
    \includegraphics[width=\linewidth]{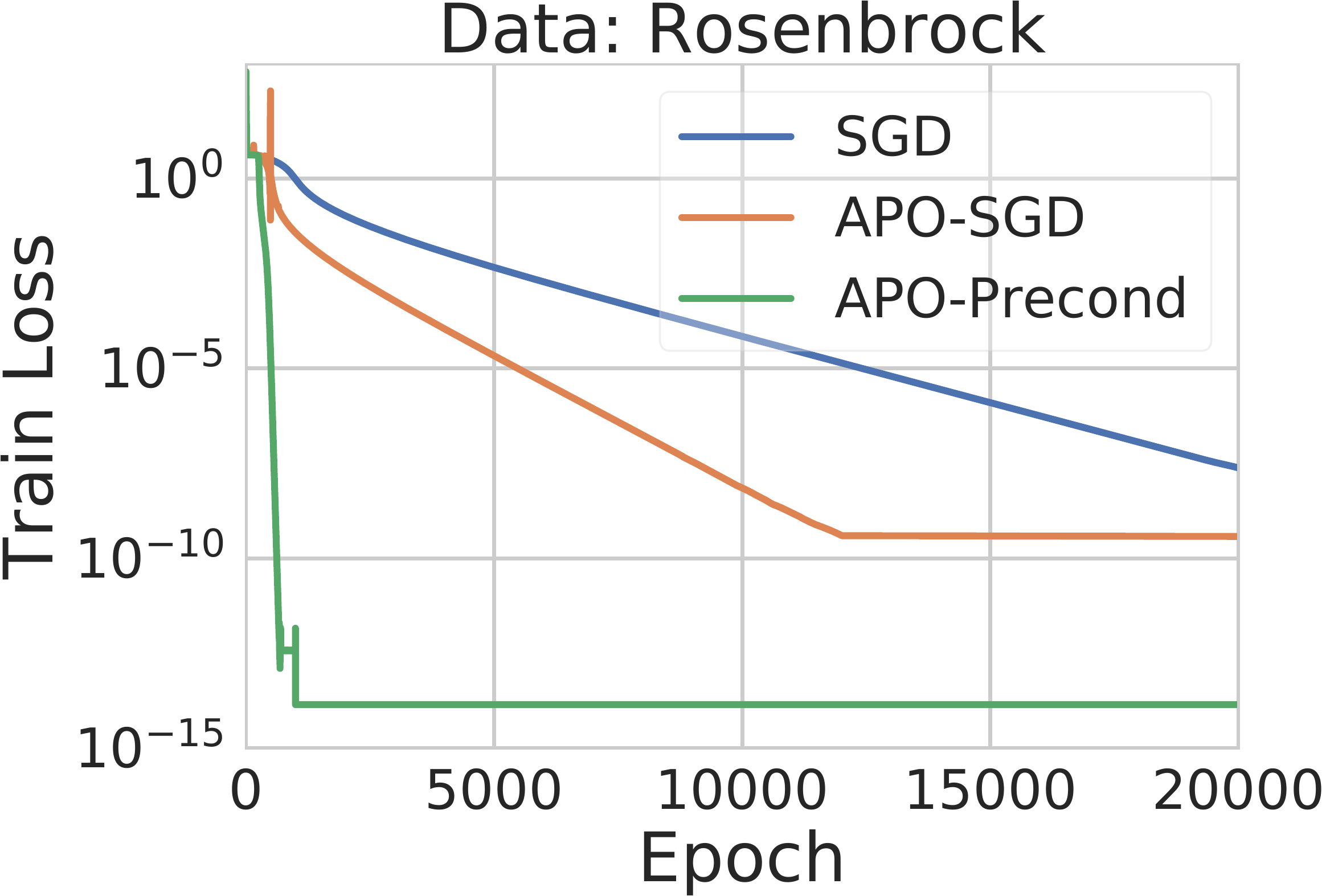}
    \vspace{-0.4cm}
    \caption{Training loss on the Rosenbrock function. Both the learning rate schedules and the preconditioner adapted by APO outperform the optimal fixed learning rate for SGD.}
    \label{fig:rosenbrock-sgd}
\end{wrapfigure}
We validated APO on the two-dimensional Rosenbrock function defined as:
\begin{align*}
    f(x,y) = (1 - x)^2 + 100 (y - x^2)^2
\end{align*}
with initialization $(x, y) = (1, -1.5)$.
In Figure~\ref{fig:rosenbrock-sgd}, we used APO to tune the learning rate for SGD, as well as the full preconditioning matrix.
Both APO-tuned methods outperformed vanilla SGD with the optimal fixed learning rate (chosen through a careful grid search).
Because Rosenbrock has ill-conditioned curvature, second-order optimization with APO-Precond dramatically speeds up the convergence and achieves a lower loss.

% \vspace{0.1cm}
\subsection{SVHN}
\label{app:svhn}
We also evaluated APO on the Google Street View House Numbers dataset (SVHN)~\citep{netzer2011reading}, using a WideResNet 16-8.
We followed the experimental setup of~\citet{zagoruyko2016wide}, and used the same decay schedule for the baseline, which decays the learning rate by $10\times$ at epochs 80 and 120.
The optimal fixed learning rate achieves test accuracy fluctuating around 97.20\%, while the manual schedule achieves 98.17\%.
APO rapidly converges to test accuracy 97.96\%, outperforming the baseline while being slightly worse than the manual schedule (Figure~\ref{fig:svhn-sgdm-test-acc}).

\begin{figure}[H]
\centering
\includegraphics[width=0.45\linewidth]{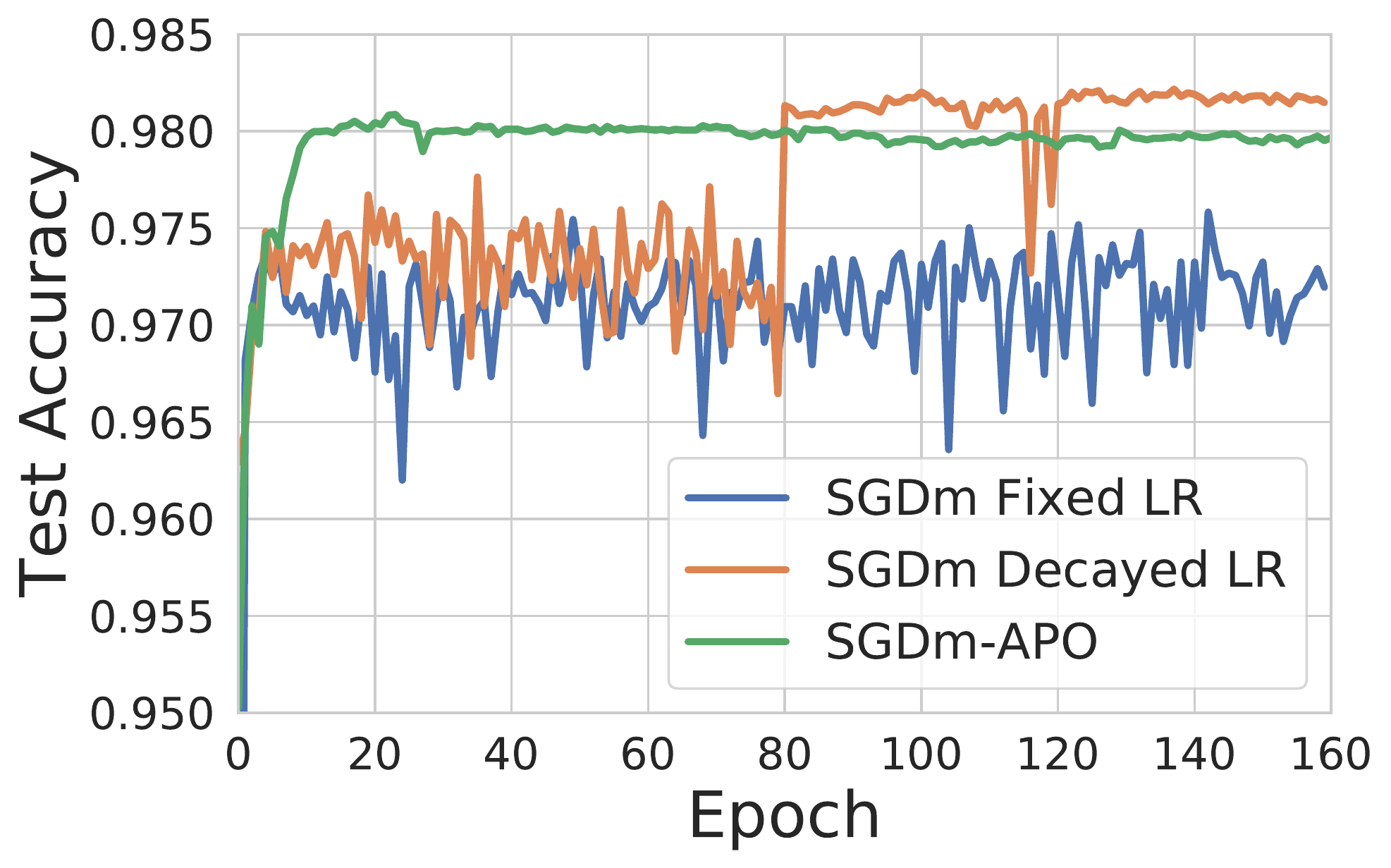}
\caption{Test Accuracy on SVHN using SGDm with fixed learning rate, decayed learning rate schedule, and APO-tuned learning rate.}
\label{fig:svhn-sgdm-test-acc}
\end{figure}

\subsection{CIFAR-10}\label{app:cifar10}

\paragraph{APO Preconditioning Adaptation.}

We show the plots for for experiments listed in Table~\ref{table:testaccuracy} in Figure~\ref{fig:precond-cifar10-nobn} and Figure~\ref{fig:precond-cifar10-bn}.

\begin{figure}[H]
    \centering
    \begin{tabular}{ccc}
    \includegraphics[width=0.3\linewidth]{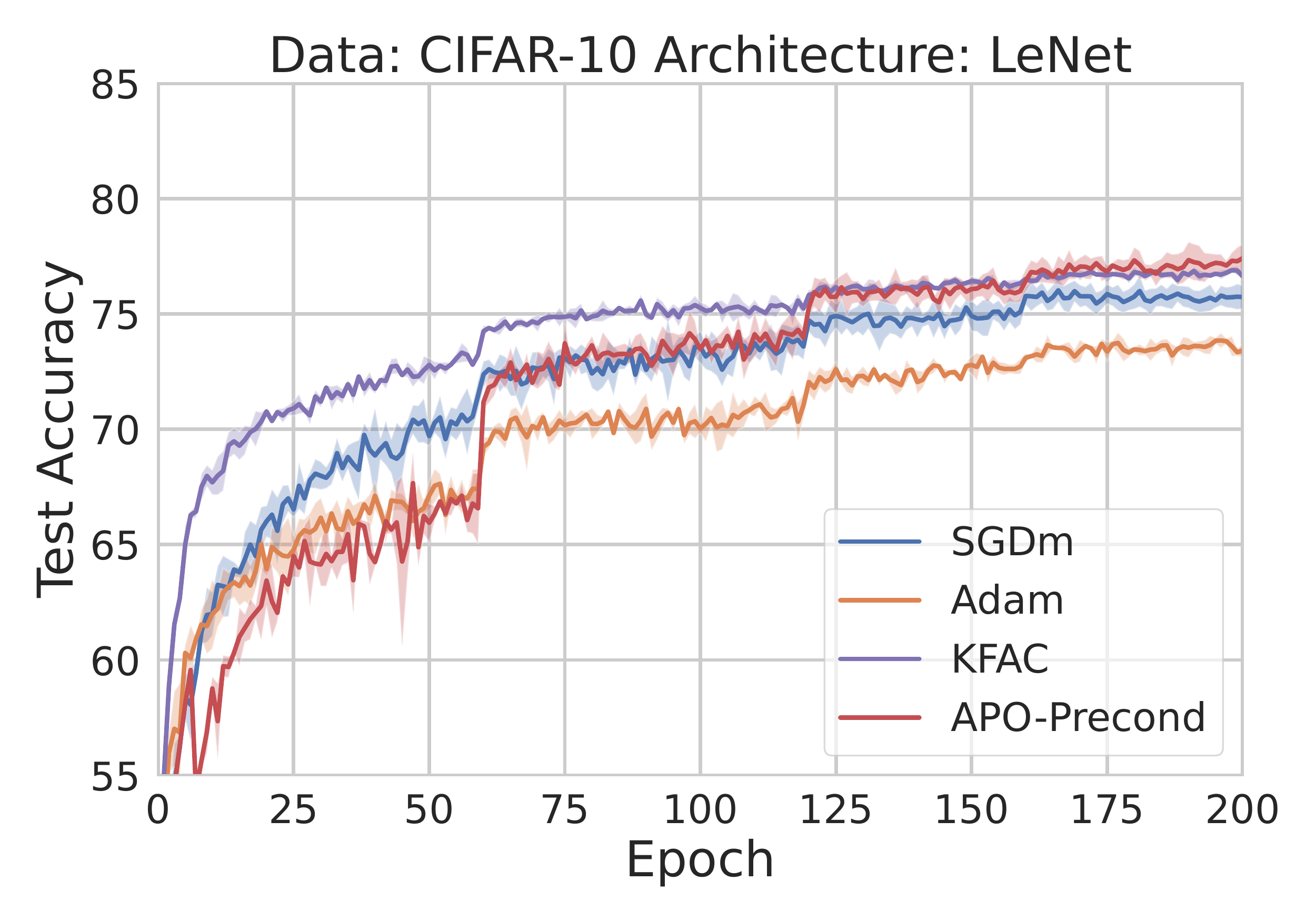}&
    \includegraphics[width=0.3\linewidth]{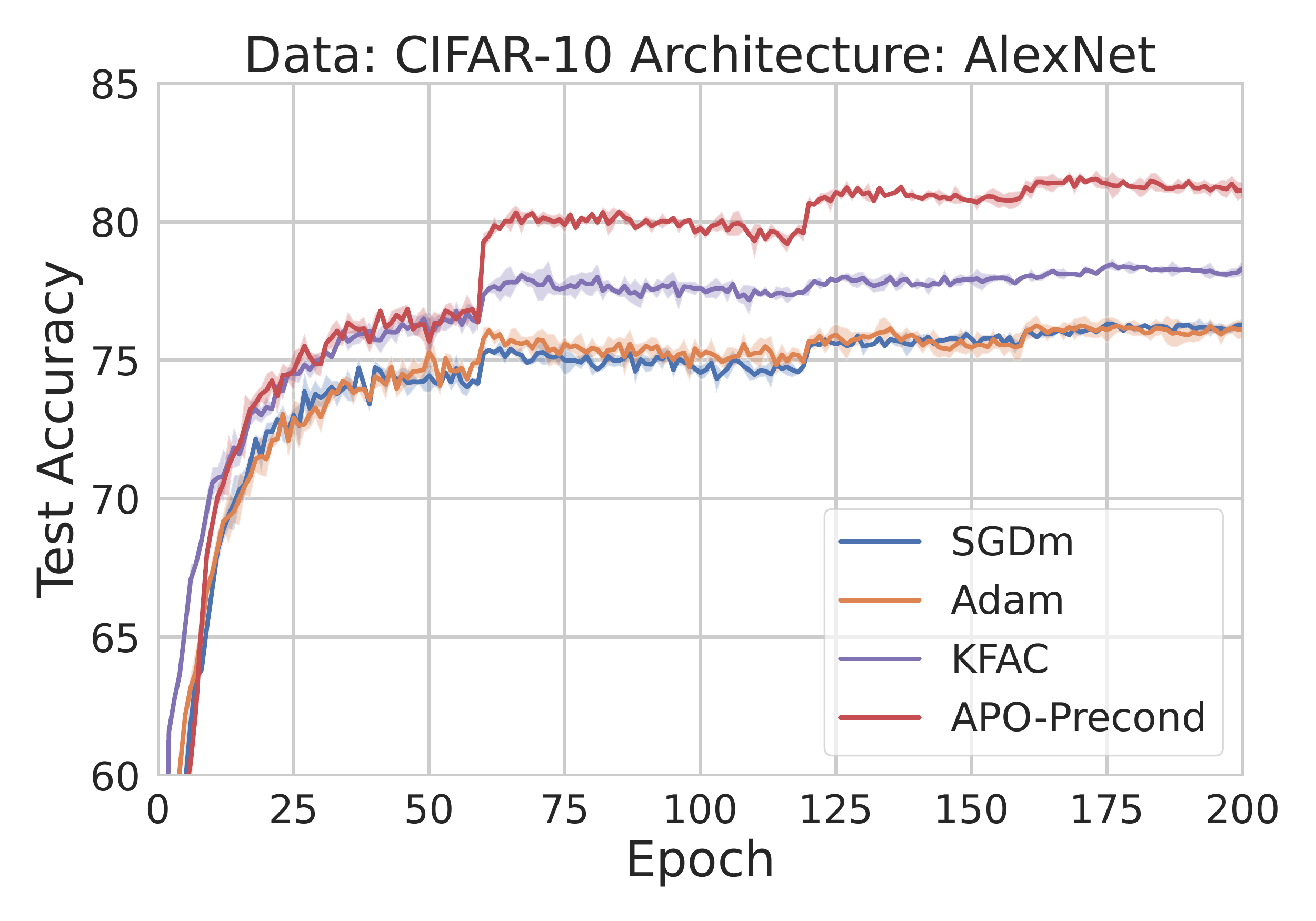}&
    \includegraphics[width=0.3\linewidth]{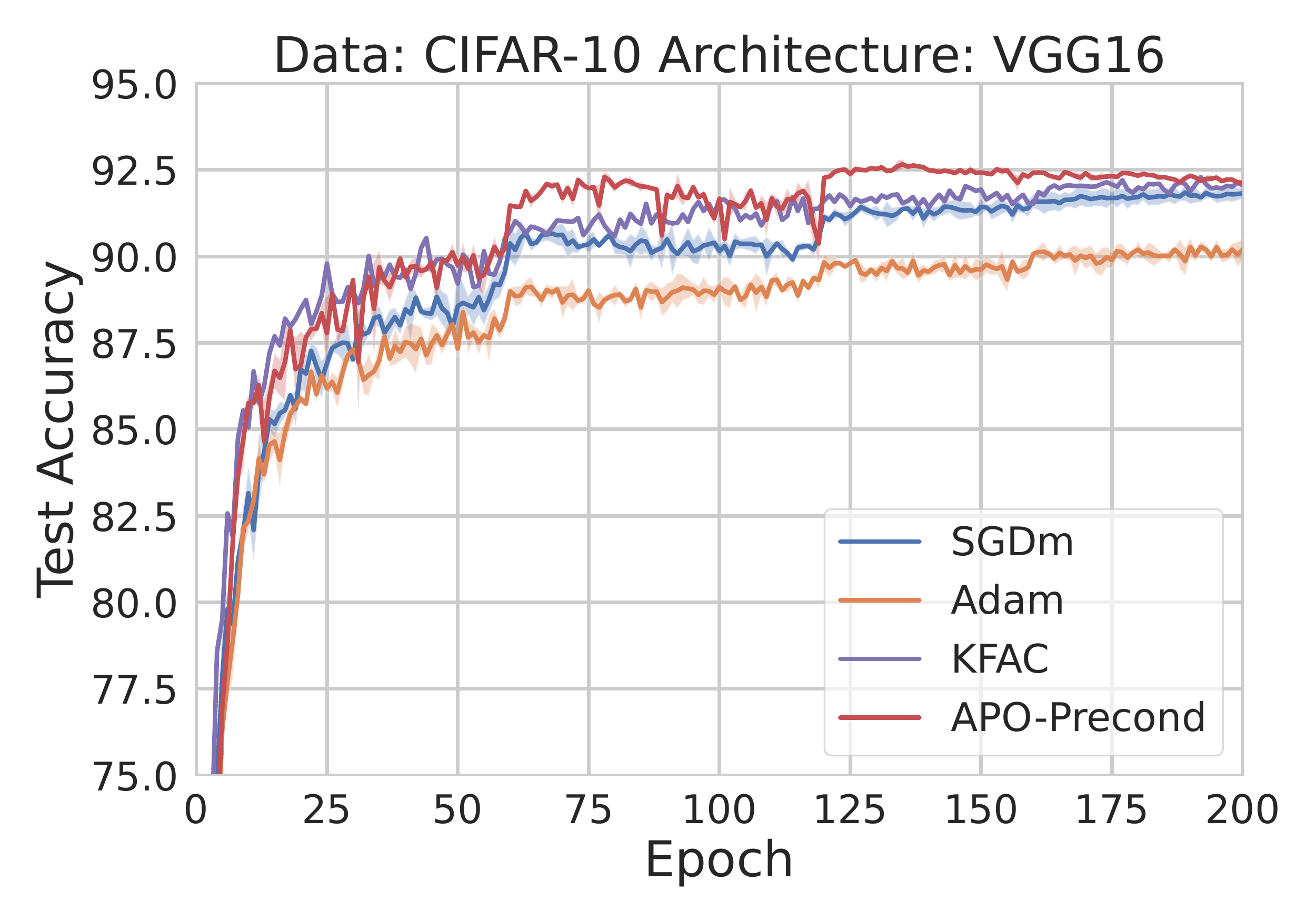}
    \end{tabular}
    \vspace{-2mm}
    \caption{LeNet, AlexNet, and VGG16 on CIFAR-10, using SGDm, Adam, KFAC, and APO-Precond.}
    \vspace{-4mm}
    \label{fig:precond-cifar10-nobn}
\end{figure}

\begin{figure}[H]
    \centering
    \begin{tabular}{ccc}
    \includegraphics[width=0.3\linewidth]{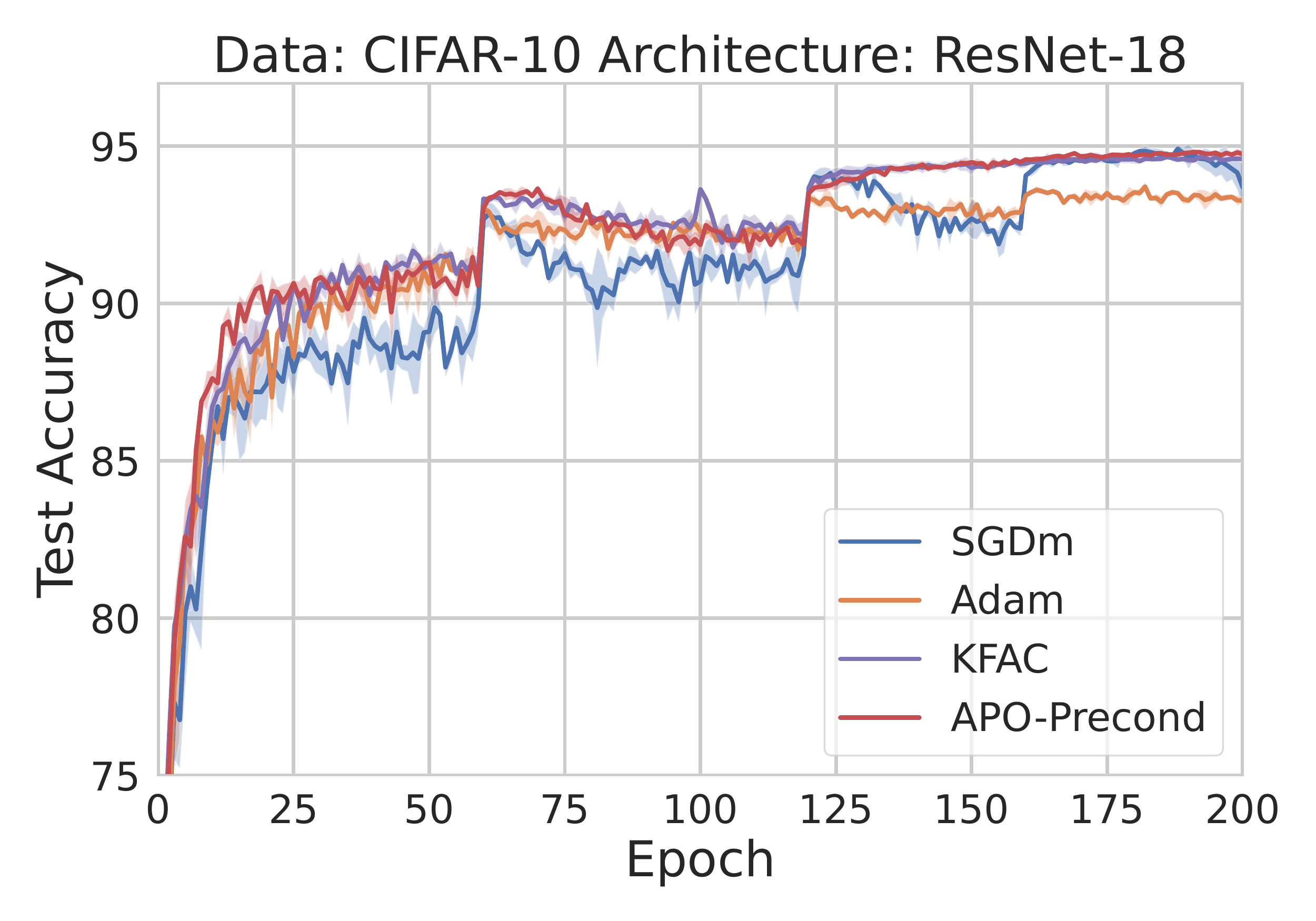}&
    \includegraphics[width=0.3\linewidth]{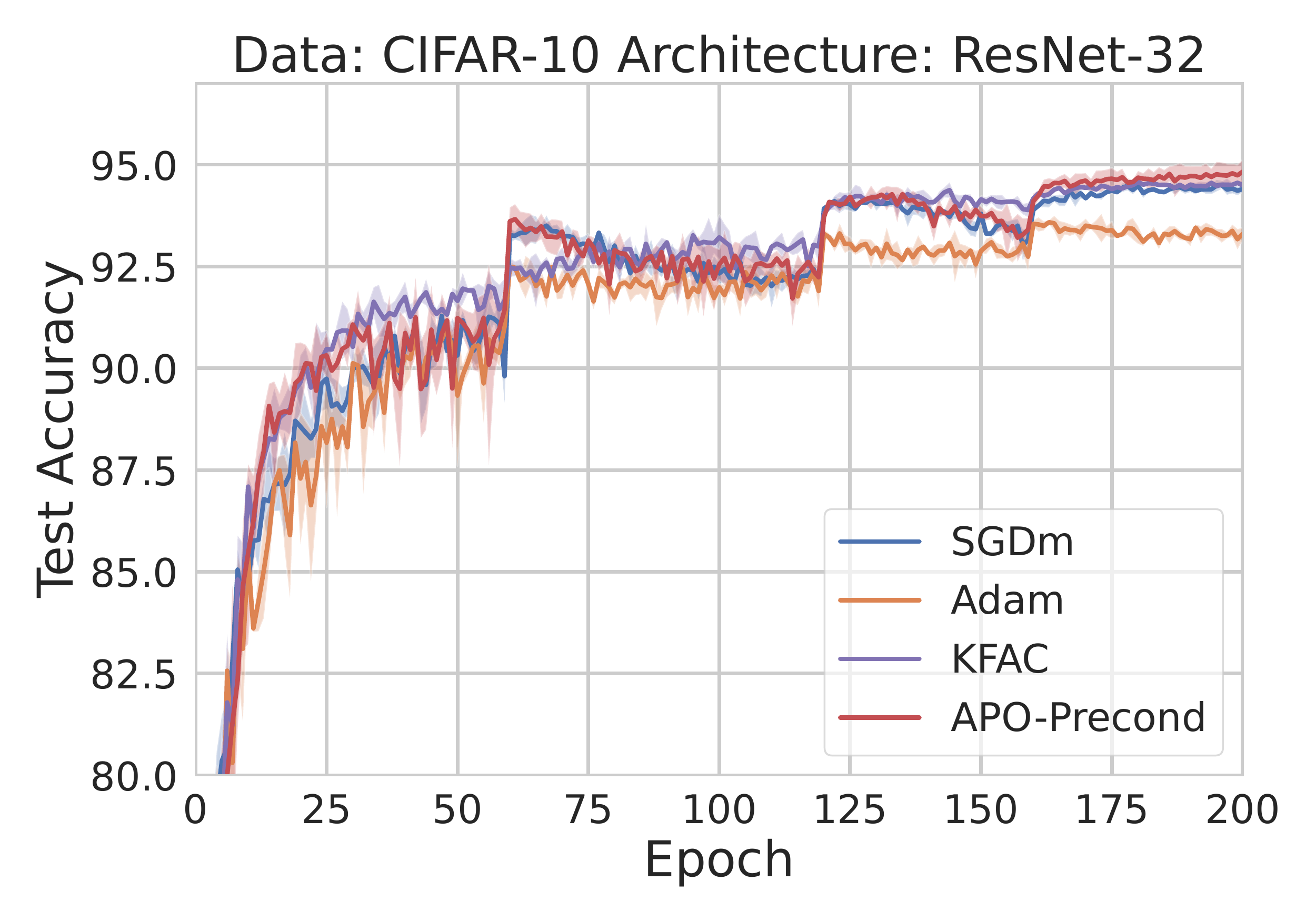}&
    \end{tabular}
    \vspace{-2mm}
    \caption{ResNet-18 and ResNet-32 on CIFAR-10, using SGDm, Adam, KFAC, and APO-Precond.}
    \vspace{-4mm}
    \label{fig:precond-cifar10-bn}
\end{figure}

% \begin{figure}[H]
%     \centering
%     \begin{tabular}{ccc}
%     \includegraphics[width=0.4\linewidth]{figures/cifar10_precond/cifar10_lenet.pdf}&
%     \includegraphics[width=0.4\linewidth]{figures/cifar10_precond/cifar10_alexnet.pdf} \\
%     \includegraphics[width=0.4\linewidth]{figures/cifar10_precond/cifar10_vgg16.pdf}&
%     \includegraphics[width=0.4\linewidth]{figures/cifar10_precond/cifar10_resnet32.pdf} \\
%     \includegraphics[width=0.4\linewidth]{figures/cifar10_precond/cifar10_resnet18.pdf}
%     \end{tabular}
%     \vspace{-2mm}
%     \caption{LeNet, AlexNet, and ResNet32 on CIFAR-10, using SGD, SGDm, RMSprop, Adam, KFAC, EKFAC, and APO-Precond.}
%     \vspace{-4mm}
%     \label{fig:precond-cifar10-nobn2}
% \end{figure}
% \begin{figure}[H]
%     \centering
%     \begin{tabular}{ccc}
%     \includegraphics[width=0.4\linewidth]{figures/cifar100_precond/cifar100_lenet.pdf}&
%     \includegraphics[width=0.4\linewidth]{figures/cifar100_precond/cifar100_alexnet.pdf} \\
%     \includegraphics[width=0.4\linewidth]{figures/cifar100_precond/cifar100_vgg16.pdf}&
%     \includegraphics[width=0.4\linewidth]{figures/cifar100_precond/cifar100_resnet32.pdf} \\
%     \includegraphics[width=0.4\linewidth]{figures/cifar100_precond/cifar100_resnet18.pdf}
%     \end{tabular}
%     \vspace{-2mm}
%     \caption{LeNet, AlexNet, VGG16, ResNet32, and ResNet18 on CIFAR-100, using SGD, SGDm, RMSprop, Adam, KFAC, EKFAC, and APO-Precond.}
%     \vspace{-4mm}
%     \label{fig:precond-cifar100-nobn2}
% \end{figure}

\paragraph{APO Learning Rate Schedules.}
We show the plots for each base-optimizer and network architecture as below. We find that APO achieves better performance than the best fixed LR, and is comparable to the manual schedule.
The learning rate schedules discovered by APO are nontrivial: most exhibit a large increase in the learning rate at the start of training to make rapid progress, followed by a gradual decrease in the learning rate to fine-tune the solution as optimization approaches a local optimum.
The APO learning rate schedules often span two orders of magnitude, similarly to manual decay schedules.

\begin{table}[h]
\centering
\begin{tabular}{c|ccc}
\toprule
    & \multicolumn{3}{c}{\textbf{CIFAR-10 (ResNet34)}}     \\ \midrule
    & \textbf{Fixed} & \textbf{Decayed}      & \textbf{APO} \\ \midrule
\textbf{SGD}     &  93.00 $\pm$ 0.35  & 93.54 $\pm$ 0.01 & 94.27 $\pm$ 0.02  \\
\textbf{SGDm}    &  92.99 $\pm$ 0.14  & 95.08 $\pm$ 0.24 & 94.47 $\pm$ 0.24  \\
\textbf{RMSprop} &  92.87 $\pm$ 0.32  & 93.97 $\pm$ 0.11 & 93.97 $\pm$ 0.07  \\
\textbf{Adam}    &  93.23 $\pm$ 0.15  & 94.12 $\pm$ 0.10 & 93.80 $\pm$ 0.14  \\ \bottomrule
\end{tabular}
\caption{Test accuracy on CIFAR-10 using ResNet34. APO consistently outperforms the best fixed learning rate for each optimizer, and is on par with carefully-tuned manual schedules. Each reported result is the mean of 4 random restarts, and we report $\pm$ the standard deviation.}
\end{table}

% ResNet34, CIFAR-10, SGD
% %%%%%%%%%%%%%%%%%%%%%%%%%%%%%%%%%%%%%%%%%%%%%%%%%%%%%%%
\begin{figure}[H]
\centering
\begin{tabular}{ccc}
\includegraphics[width=0.3\linewidth]{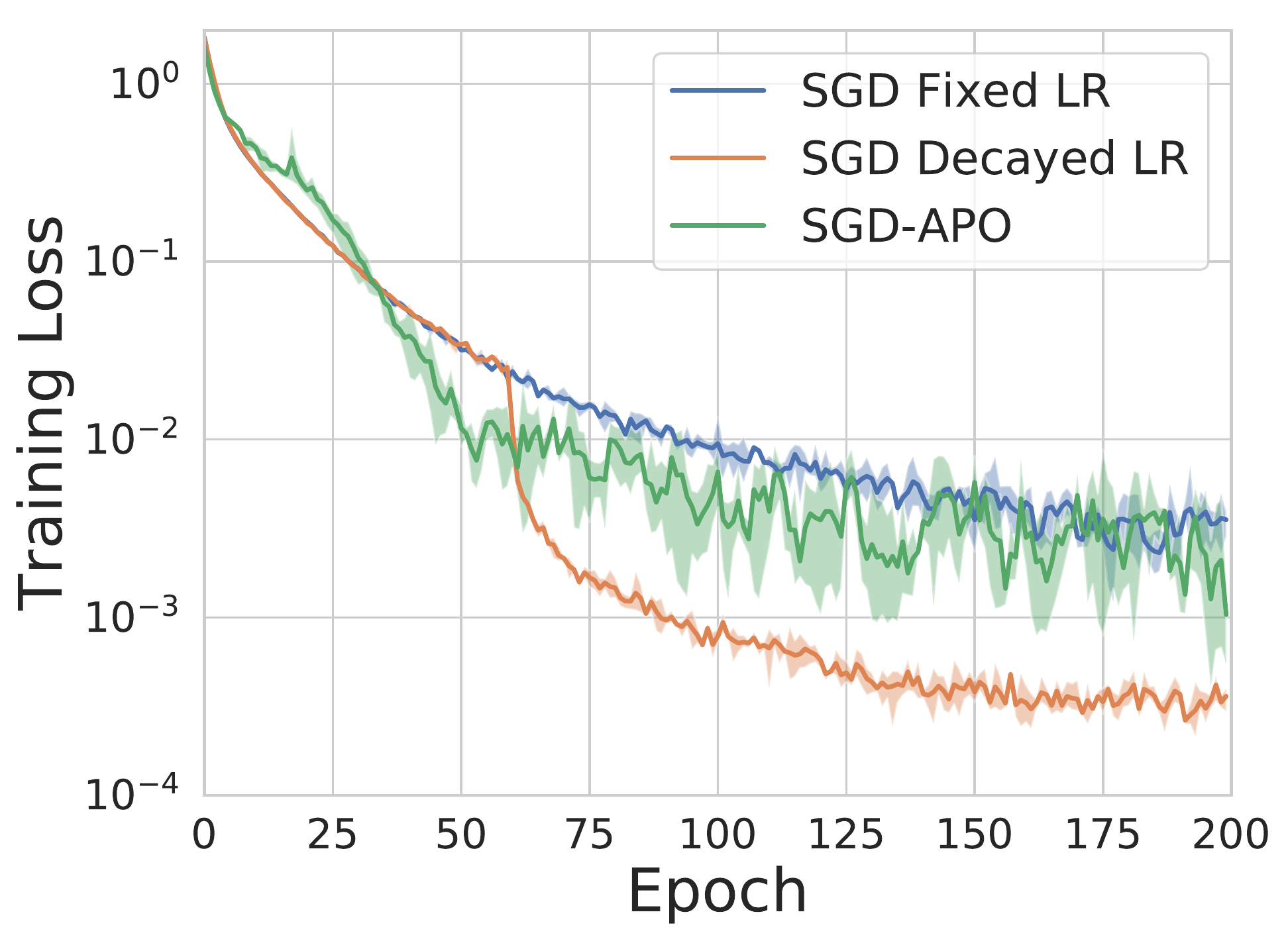}&
\includegraphics[width=0.3\linewidth]{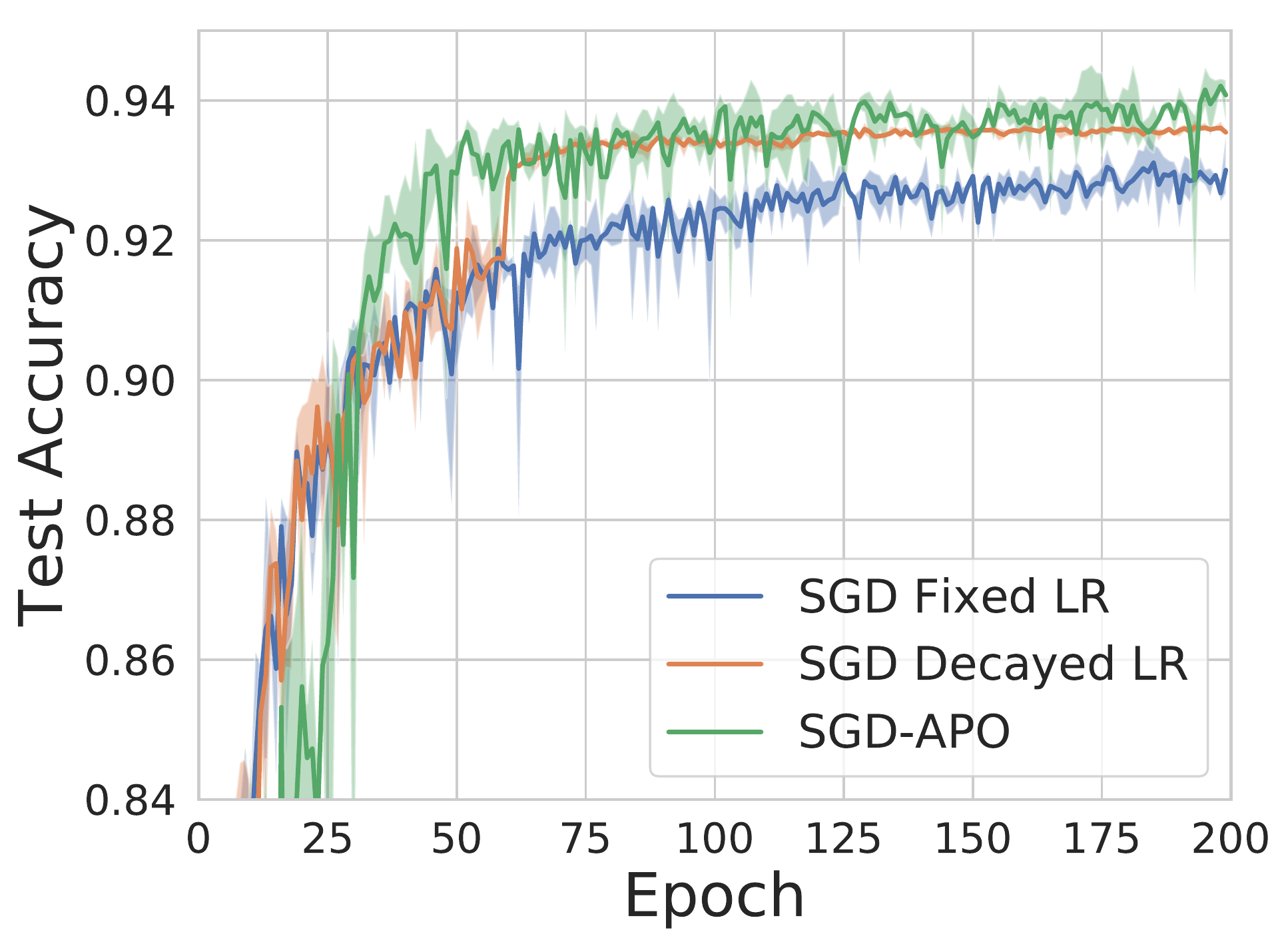}&
\includegraphics[width=0.3\linewidth]{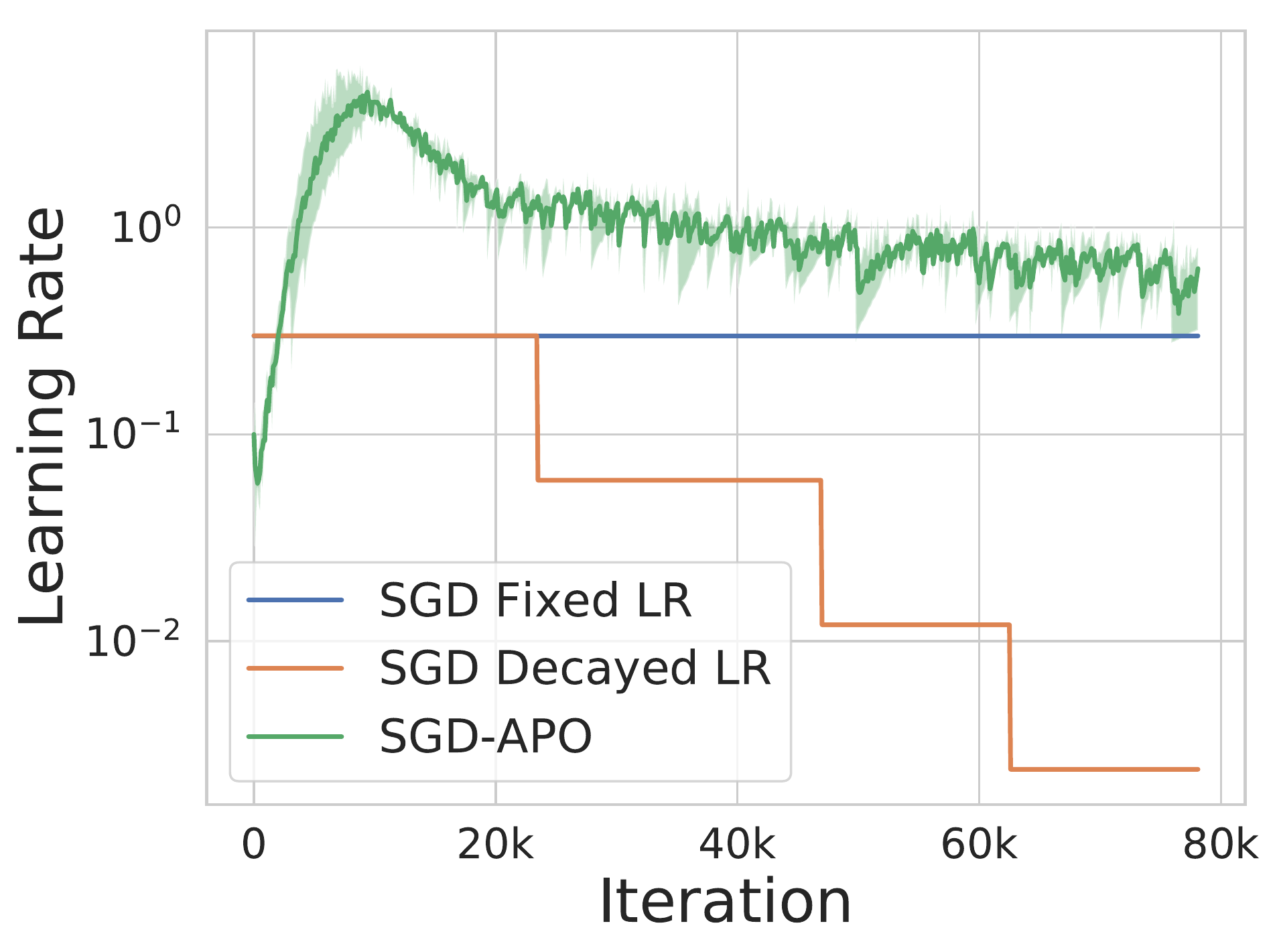}
\end{tabular}
\vspace{-2mm}
\caption{ResNet34 on CIFAR-10, using SGD. The shaded regions show the min/max values over 4 random restarts.}
\vspace{-4mm}
\label{fig:resnet34-cifar10-plain-sgd}
\end{figure}
% %%%%%%%%%%%%%%%%%%%%%%%%%%%%%%%%%%%%%%%%%%%%%%%%%%%%%%%

% ResNet34, CIFAR-10, SGDm
% %%%%%%%%%%%%%%%%%%%%%%%%%%%%%%%%%%%%%%%%%%%%%%%%%%%%%%%
\begin{figure}[H]
\centering
\begin{tabular}{ccc}
\includegraphics[width=0.3\linewidth]{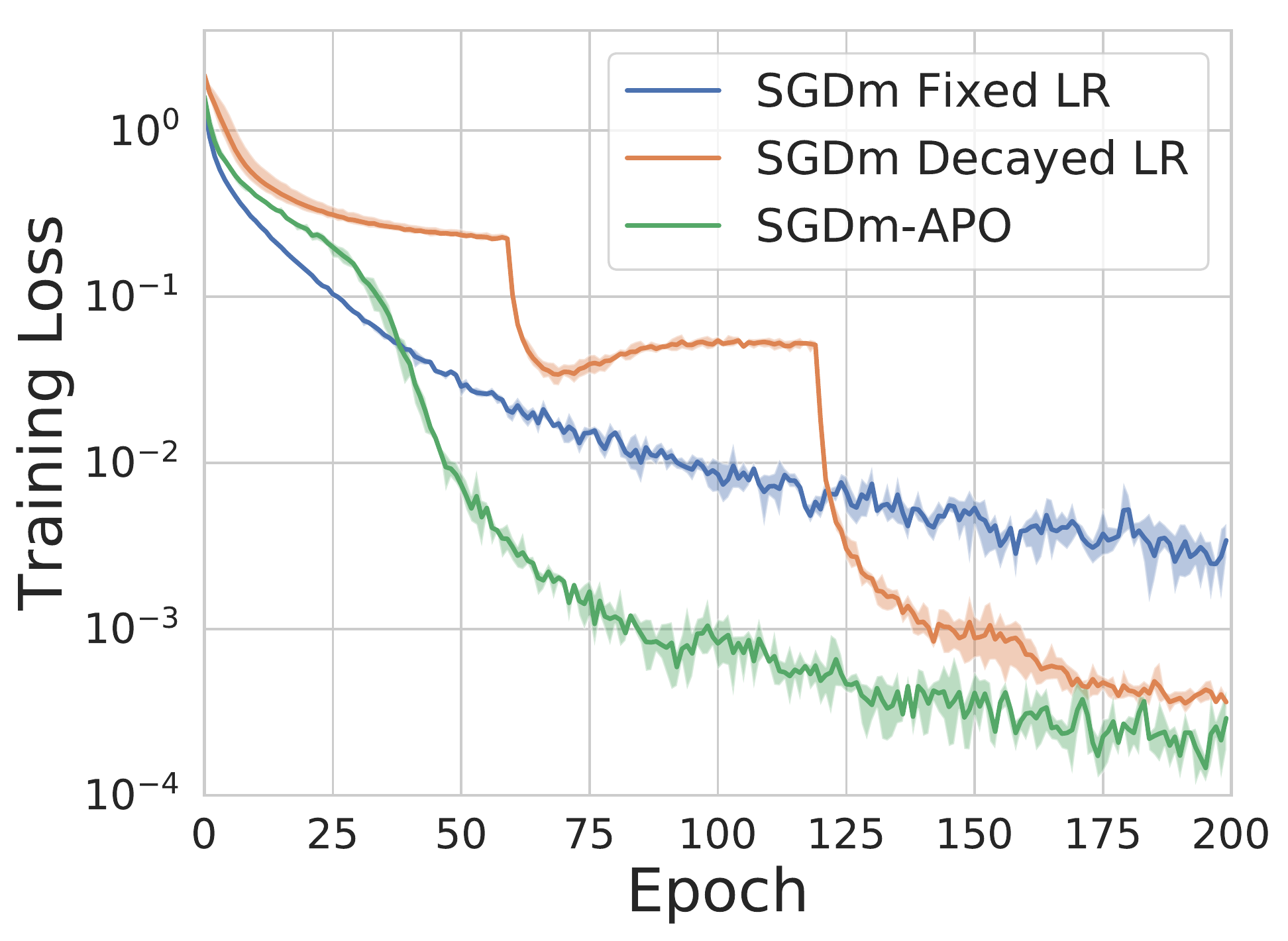}&
\includegraphics[width=0.3\linewidth]{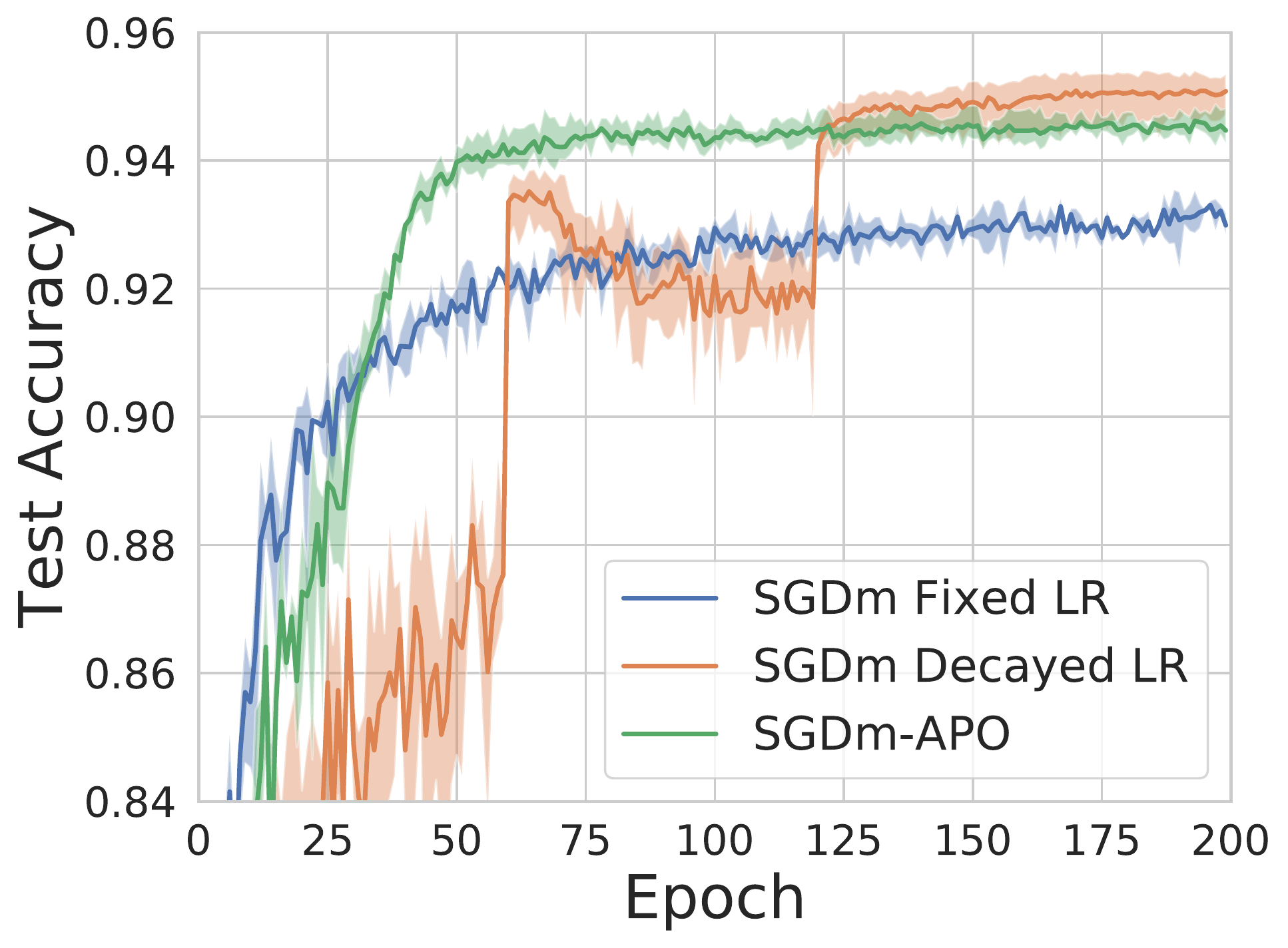}&
\includegraphics[width=0.3\linewidth]{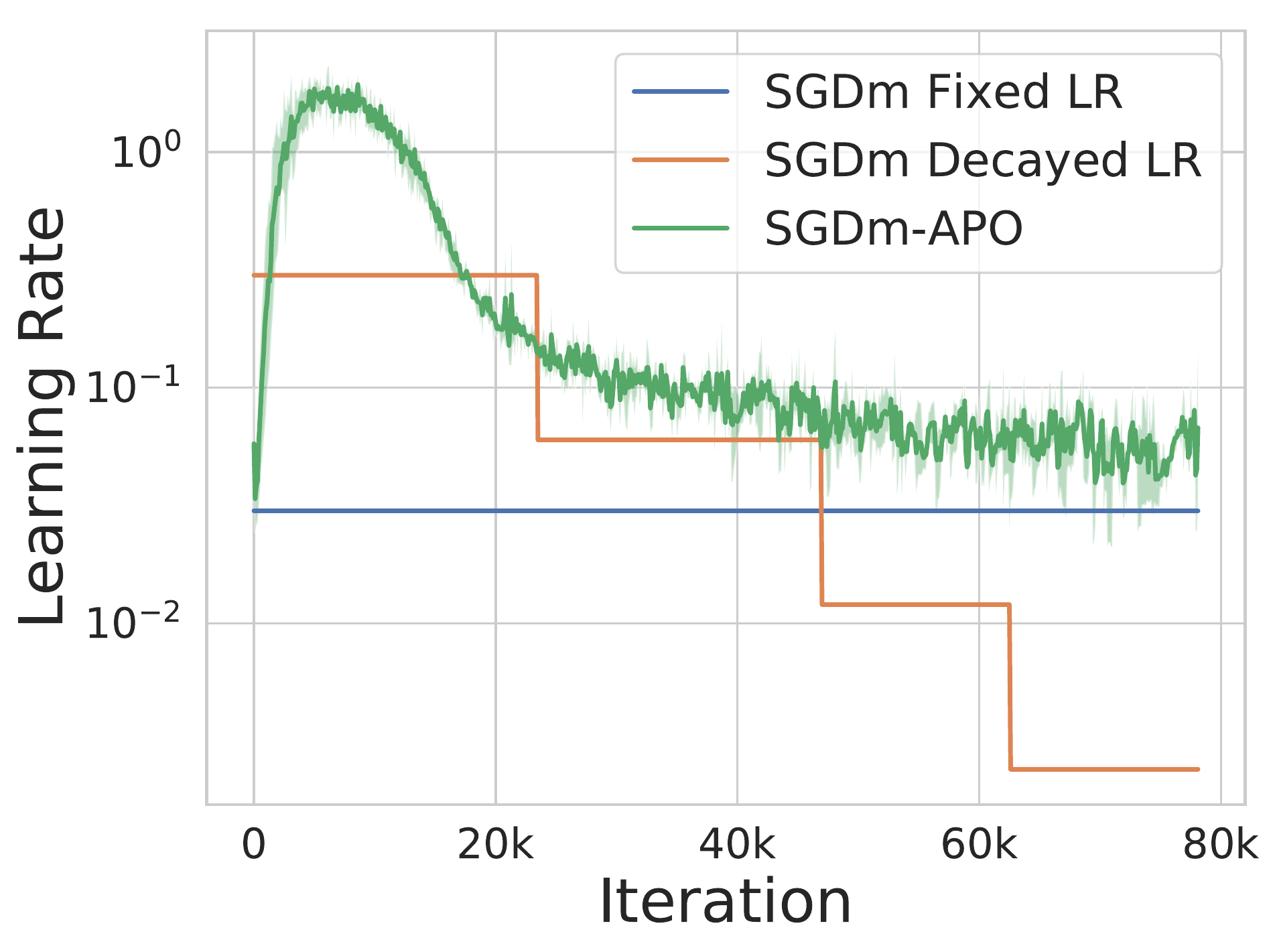}
\end{tabular}
\vspace{-2mm}
\caption{ResNet34 on CIFAR-10, using SGDm. The shaded regions show the min/max values over 4 random restarts.}
\vspace{-4mm}
\label{fig:resnet34-cifar10-sgdm}
\end{figure}
% %%%%%%%%%%%%%%%%%%%%%%%%%%%%%%%%%%%%%%%%%%%%%%%%%%%%%%%

% ResNet34, CIFAR-10, RMSprop
% %%%%%%%%%%%%%%%%%%%%%%%%%%%%%%%%%%%%%%%%%%%%%%%%%%%%%%%
\begin{figure}[H]
\centering
\begin{tabular}{ccc}
\includegraphics[width=0.3\linewidth]{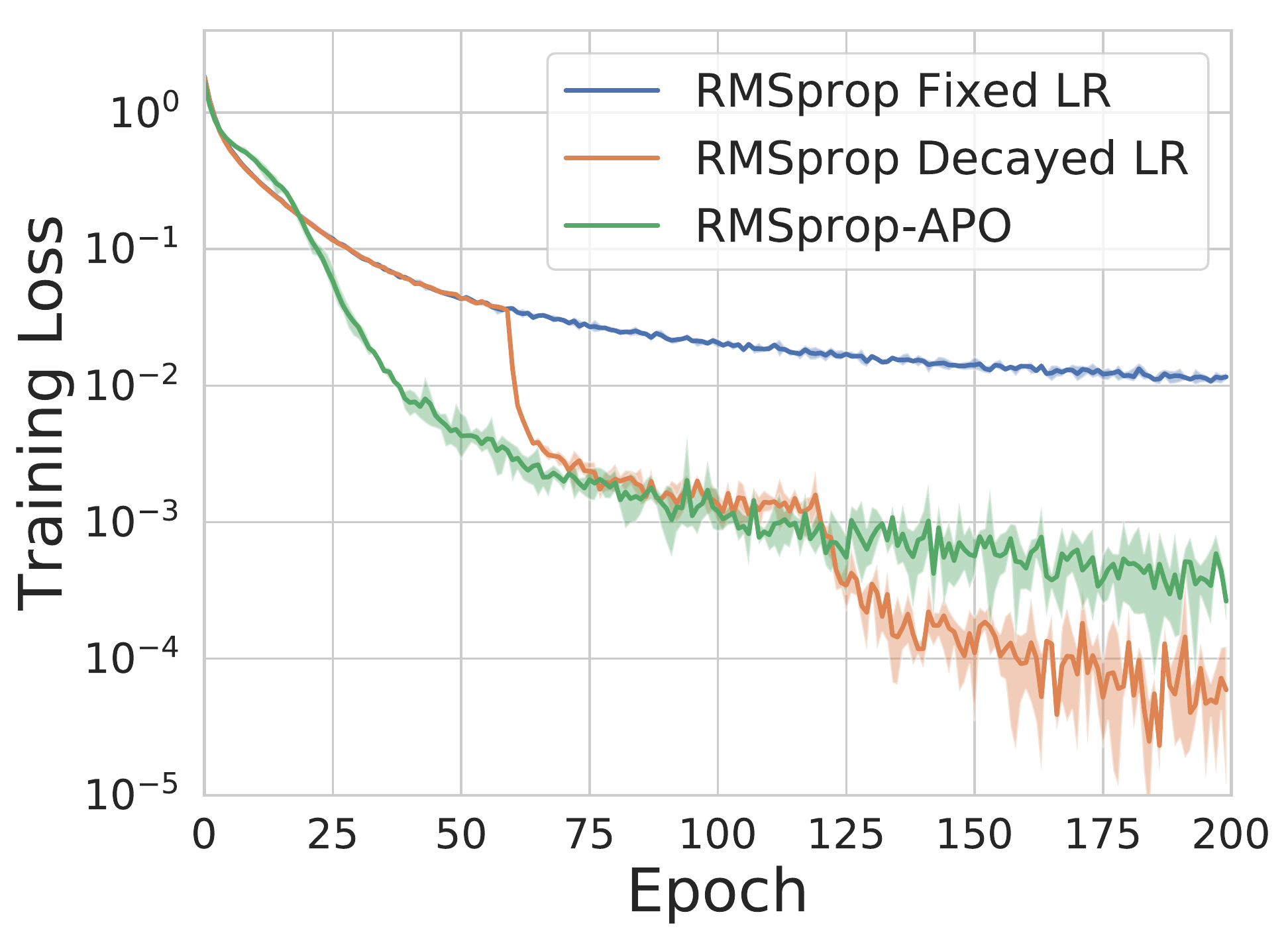}&
\includegraphics[width=0.3\linewidth]{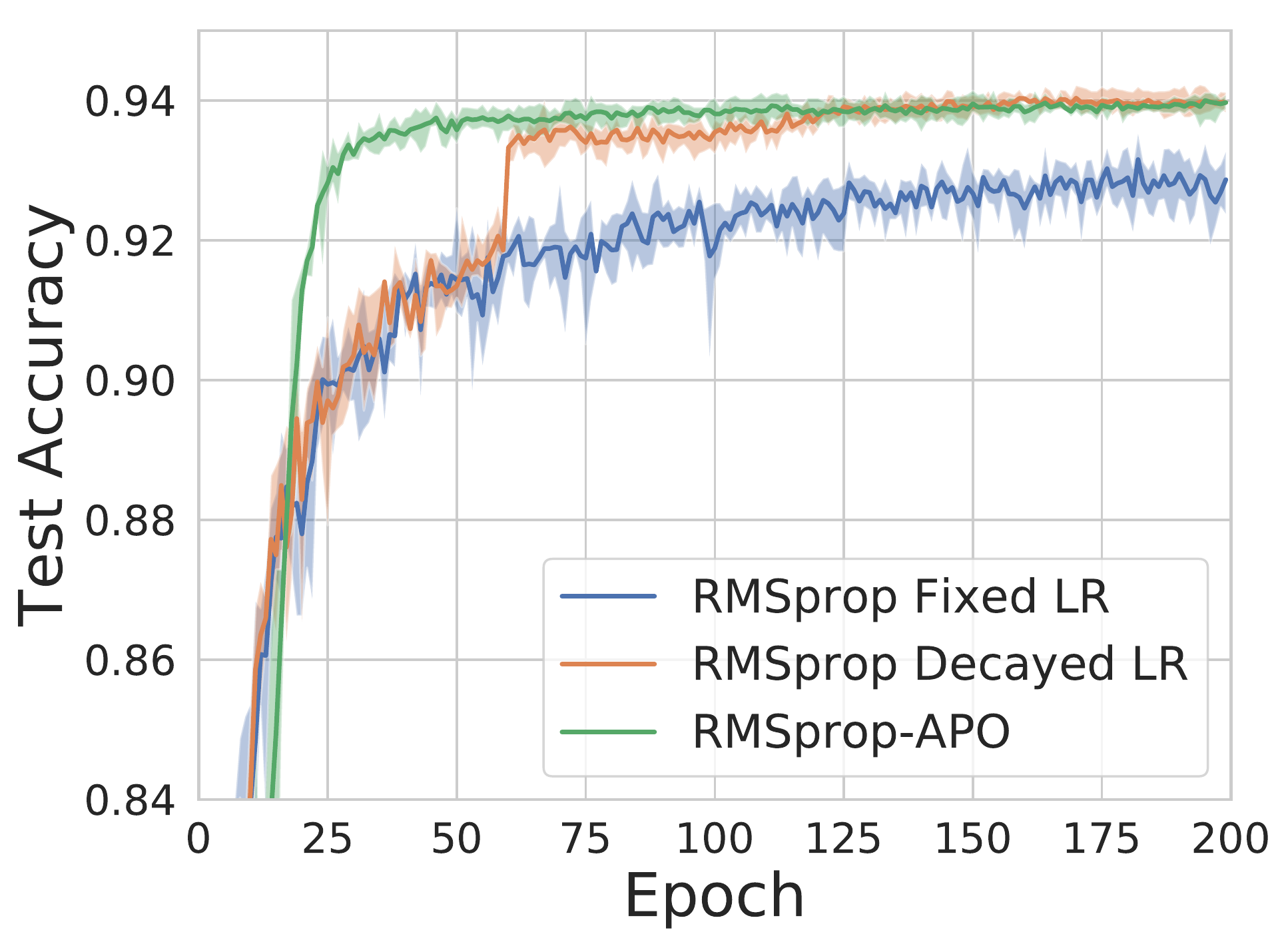}&
\includegraphics[width=0.3\linewidth]{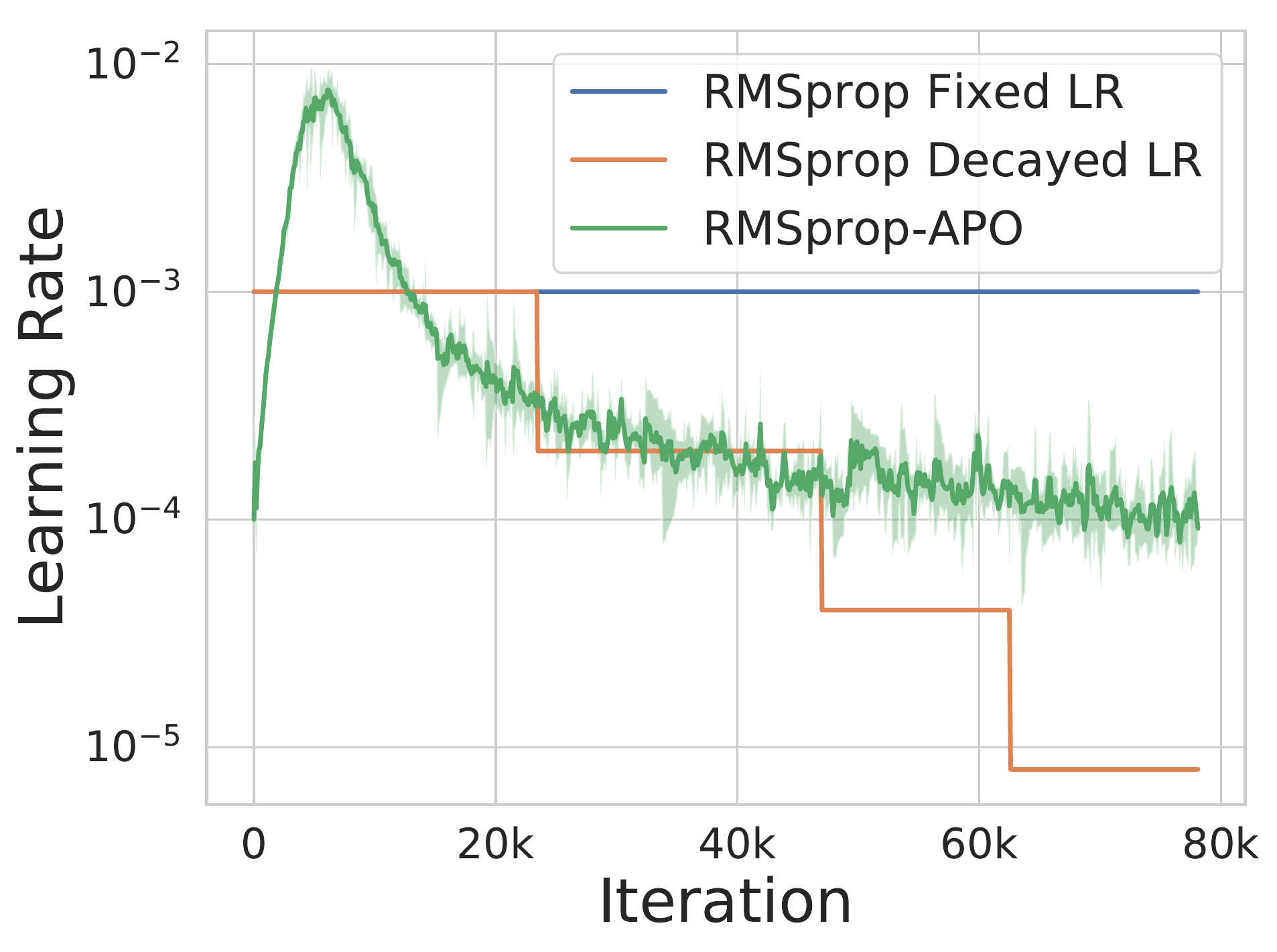}
\end{tabular}
\vspace{-2mm}
\caption{ResNet34 on CIFAR-10, using RMSprop. The shaded regions show the min/max values over 4 random restarts.}
\vspace{-4mm}
\label{fig:resnet34-cifar10-rmsprop}
\end{figure}
% %%%%%%%%%%%%%%%%%%%%%%%%%%%%%%%%%%%%%%%%%%%%%%%%%%%%%%%

% ResNet34, CIFAR-10, Adam
% %%%%%%%%%%%%%%%%%%%%%%%%%%%%%%%%%%%%%%%%%%%%%%%%%%%%%%%
\begin{figure}[H]
\centering
\begin{tabular}{ccc}
\includegraphics[width=0.3\linewidth]{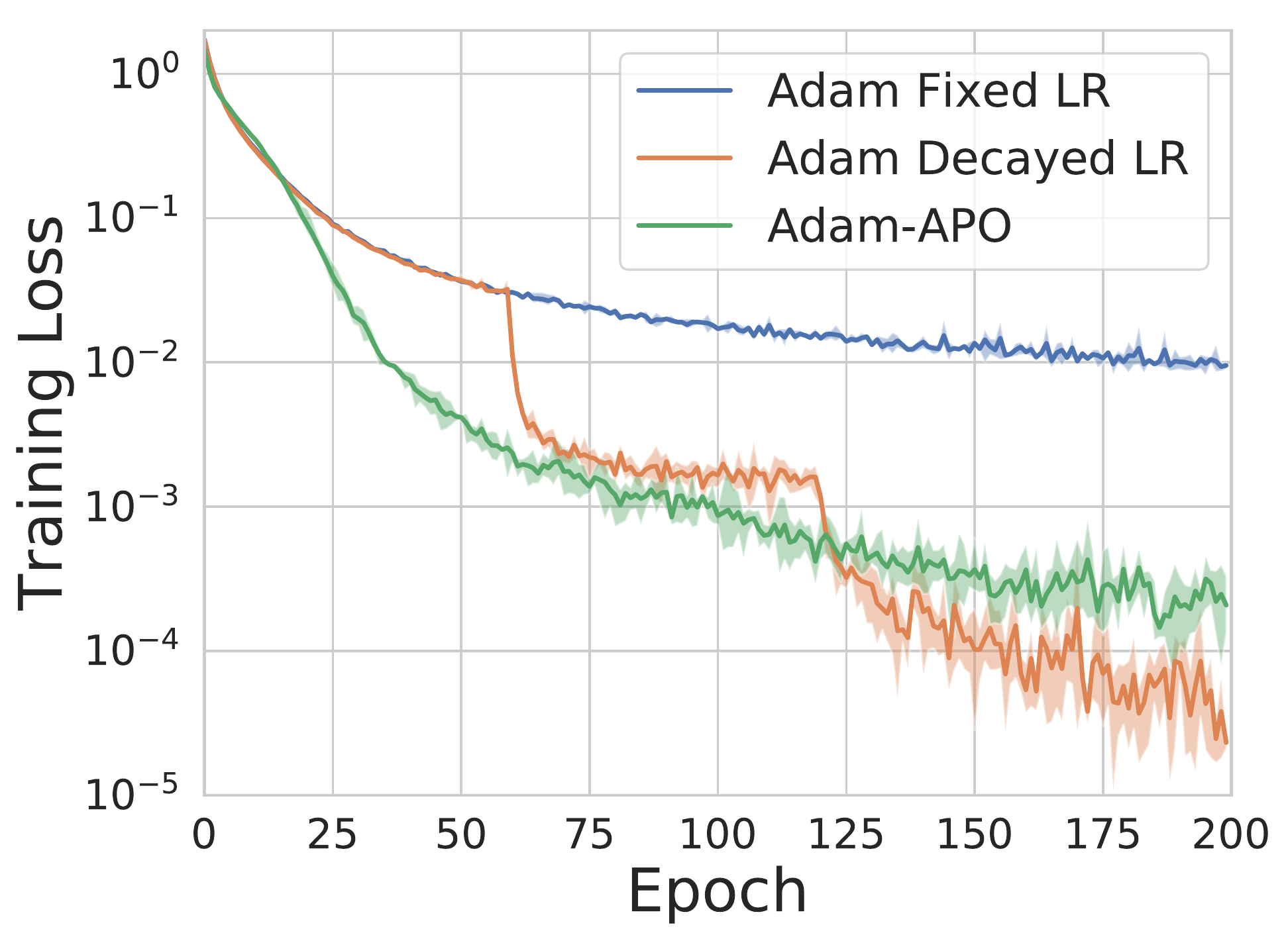}&
\includegraphics[width=0.3\linewidth]{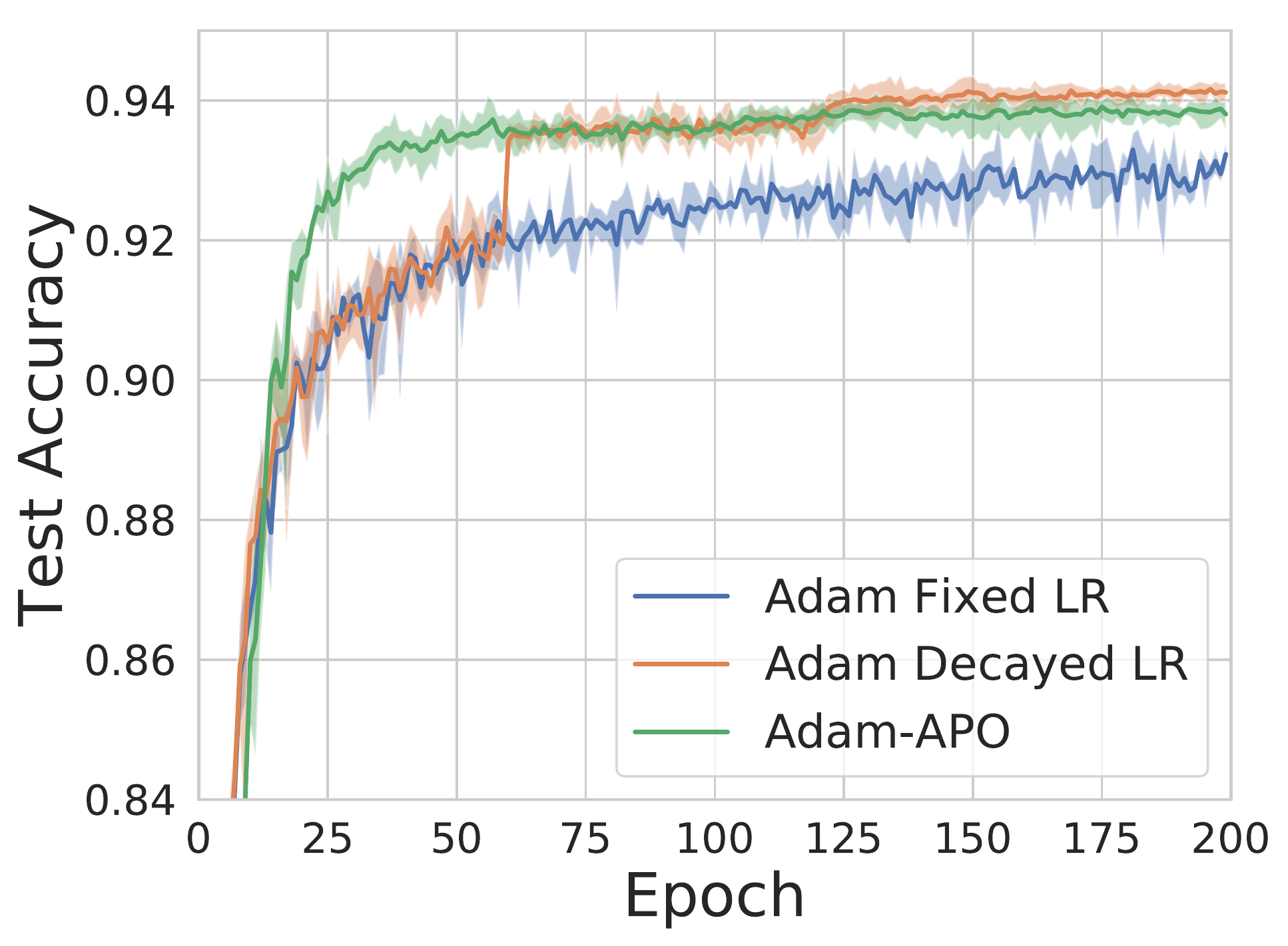}&
\includegraphics[width=0.3\linewidth]{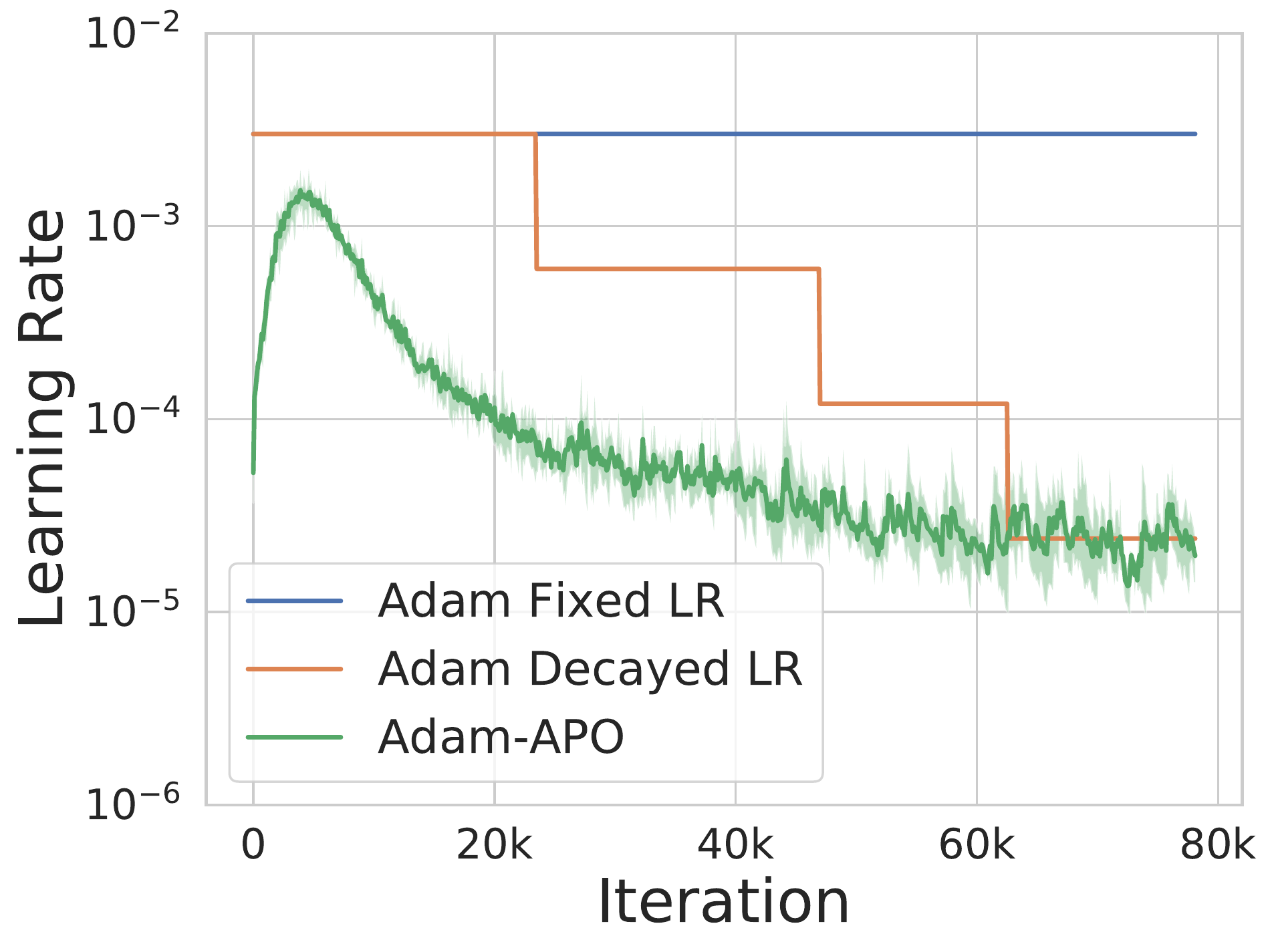}
\end{tabular}
\vspace{-2mm}
\caption{ResNet34 on CIFAR-10, using Adam. The shaded regions show the min/max values over 4 random restarts.}
\vspace{-4mm}
\label{fig:resnet34-cifar10-adam}
\end{figure}
% %%%%%%%%%%%%%%%%%%%%%%%%%%%%%%%%%%%%%%%%%%%%%%%%%%%%%%%

% ResNet32, SGD
% %%%%%%%%%%%%%%%%%%%%%%%%%%%%%%%%%%%%%%%%%%%%%%%%%%%%%%%
%\begin{figure}[H]
%\centering
%\begin{tabular}{ccc}
%\includegraphics[width=0.3\linewidth]{icml_figures/resnet32_cifar10/resnet32_cifar10_plain_sgd/resnet32_cifar10_plain_sgd_train_loss.pdf}&
%\includegraphics[width=0.3\linewidth]{icml_figures/resnet32_cifar10/resnet32_cifar10_plain_sgd/resnet32_cifar10_plain_sgd_test_acc.pdf}&
%\includegraphics[width=0.3\linewidth]{icml_figures/resnet32_cifar10/resnet32_cifar10_plain_sgd/resnet32_cifar10_plain_sgd_hparams.pdf}
%\end{tabular}
%\vspace{-2mm}
%\caption{\textbf{ResNet32 on CIFAR-10, using SGD.} The shaded regions show the min/max values over 4 random restarts.}
%\vspace{-4mm}
%\label{fig:resnet32-cifar10-sgd}
%\end{figure}
% %%%%%%%%%%%%%%%%%%%%%%%%%%%%%%%%%%%%%%%%%%%%%%%%%%%%%%%

% ResNet32, SGDm
% %%%%%%%%%%%%%%%%%%%%%%%%%%%%%%%%%%%%%%%%%%%%%%%%%%%%%%%
\begin{figure}[H]
\centering
\begin{tabular}{ccc}
\includegraphics[width=0.3\linewidth]{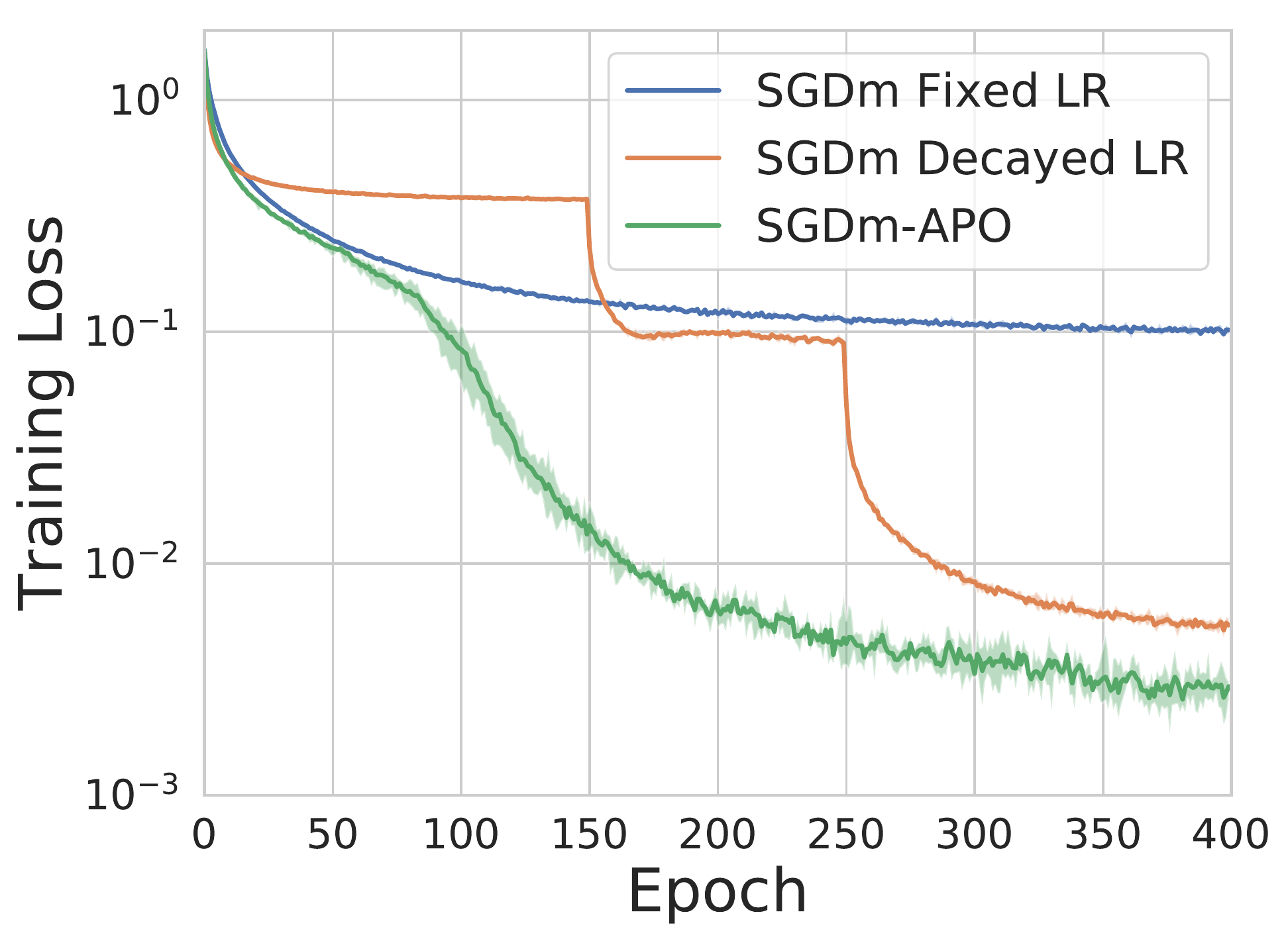}&
\includegraphics[width=0.3\linewidth]{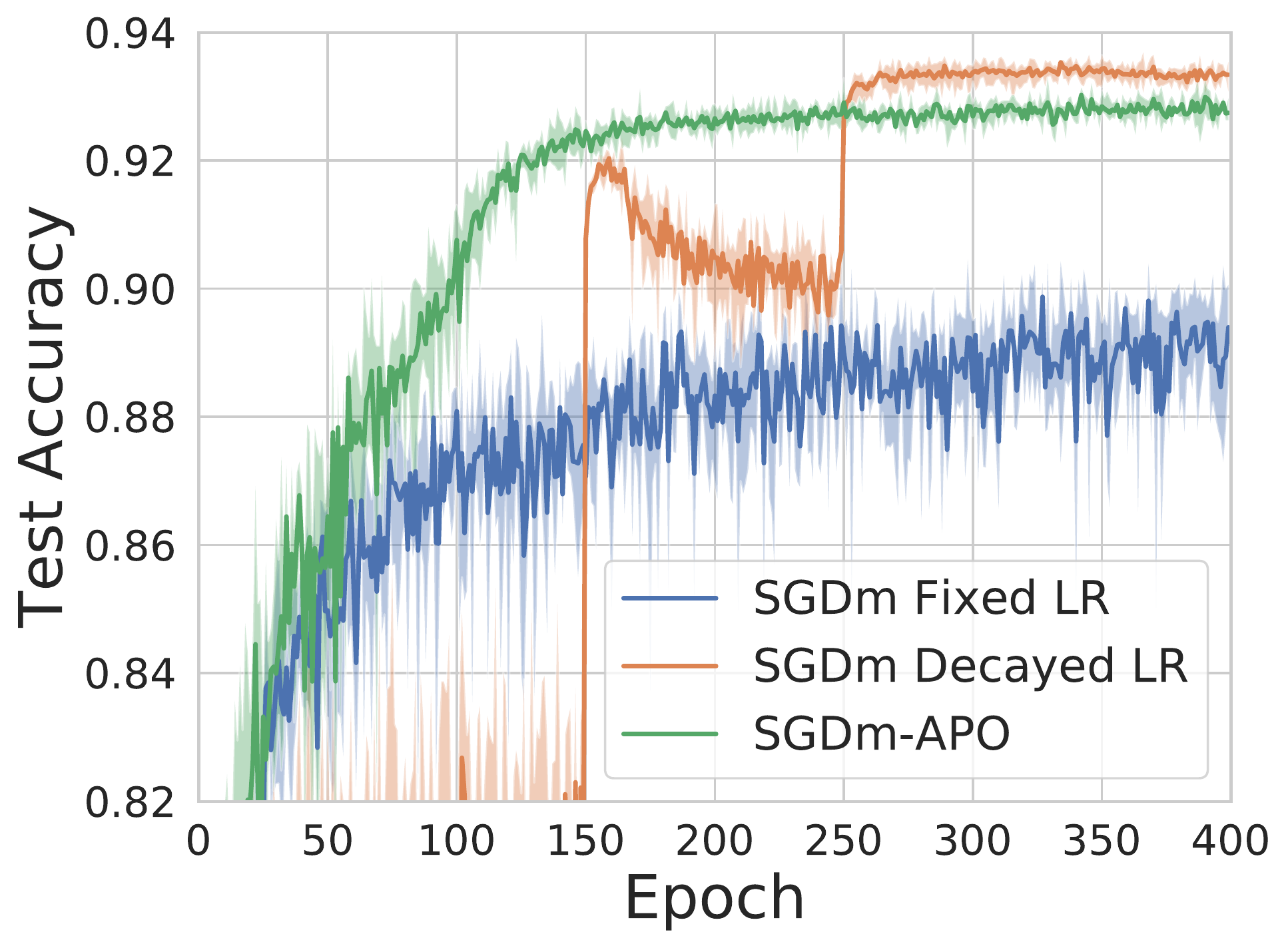}&
\includegraphics[width=0.3\linewidth]{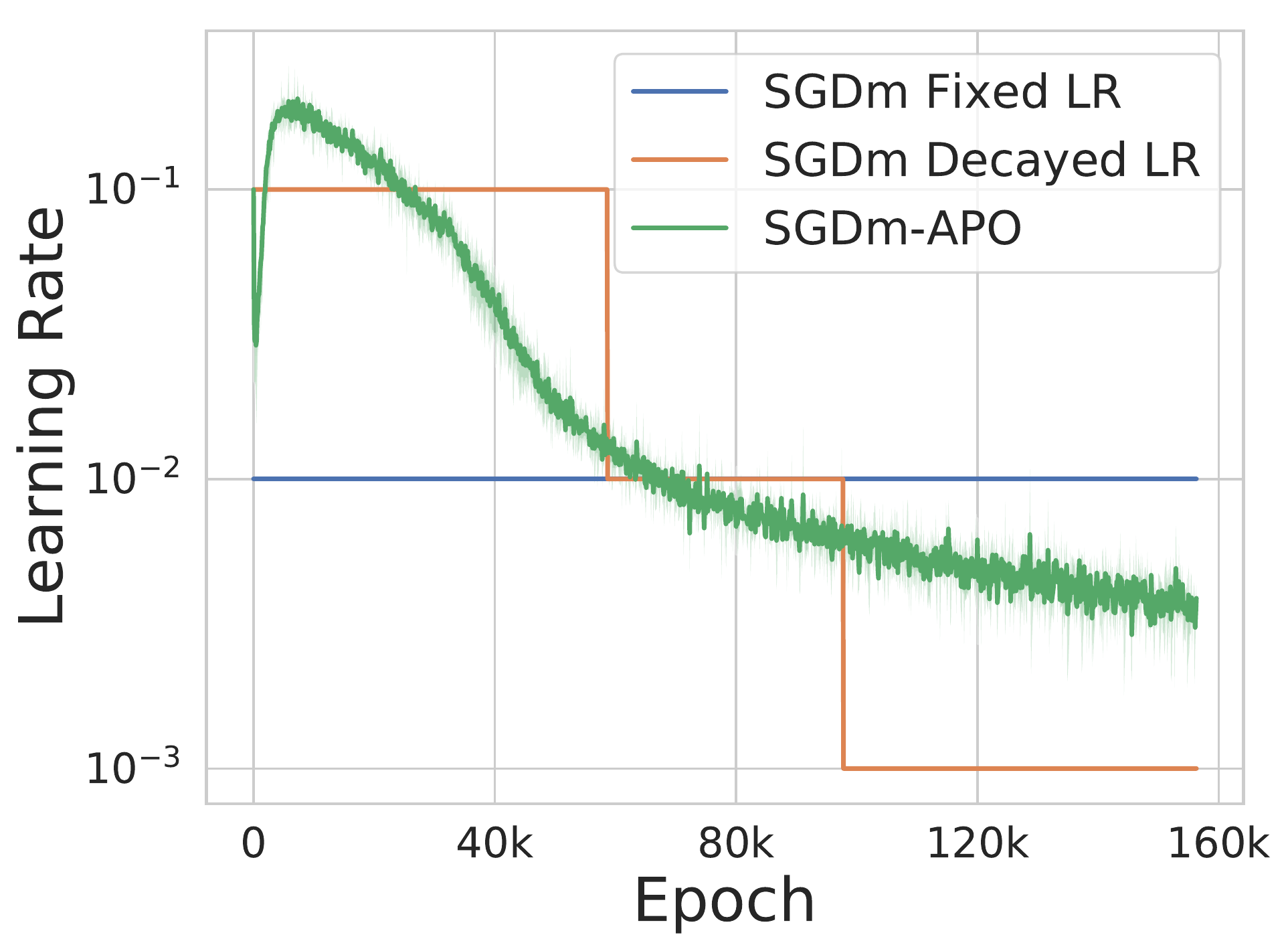}
\end{tabular}
\vspace{-2mm}
\caption{ResNet32 on CIFAR-10, using SGDm. The shaded regions show the min/max values over 4 random restarts.}
\vspace{-4mm}
\label{fig:resnet32-cifar10-sgdm}
\end{figure}
% %%%%%%%%%%%%%%%%%%%%%%%%%%%%%%%%%%%%%%%%%%%%%%%%%%%%%%%

% ResNet32, CIFAR-10, RMSprop
% %%%%%%%%%%%%%%%%%%%%%%%%%%%%%%%%%%%%%%%%%%%%%%%%%%%%%%%
\begin{figure}[H]
\centering
\begin{tabular}{ccc}
\includegraphics[width=0.3\linewidth]{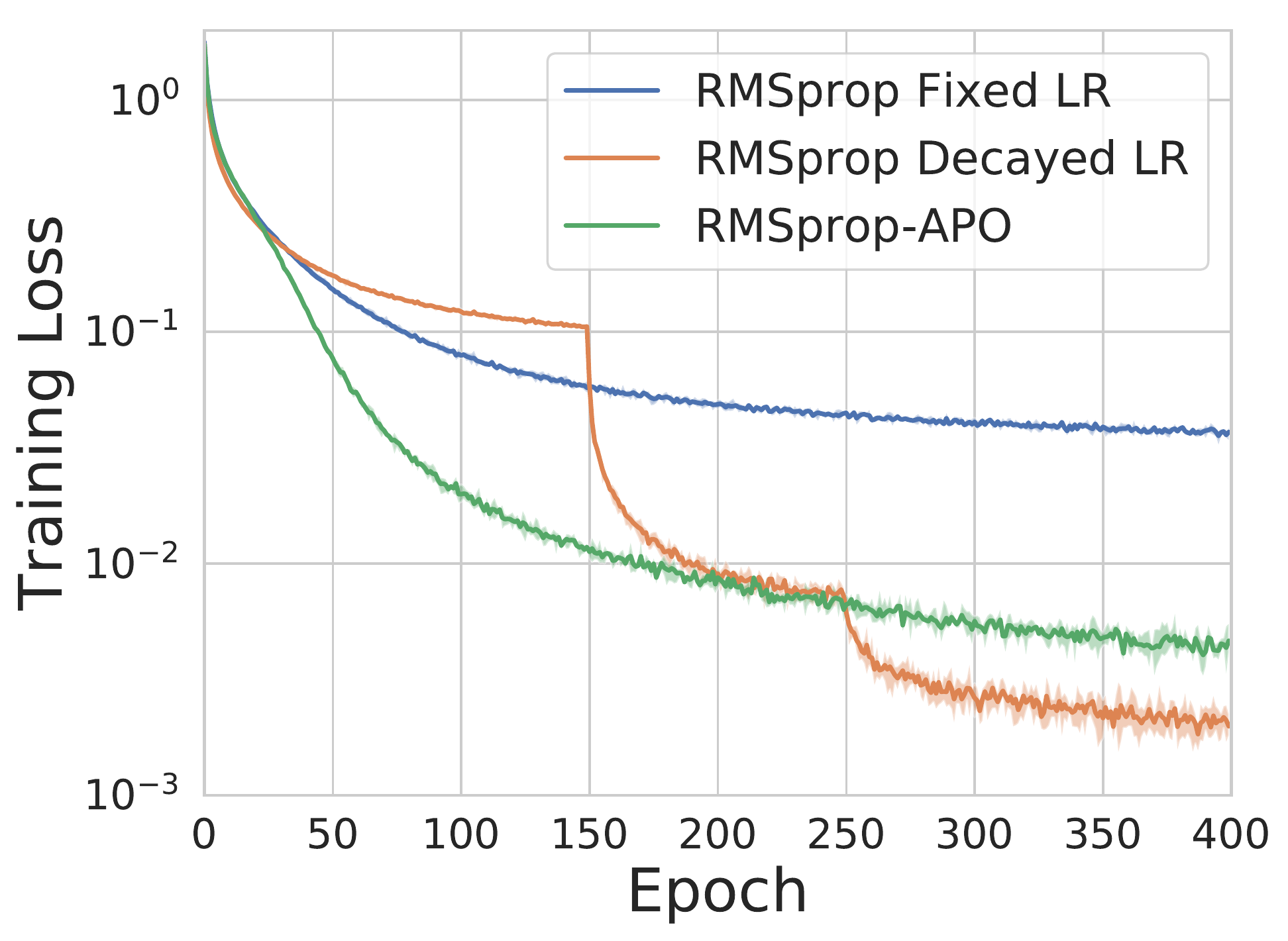}&
\includegraphics[width=0.3\linewidth]{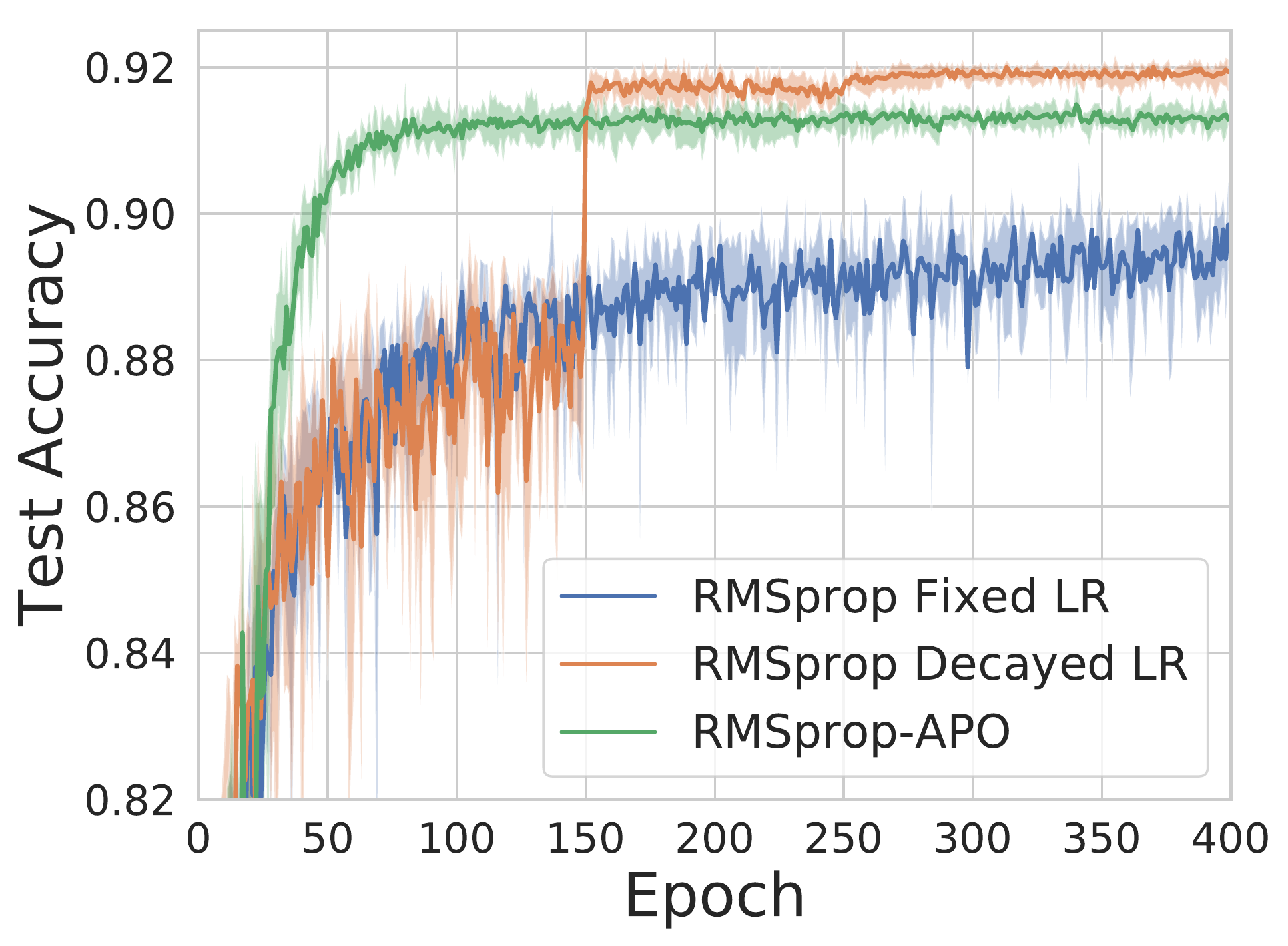}&
\includegraphics[width=0.3\linewidth]{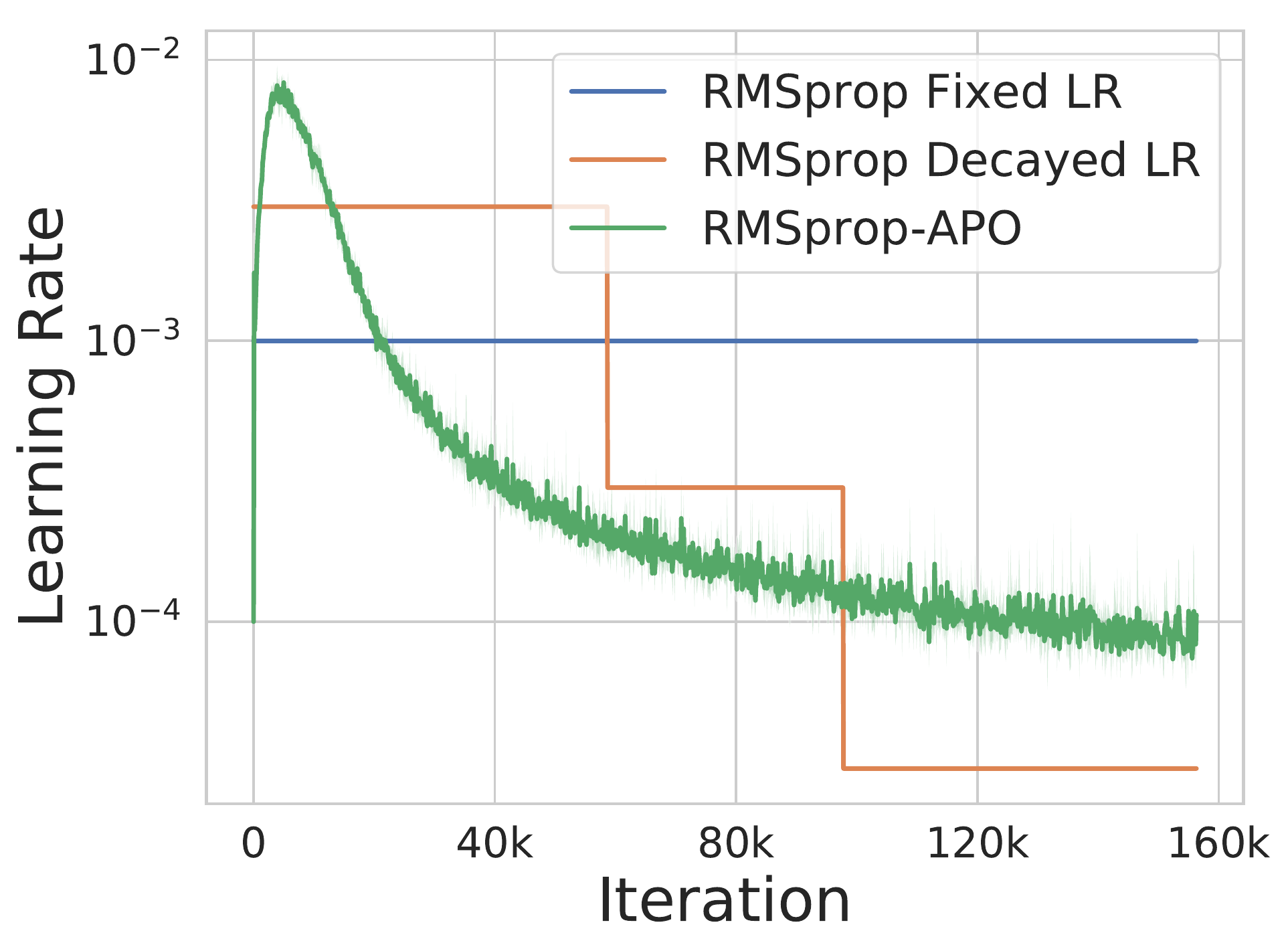}
\end{tabular}
\vspace{-2mm}
\caption{ResNet32 on CIFAR-10, using RMSprop. The shaded regions show the min/max values over 4 random restarts.}
\vspace{-4mm}
\label{fig:resnet32-cifar10-rmsprop}
\end{figure}
% %%%%%%%%%%%%%%%%%%%%%%%%%%%%%%%%%%%%%%%%%%%%%%%%%%%%%%%

% ResNet32, CIFAR-10, Adam
% %%%%%%%%%%%%%%%%%%%%%%%%%%%%%%%%%%%%%%%%%%%%%%%%%%%%%%%
\begin{figure}[H]
\centering
\begin{tabular}{ccc}
\includegraphics[width=0.3\linewidth]{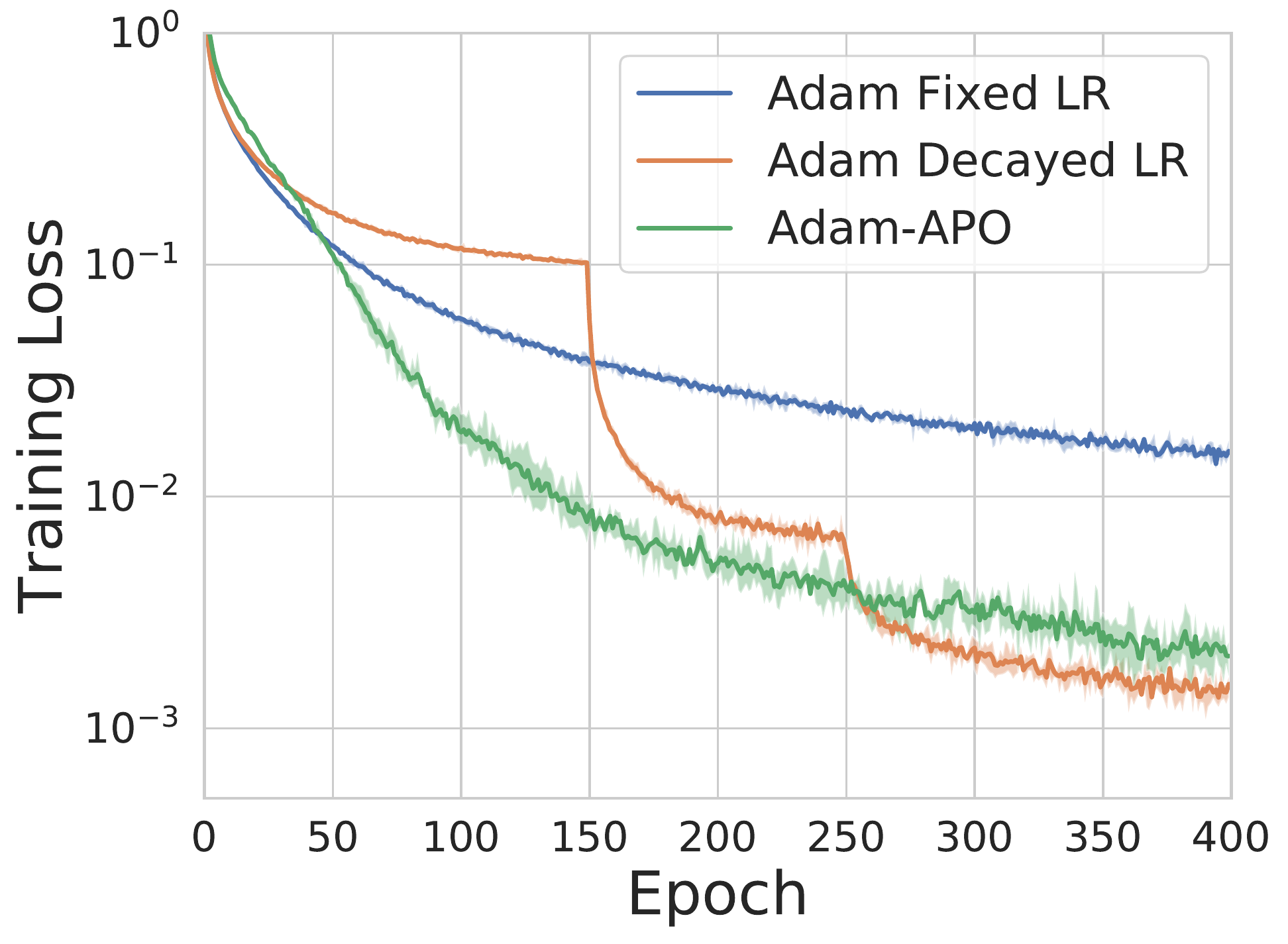}&
\includegraphics[width=0.3\linewidth]{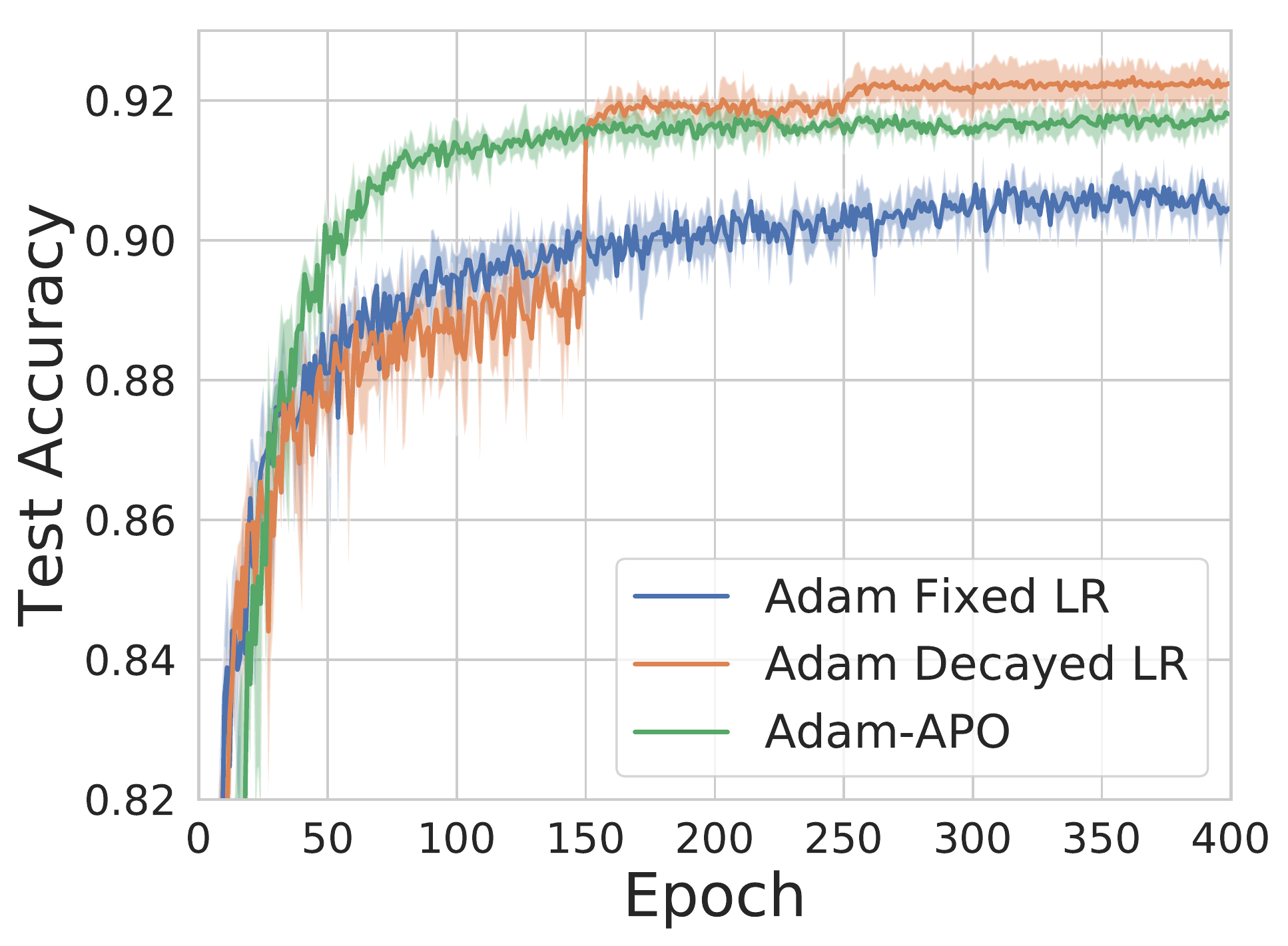}&
\includegraphics[width=0.3\linewidth]{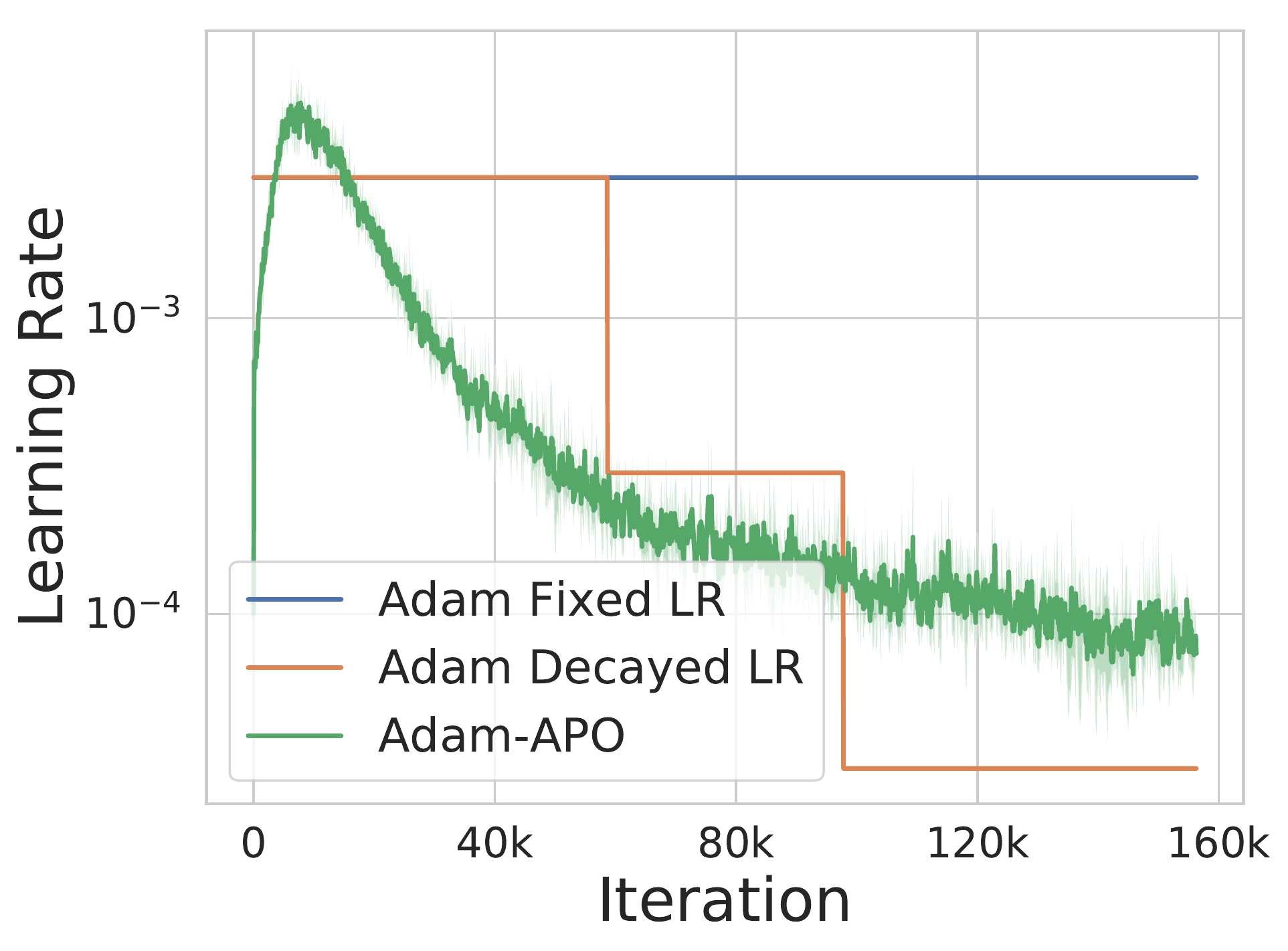}
\end{tabular}
\vspace{-2mm}
\caption{ResNet32 on CIFAR-10, using Adam. The shaded regions show the min/max values over 4 random restarts.}
\vspace{-4mm}
\label{fig:resnet32-cifar10-adam}
\end{figure}
% %%%%%%%%%%%%%%%%%%%%%%%%%%%%%%%%%%%%%%%%%%%%%%%%%%%%%%%

% WRN 28-10, CIFAR-10, Plain SGD
% %%%%%%%%%%%%%%%%%%%%%%%%%%%%%%%%%%%%%%%%%%%%%%%%%%%%%%%
% \begin{figure}[H]
% \centering
% \begin{tabular}{ccc}
% \includegraphics[width=0.3\linewidth]{icml_figures/wideresnet_cifar10/wrn_cifar10_plain_sgd/wideresnet_cifar10_plain_sgd_train_loss.pdf}&
% \includegraphics[width=0.3\linewidth]{icml_figures/wideresnet_cifar10/wrn_cifar10_plain_sgd/wideresnet_cifar10_plain_sgd_test_acc.pdf}&
% \includegraphics[width=0.3\linewidth]{icml_figures/wideresnet_cifar10/wrn_cifar10_plain_sgd/wideresnet_cifar10_plain_sgd_hparams.pdf}
% \end{tabular}
% \vspace{-2mm}
% \caption{\textbf{WideResNet 28-10 on CIFAR-10, using SGD.} The shaded regions show the min/max values over 4 random restarts.}
% \vspace{-4mm}
% \label{fig:wrn-cifar10-plain-sgd}
% \end{figure}
% %%%%%%%%%%%%%%%%%%%%%%%%%%%%%%%%%%%%%%%%%%%%%%%%%%%%%%%

% WRN 28-10, CIFAR-10, SGDm
% %%%%%%%%%%%%%%%%%%%%%%%%%%%%%%%%%%%%%%%%%%%%%%%%%%%%%%%
\begin{figure}[H]
\centering
\begin{tabular}{ccc}
\includegraphics[width=0.3\linewidth]{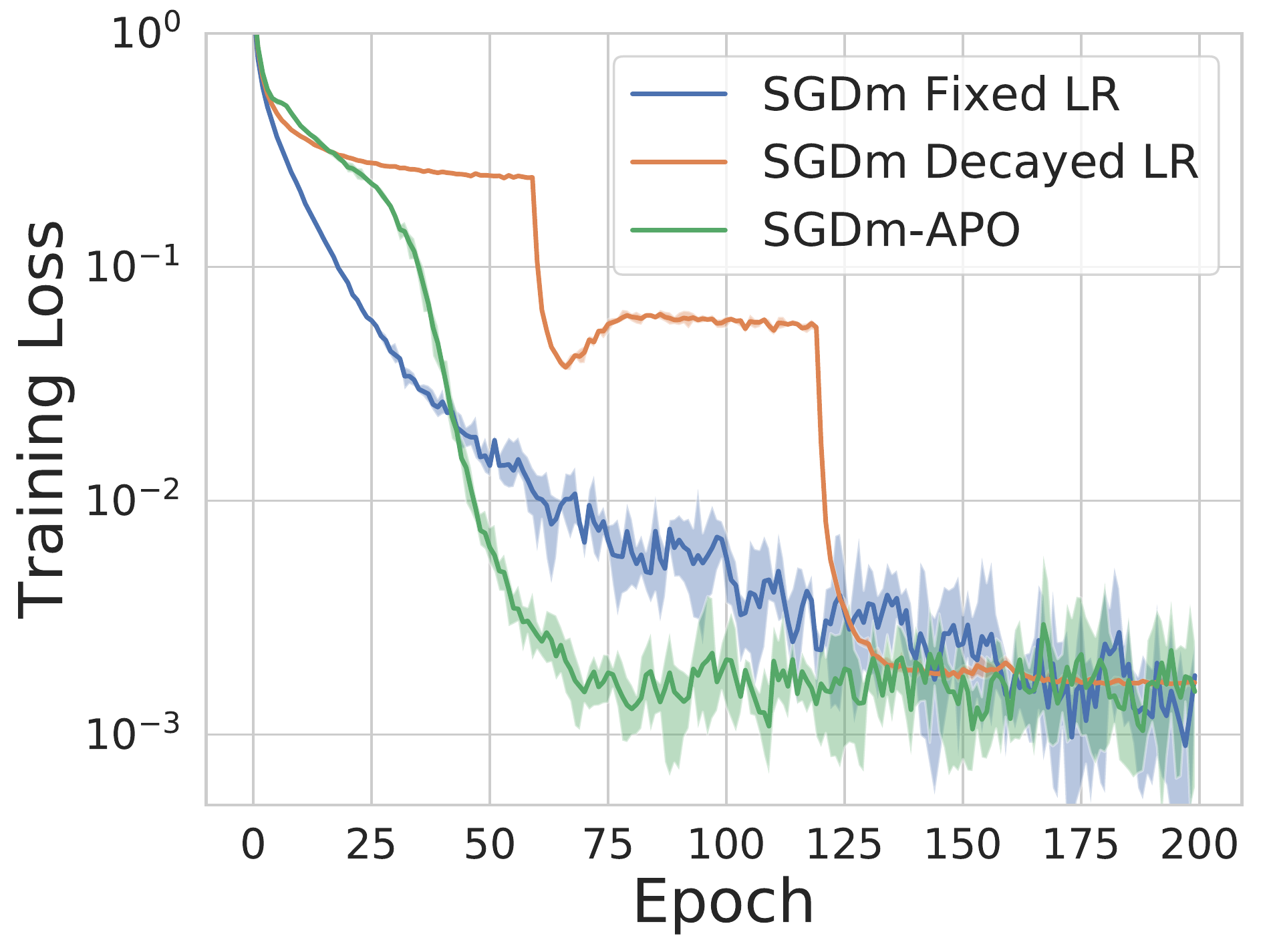}&
\includegraphics[width=0.3\linewidth]{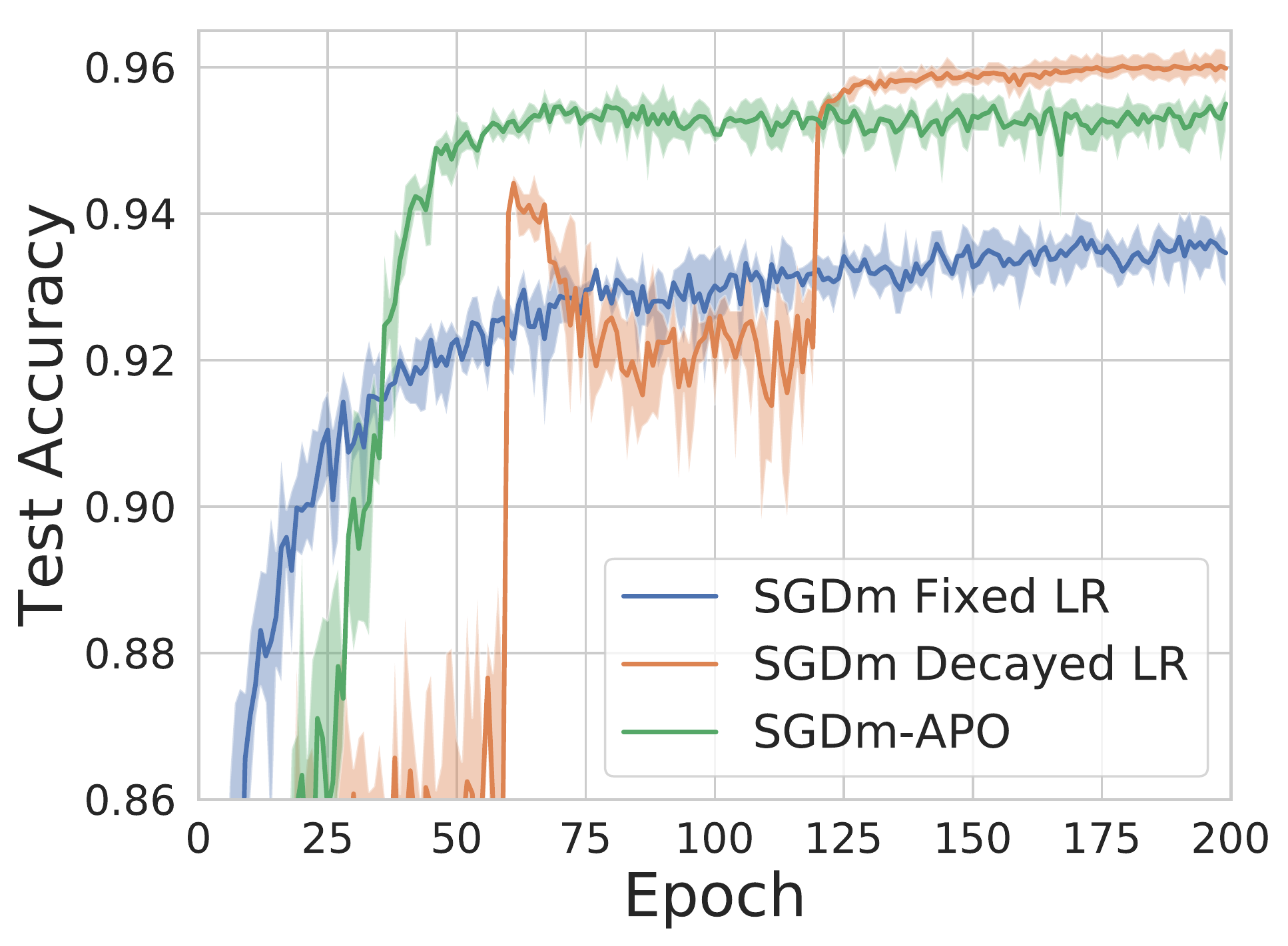}&
\includegraphics[width=0.3\linewidth]{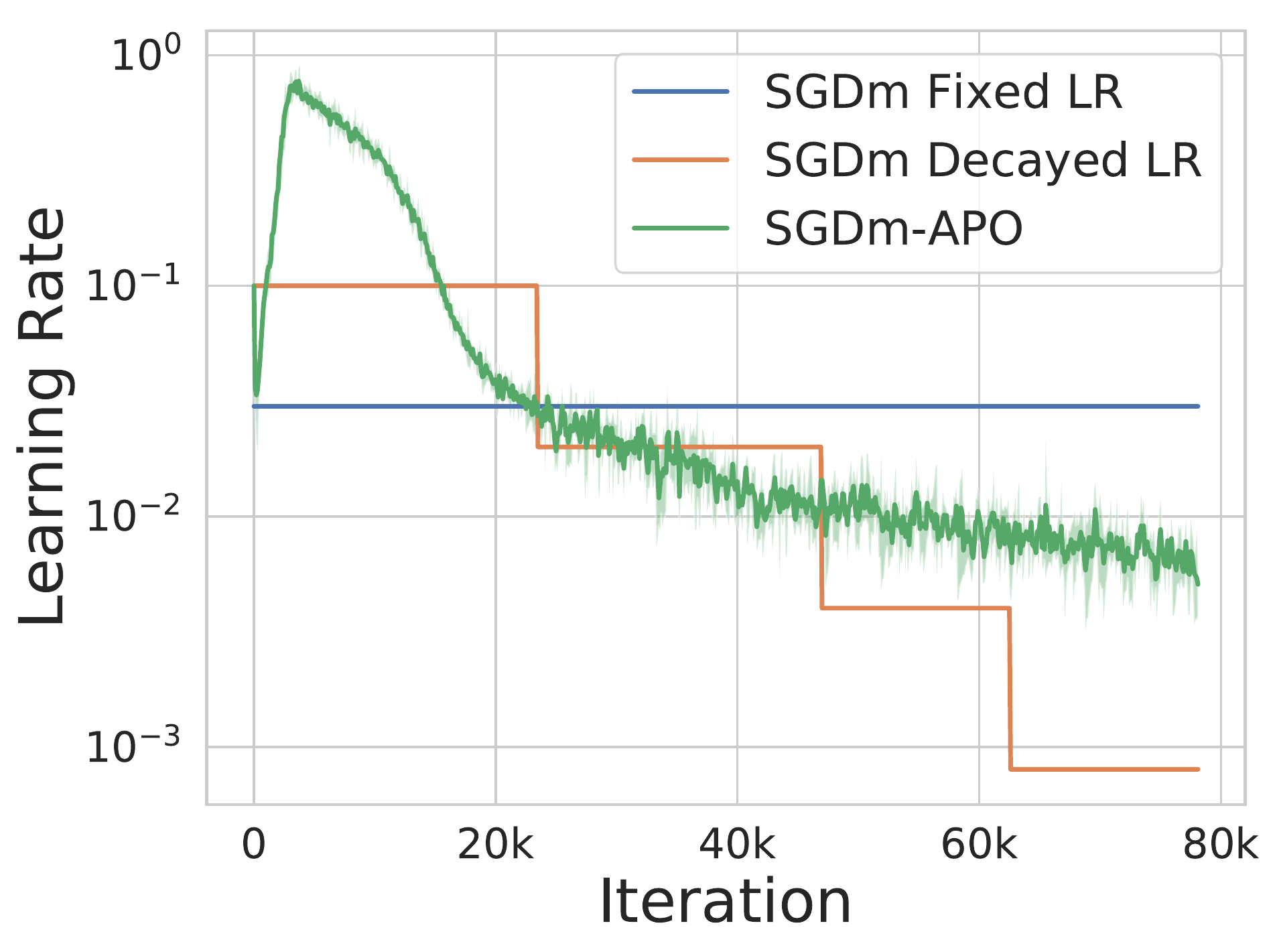}
\end{tabular}
\vspace{-2mm}
\caption{WideResNet 28-10 on CIFAR-10, using SGDm. The shaded regions show the min/max values over 4 random restarts.}
\vspace{-4mm}
\label{fig:wrn-cifar10-plain-sgdm}
\end{figure}
% %%%%%%%%%%%%%%%%%%%%%%%%%%%%%%%%%%%%%%%%%%%%%%%%%%%%%%%

% WRN 28-10, CIFAR-10, RMSprop
% %%%%%%%%%%%%%%%%%%%%%%%%%%%%%%%%%%%%%%%%%%%%%%%%%%%%%%%
\begin{figure}[H]
\centering
\begin{tabular}{ccc}
\includegraphics[width=0.3\linewidth]{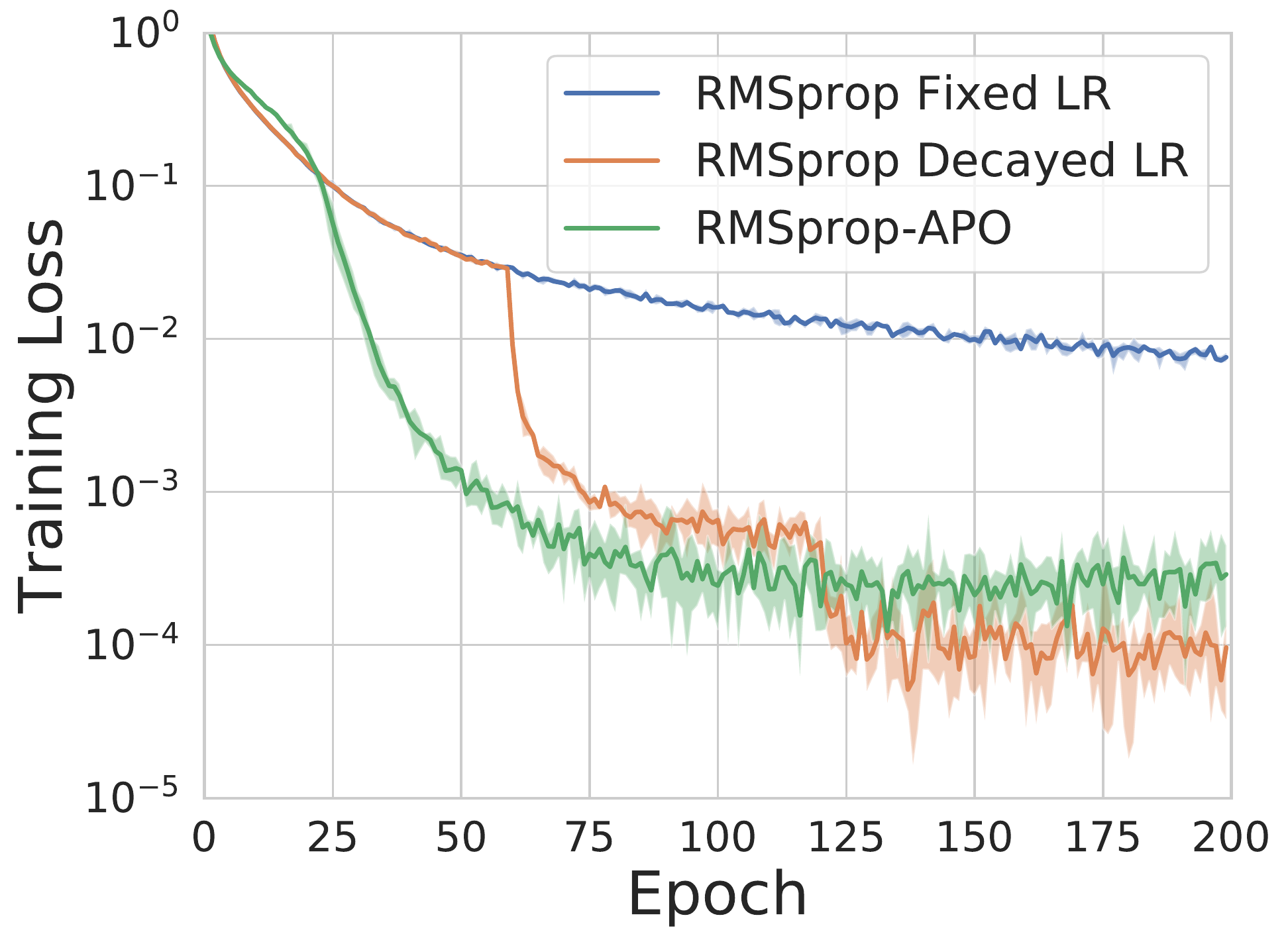}&
\includegraphics[width=0.3\linewidth]{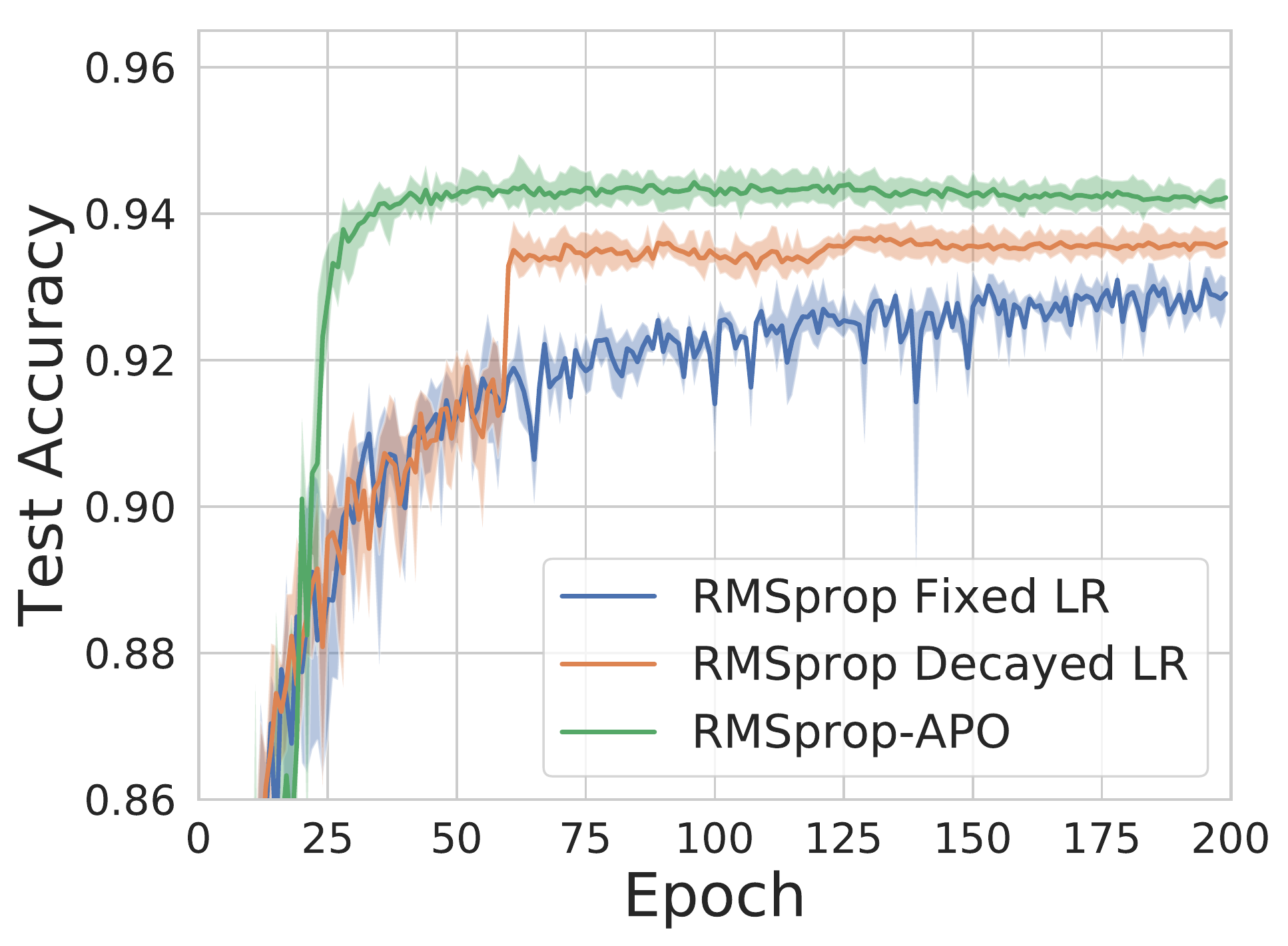}&
\includegraphics[width=0.3\linewidth]{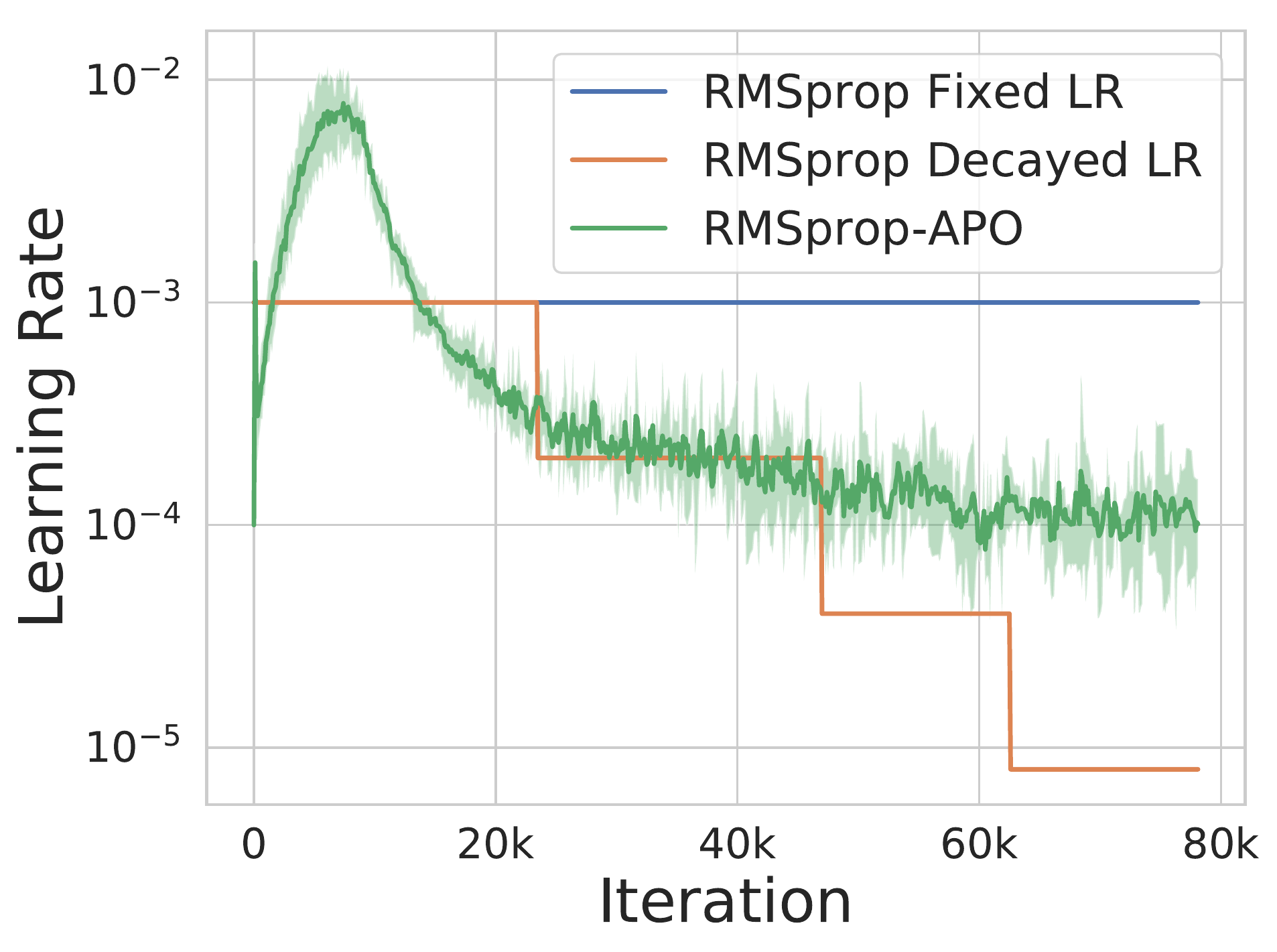}
\end{tabular}
\vspace{-2mm}
\caption{WideResNet 28-10 on CIFAR-10, using RMSprop. The shaded regions show the min/max values over 4 random restarts.}
\vspace{-4mm}
\label{fig:wrn-cifar10-rmsprop}
\end{figure}
% %%%%%%%%%%%%%%%%%%%%%%%%%%%%%%%%%%%%%%%%%%%%%%%%%%%%%%%

% WRN 28-10, CIFAR-10, Adam
% %%%%%%%%%%%%%%%%%%%%%%%%%%%%%%%%%%%%%%%%%%%%%%%%%%%%%%%
\begin{figure}[H]
\centering
\begin{tabular}{ccc}
\includegraphics[width=0.3\linewidth]{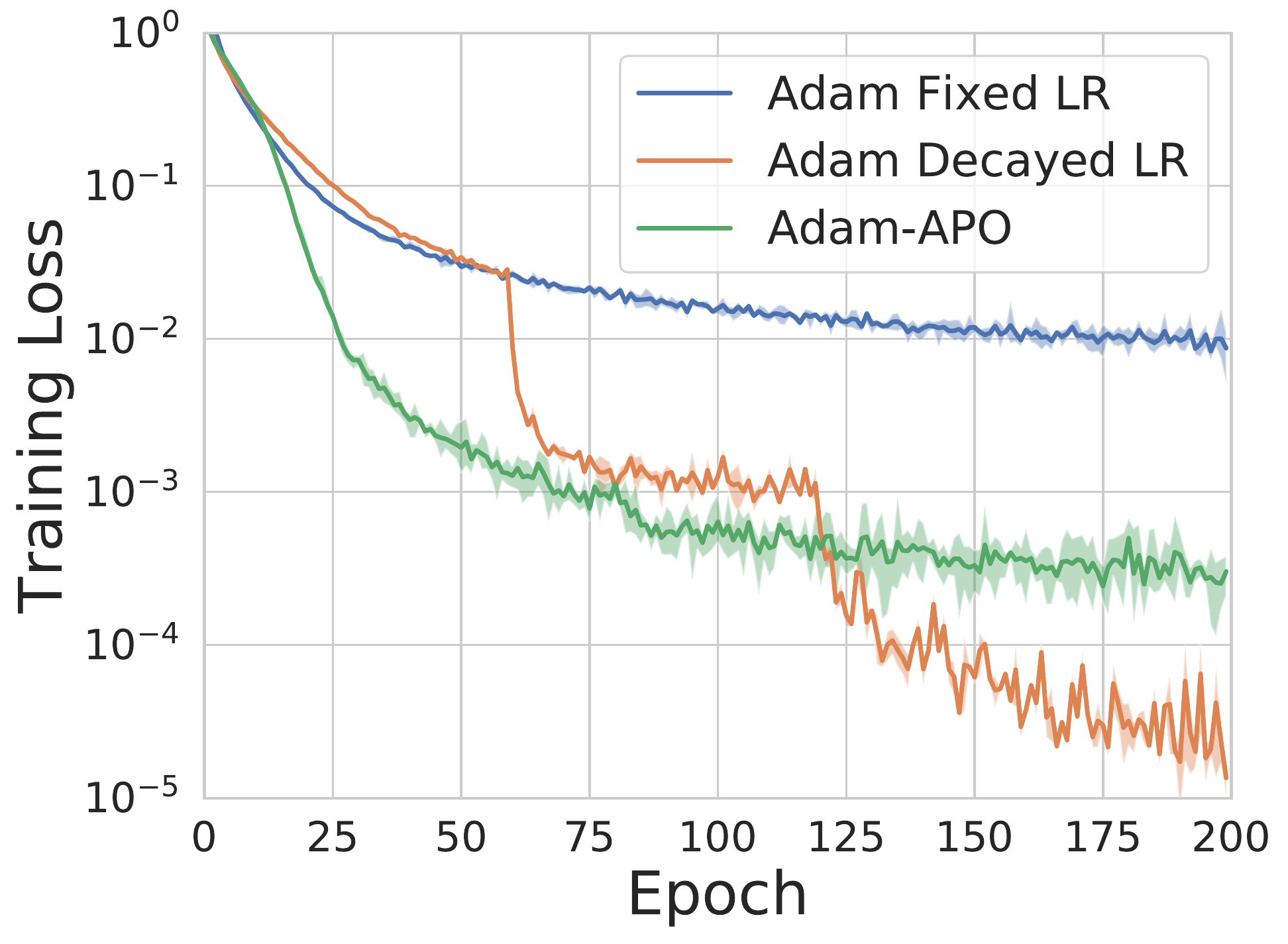}&
\includegraphics[width=0.3\linewidth]{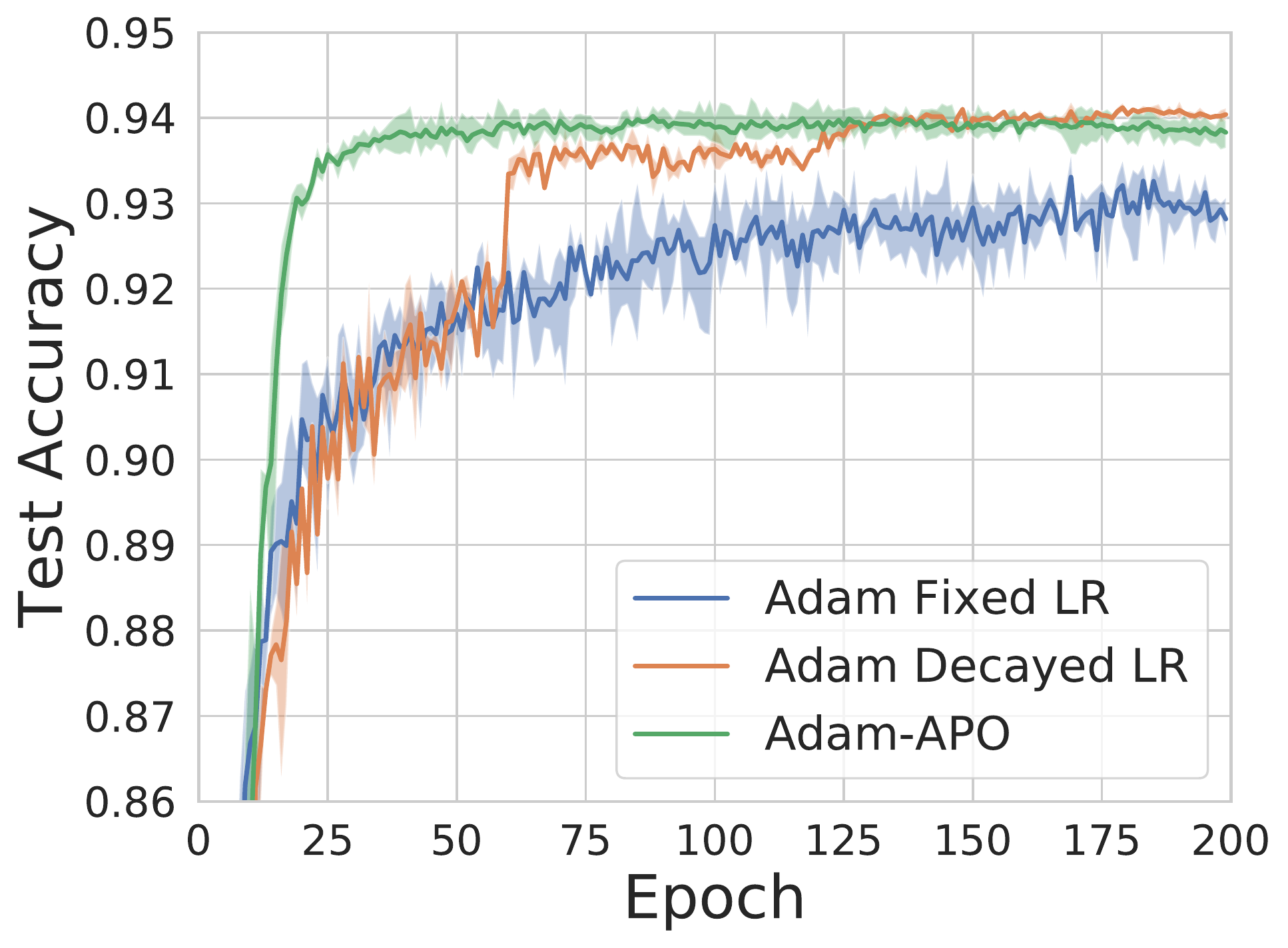}&
\includegraphics[width=0.3\linewidth]{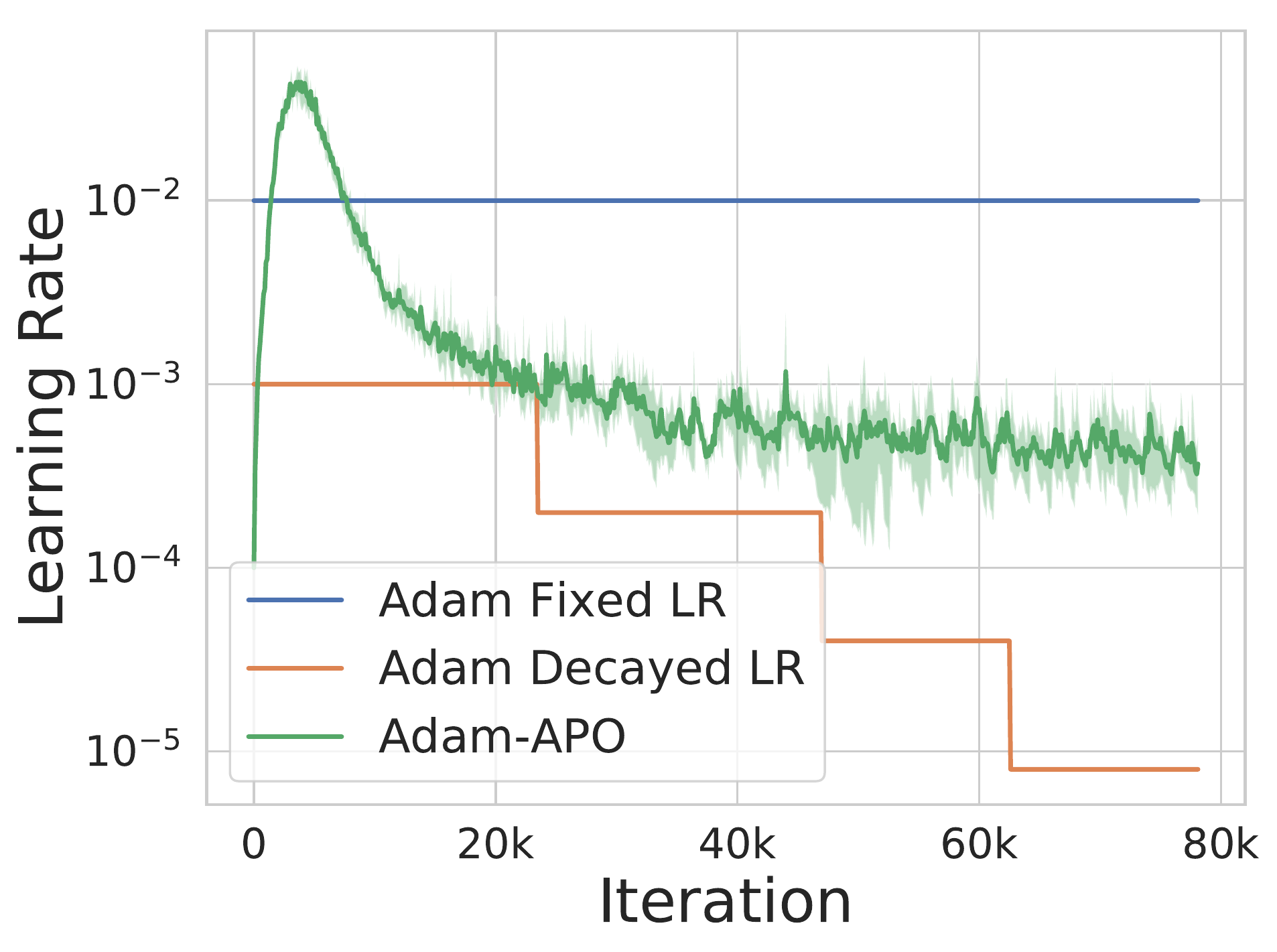}
\end{tabular}
\vspace{-2mm}
\caption{WideResNet 28-10 on CIFAR-10, using Adam. The shaded regions show the min/max values over 4 random restarts.}
\vspace{-4mm}
\label{fig:wrn-cifar10-adam}
\end{figure}
% %%%%%%%%%%%%%%%%%%%%%%%%%%%%%%%%%%%%%%%%%%%%%%%%%%%%%%%

\subsection{CIFAR-100}\label{app:cifar100}

\paragraph{APO Learning Rate Schedules.}

We show training curve plots for different base-optimizers below.
We find that APO usually achieves better performance than fixed LR, and is comparable with the manual schedule.

\begin{figure}[H]
\centering
\begin{tabular}{ccc}
\includegraphics[width=0.3\linewidth]{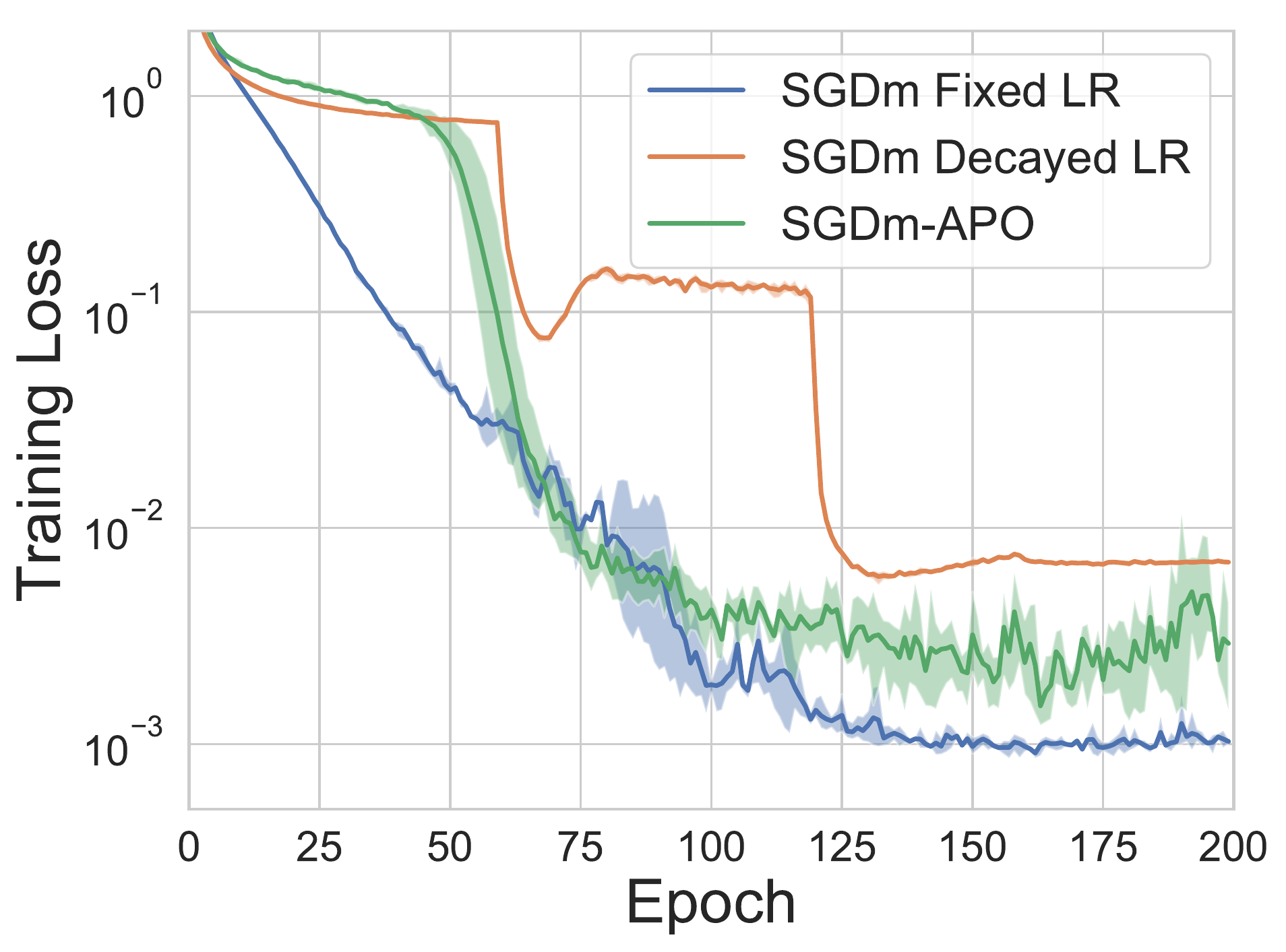}&
\includegraphics[width=0.3\linewidth]{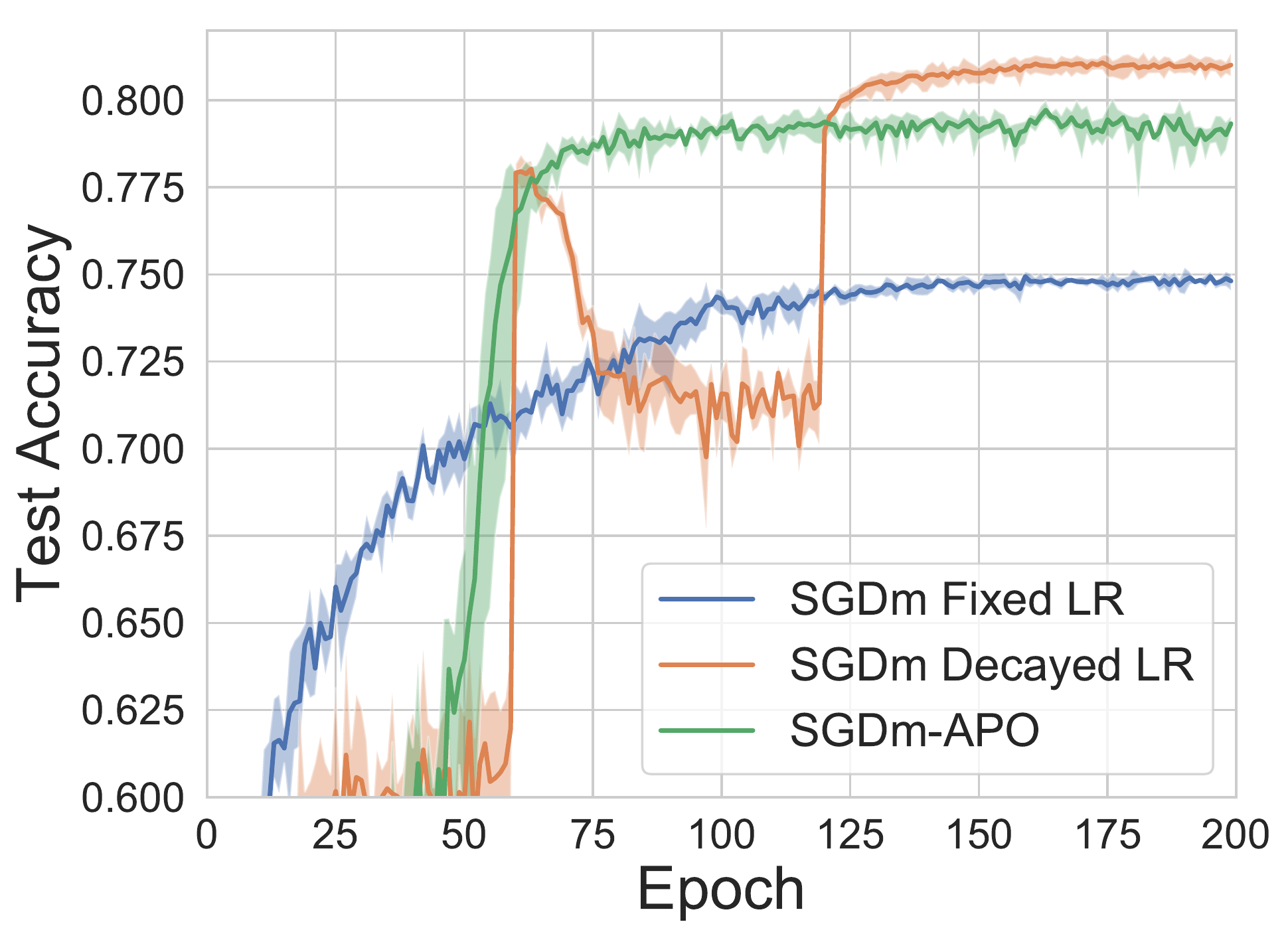}&
\includegraphics[width=0.3\linewidth]{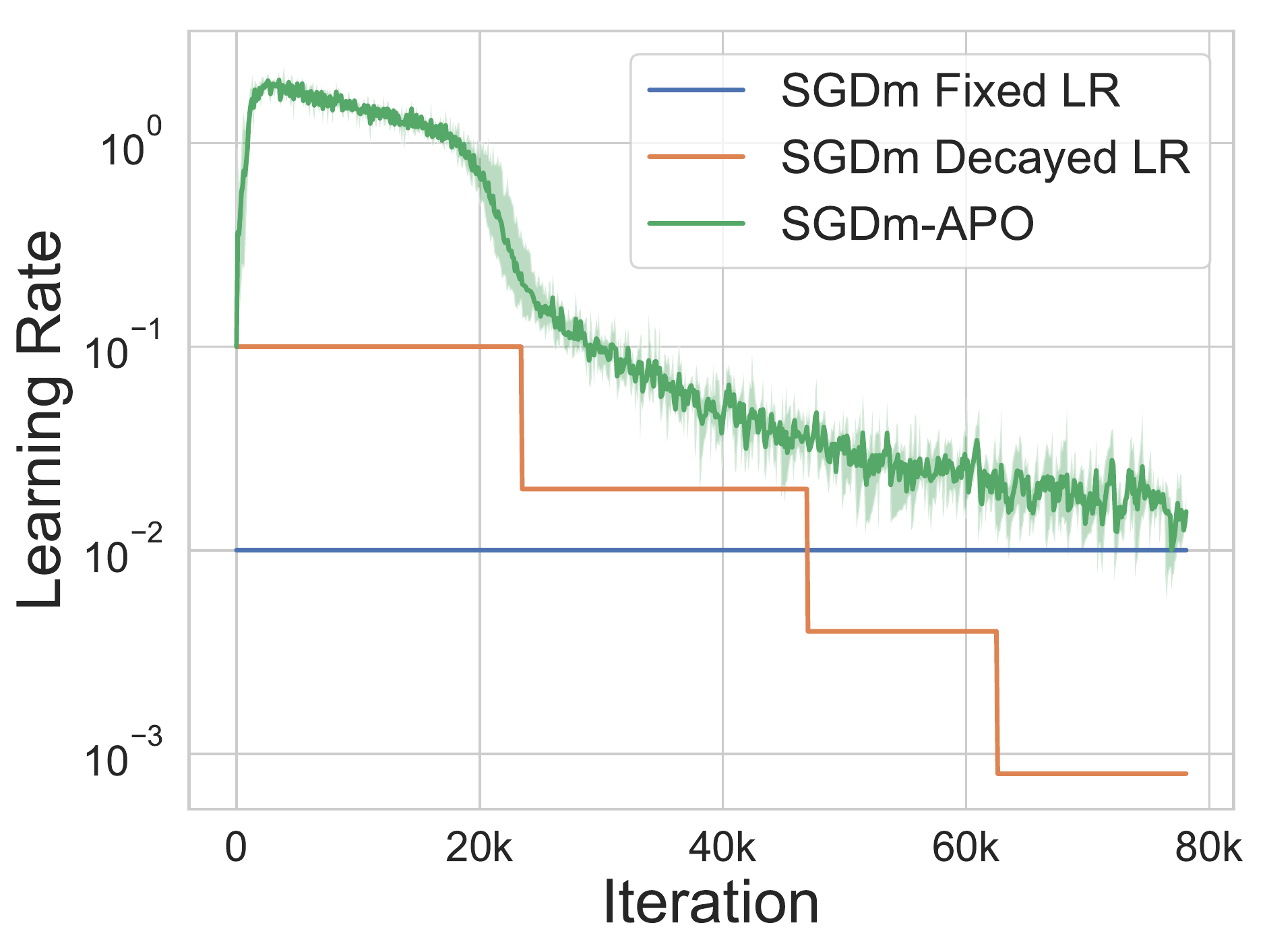}
\end{tabular}
\vspace{-2mm}
\caption{WideResNet on CIFAR-100, using SGDm. The shaded regions show the min/max values over 4 random restarts.}
\vspace{-4mm}
\label{fig:wideresnet-cifar100-sgdm}
\end{figure}

% \begin{figure}[H]
% \centering
% \begin{tabular}{ccc}
% \includegraphics[width=0.3\linewidth]{icml_figures/wideresnet_cifar100/wrn_cifar100_rmsprop/wrn_cifar100_rmsprop_train_loss.pdf}&
% \includegraphics[width=0.3\linewidth]{icml_figures/wideresnet_cifar100/wrn_cifar100_rmsprop/wrn_cifar100_rmsprop_test_acc.pdf}&
% \includegraphics[width=0.3\linewidth]{icml_figures/wideresnet_cifar100/wrn_cifar100_rmsprop/wrn_cifar100_rmsprop_hparams.pdf}
% \end{tabular}
% \vspace{-2mm}
% \caption{\textbf{WideResNet on CIFAR-100, using RMSprop.} The shaded regions show the min/max values over 4 random restarts.}
% \vspace{-4mm}
% \label{fig:wideresnet-cifar100-rmsprop}
% \end{figure}

\begin{figure}[H]
\centering
\begin{tabular}{ccc}
\includegraphics[width=0.3\linewidth]{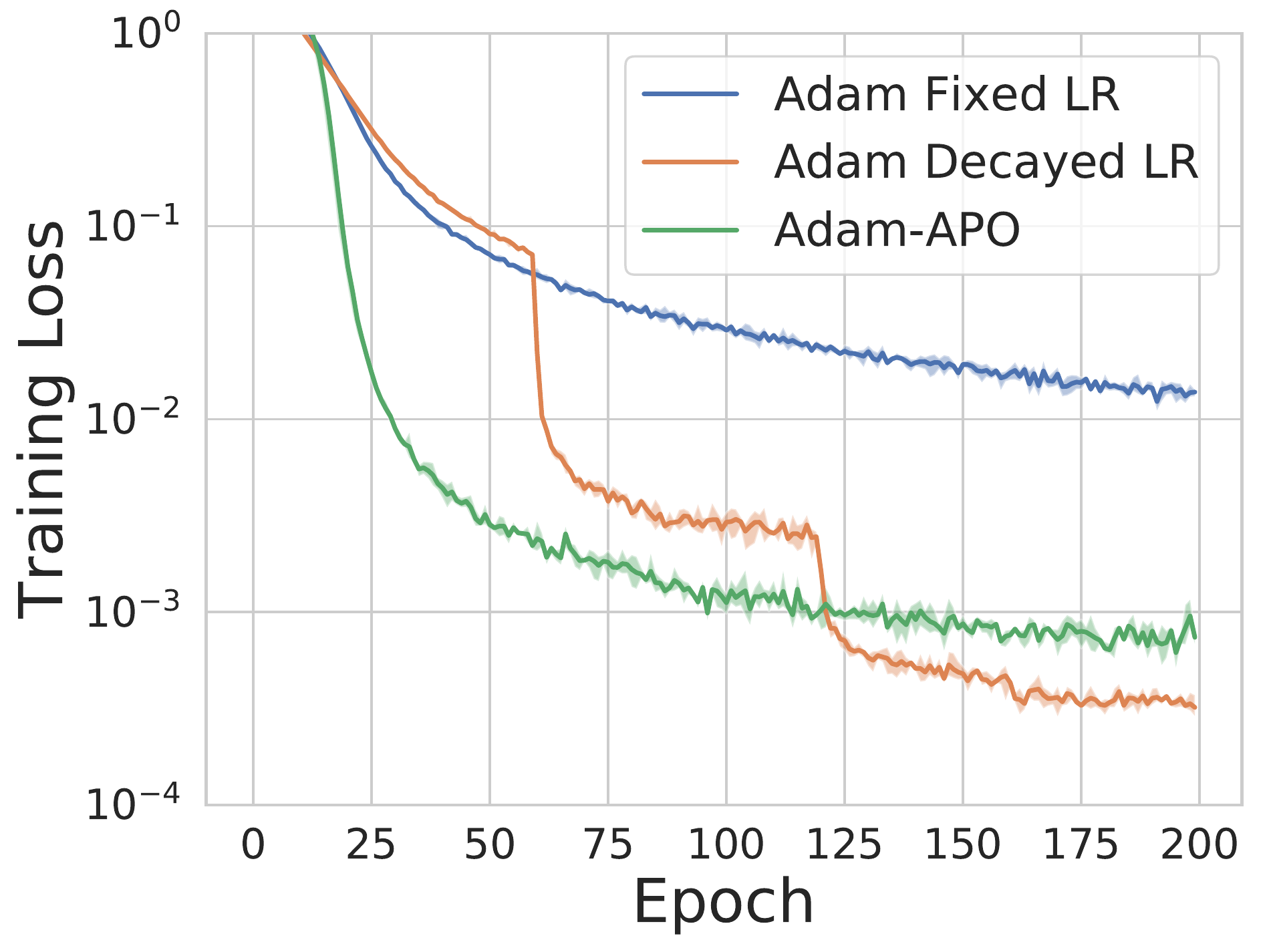}&
\includegraphics[width=0.3\linewidth]{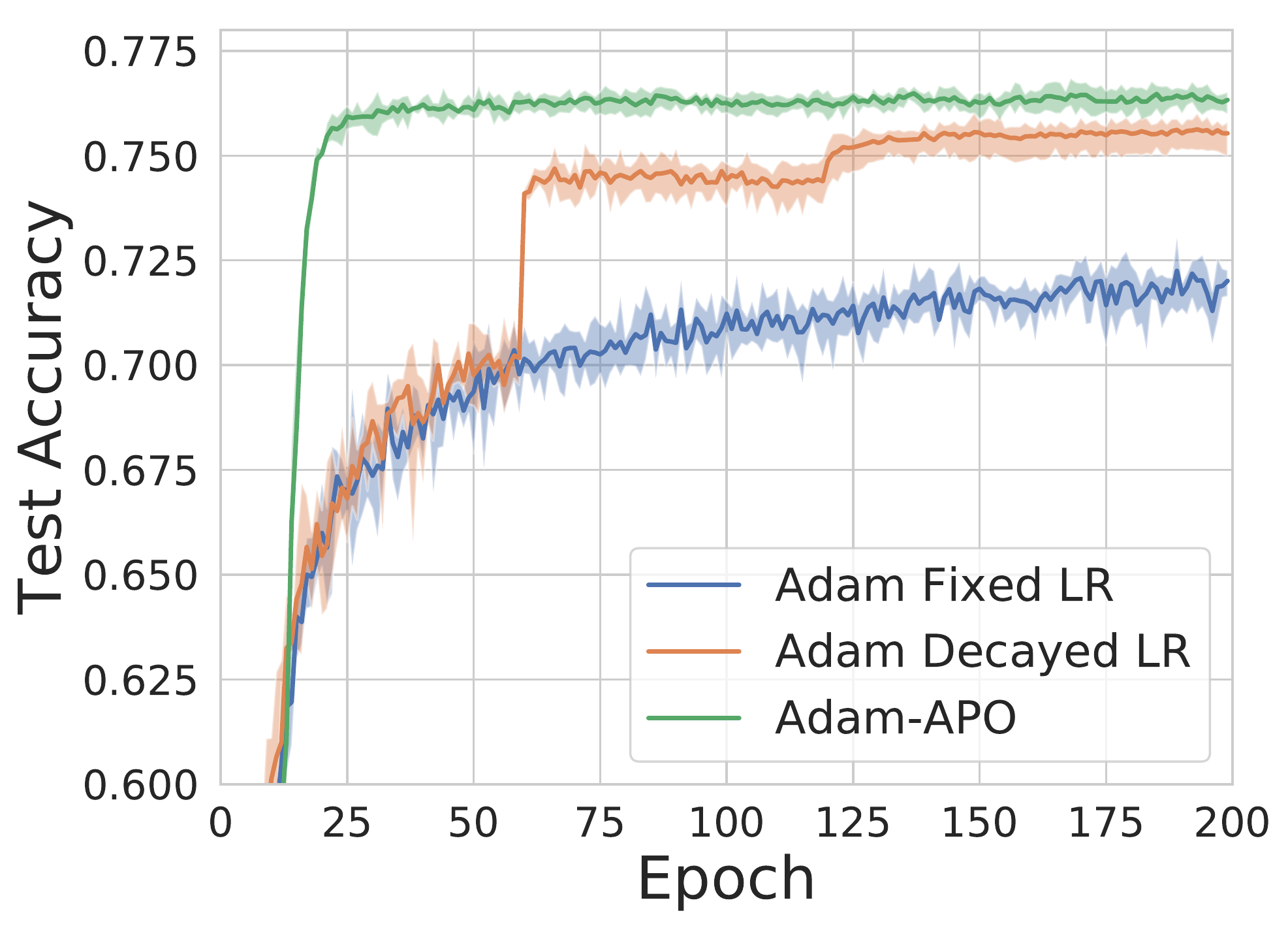}&
\includegraphics[width=0.3\linewidth]{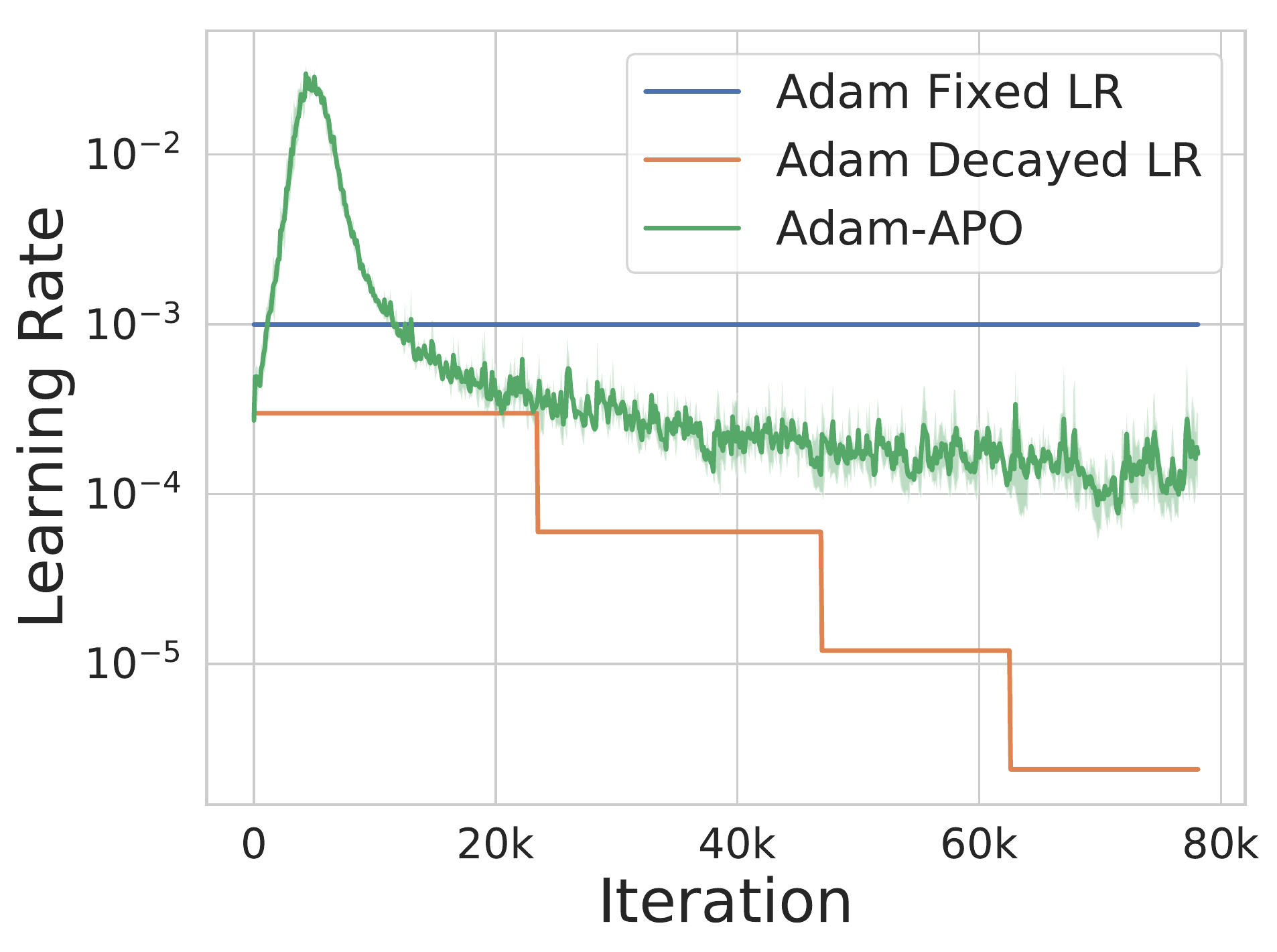}
\end{tabular}
\vspace{-2mm}
\caption{WideResNet on CIFAR-100, using Adam. The shaded regions show the min/max values over 4 random restarts.}
\vspace{-4mm}
\label{fig:wideresnet-cifar100-adam}
\end{figure}

\subsection{16-bit Neural Network Training}
\label{app:16-bit}
We show the training curve plots for training 16-bit ResNet-18 on CIFAR-10 and CIFAR-100 in Figure~\ref{fig:16c} .

\begin{figure}[H]
\centering
\includegraphics[width=0.335\linewidth]{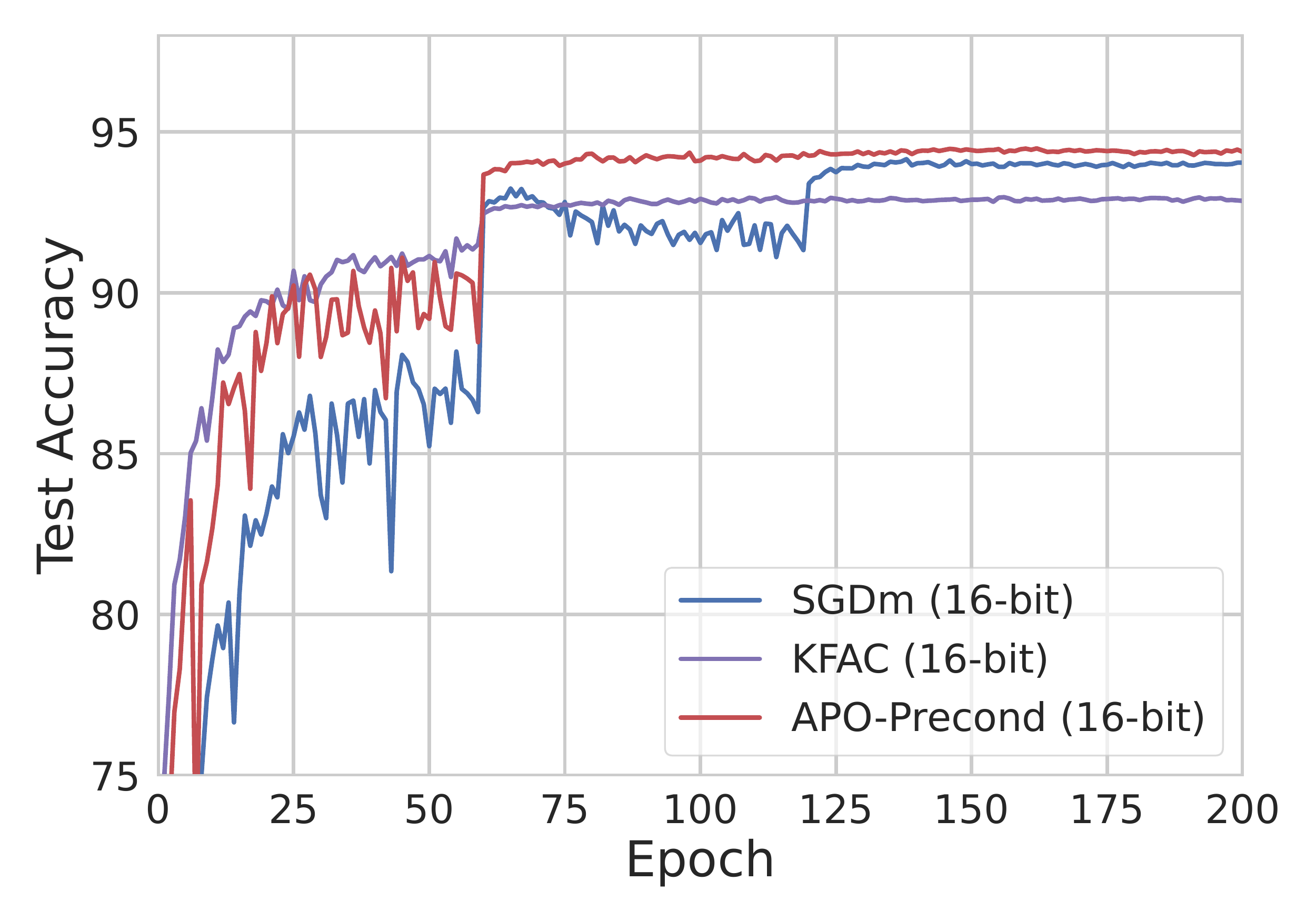}
\includegraphics[width=0.35\linewidth]{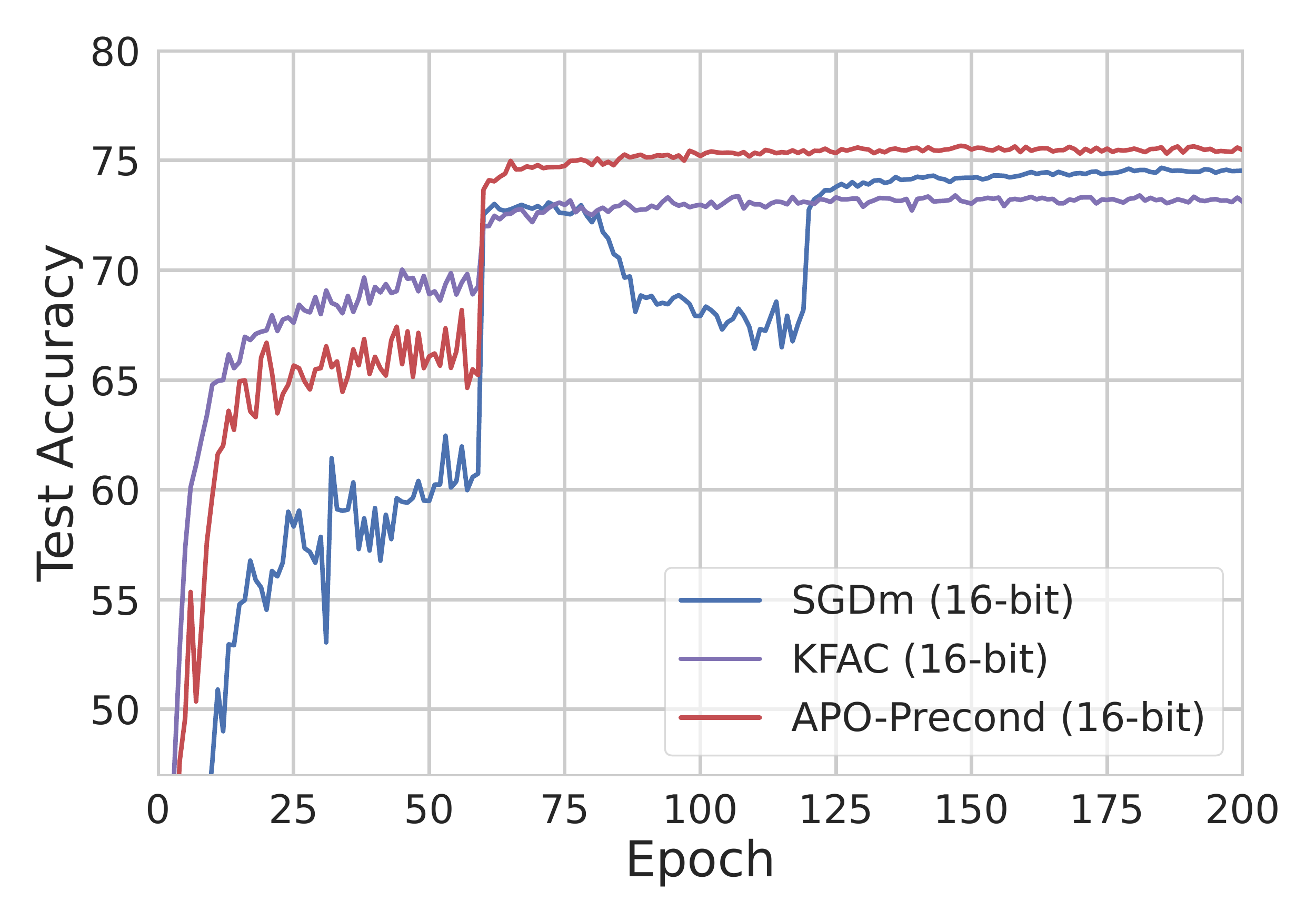}
\caption{Test accuracy curves for 16-bit neural networks: ResNet-18 on CIFAR-10 (\textbf{left}) and CIFAR-100 (\textbf{right}).}
\label{fig:16c}
\end{figure}

\section{Importance of Using the Same Minibatch}
\label{app:same-minibatch-first-term}

In this section, we empirically verify the importance of computing the first term of the meta-objective on the \textit{same mini-batch} that is used to compute the gradient of the base optimizer. Recall that our meta-objective is:
\begin{equation} \label{eq:same-first-term}
    \mathcal{Q}(\boldphi) =
    \cost_{{\color{dkgreen}\batch}}(u(\boldtheta, \boldphi, {\color{dkgreen}\batch}))
    +
    \lamfsd \mathbb{E}_{({\color{blue}\boldxtil}, \cdot) \sim \cD} [ \disf(u(\boldtheta, \boldphi, {\color{dkgreen}\batch}), \boldtheta, {\color{blue}\boldxtil}) ]
    +
    \frac{\lamwsd}{2} || u(\boldtheta, \boldphi, {\color{dkgreen}\batch}) - \boldtheta ||_2^2
\end{equation}
Note that the loss term is computed using the same mini-batch ${\color{dkgreen} \batch}$ used for the update $u(\boldtheta, \boldphi, {\color{dkgreen} \batch})$.
An alternative is to evaluate the loss term on a randomly-sampled mini-batch ${\color{dkred} \batch'} \sim \cD$, yielding:
\begin{equation} \label{eq:diff-first-term}
    \mathcal{Q}(\boldphi) =
    \cost_{{\color{dkred}\batch'}}(u(\boldtheta, \boldphi, {\color{dkgreen}\batch}))
    +
    \lamfsd \mathbb{E}_{({\color{blue}\boldxtil}, \cdot) \sim \cD} [ \disf(u(\boldtheta, \boldphi, {\color{dkgreen}\batch}), \boldtheta, {\color{blue}\boldxtil}) ]
    +
    \frac{\lamwsd}{2} || u(\boldtheta, \boldphi, {\color{dkgreen}\batch}) - \boldtheta ||_2^2
\end{equation}
If we set $\lamfsd = \lamwsd = 0$, then Eq.~\ref{eq:diff-first-term} becomes the meta-objective analyzed by~\citet{wu2018understanding}, which was used to illustrate the short-horizon bias issue in stochastic meta-optimization.

Figure~\ref{fig:same-vs-diff-first-term-rmsprop} shows the result of using Eq.~\ref{eq:same-first-term} vs. Eq.~\ref{eq:diff-first-term} for meta-optimizing the learning rate of a ResNet32 model on CIFAR-10.
We observe that when using a random minibatch ${\color{dkred} \batch'}$ to compute the loss term, the learning rate decays rapidly, preventing long-term progress; in contrast, using the same minibatch ${\color{dkgreen} \batch}$ for the loss term yields a resonable LR schedule which does not suffer from the short-horizon bias issue empirically.

\begin{figure}[h]
\centering
\includegraphics[width=0.32\linewidth]{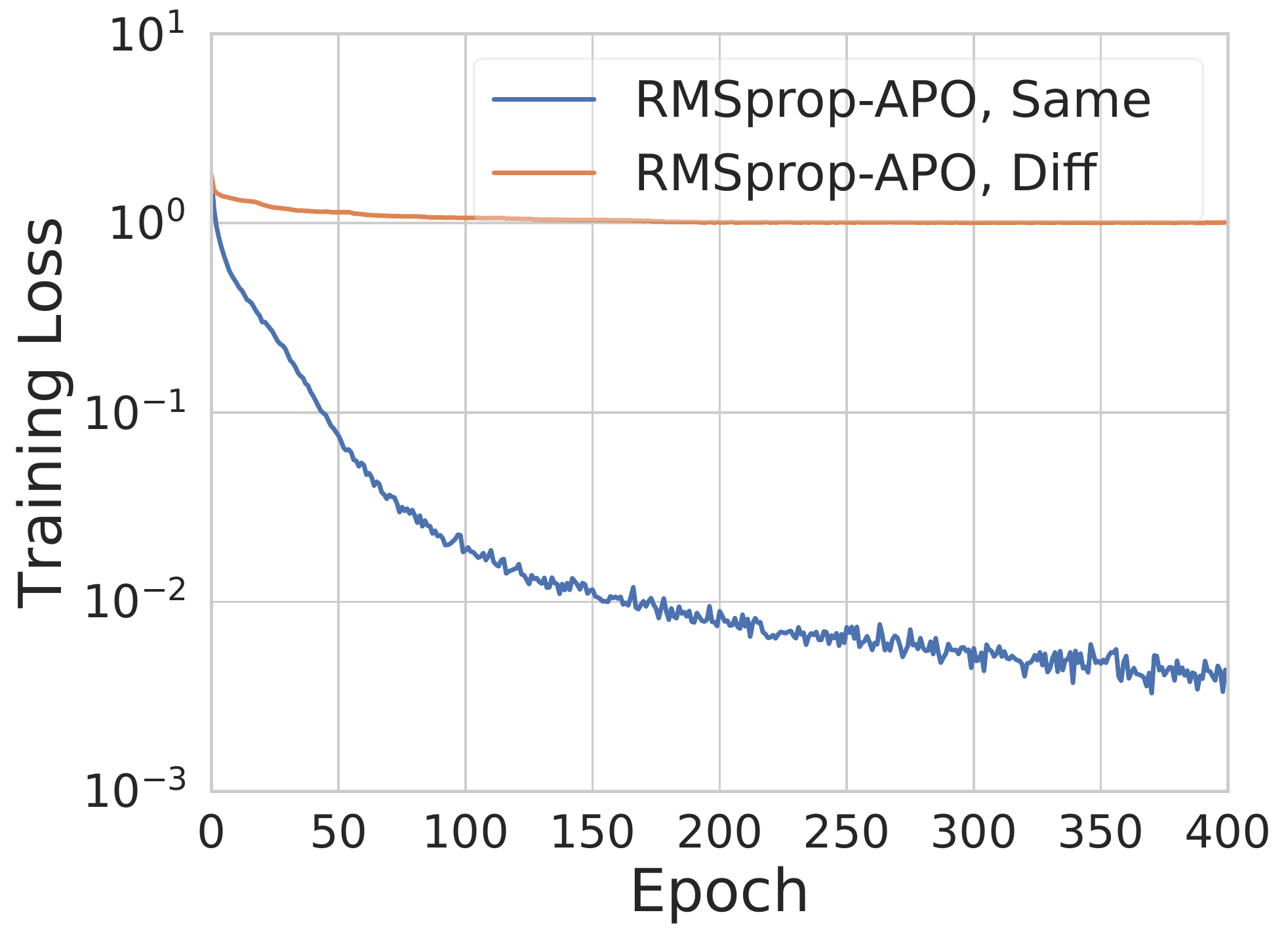}
\includegraphics[width=0.31\linewidth]{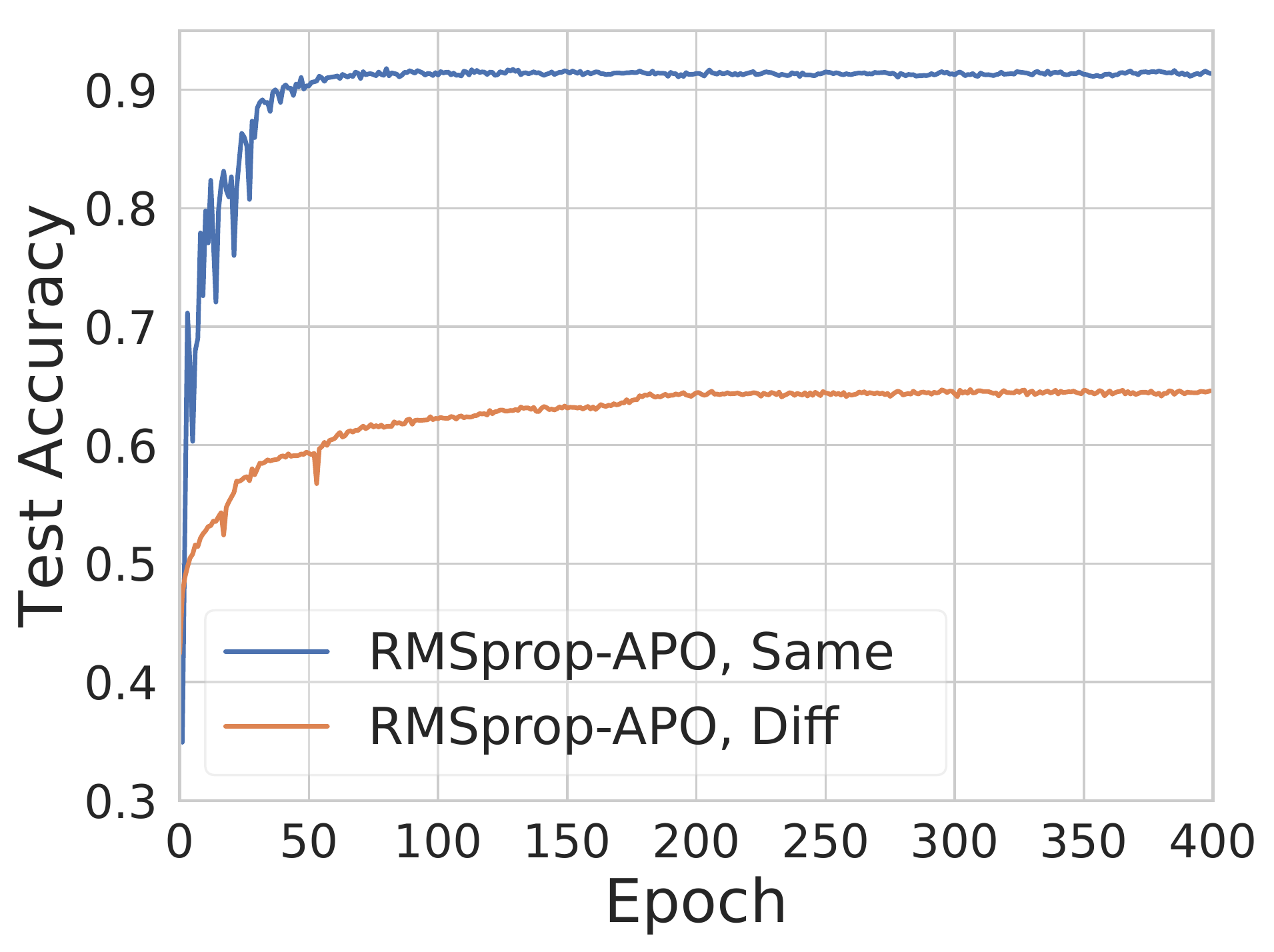}
\includegraphics[width=0.32\linewidth]{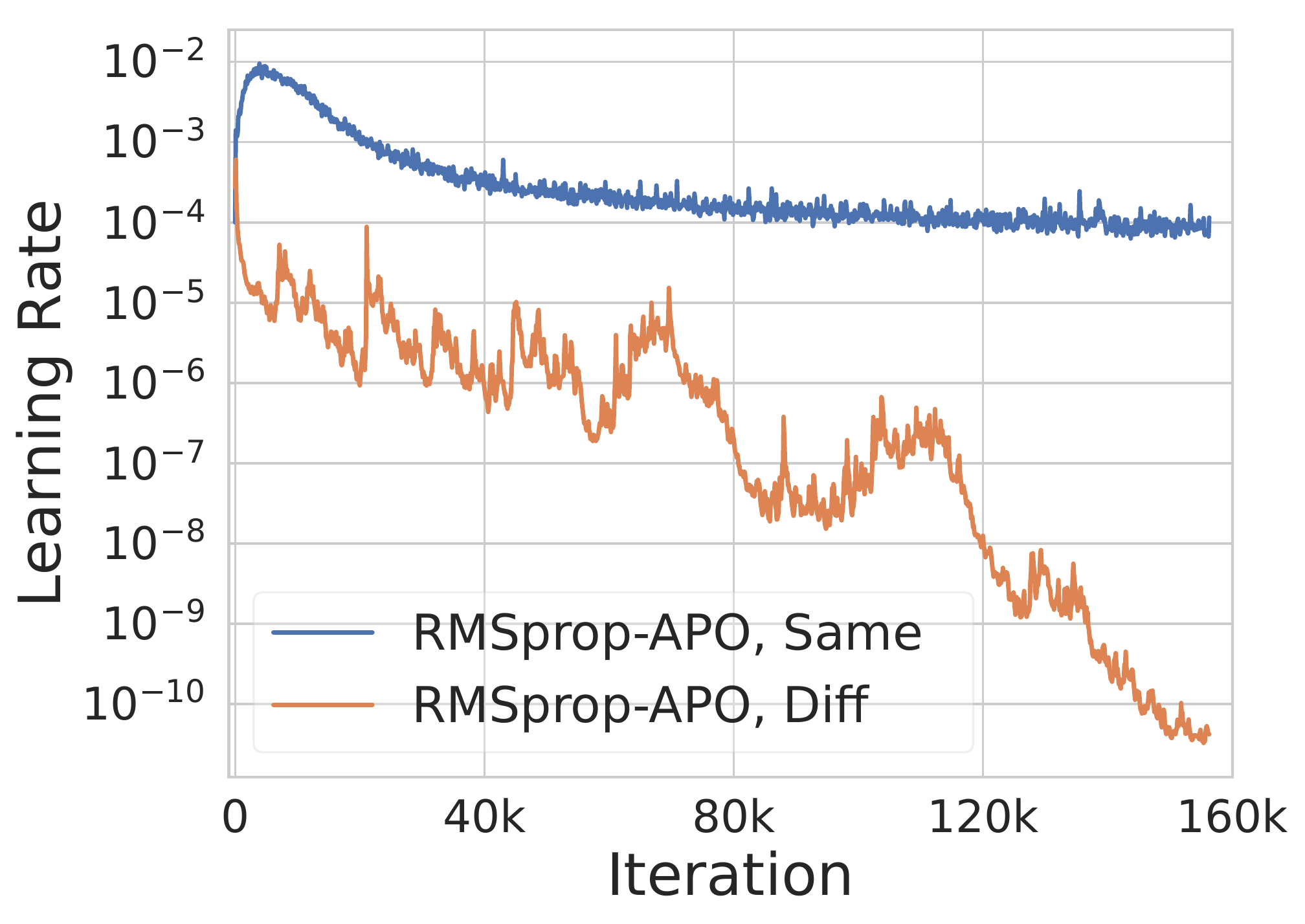}
\caption{\textbf{Left:} Training loss compared between 1) using the same mini-batch for the first term and 2) using a different mini-batch for the first term. \textbf{Right:} Comparison of learning rates for the two conditions.}
\label{fig:same-vs-diff-first-term-rmsprop}
\end{figure}

\begin{figure}[h]
\centering
\includegraphics[width=0.32\linewidth]{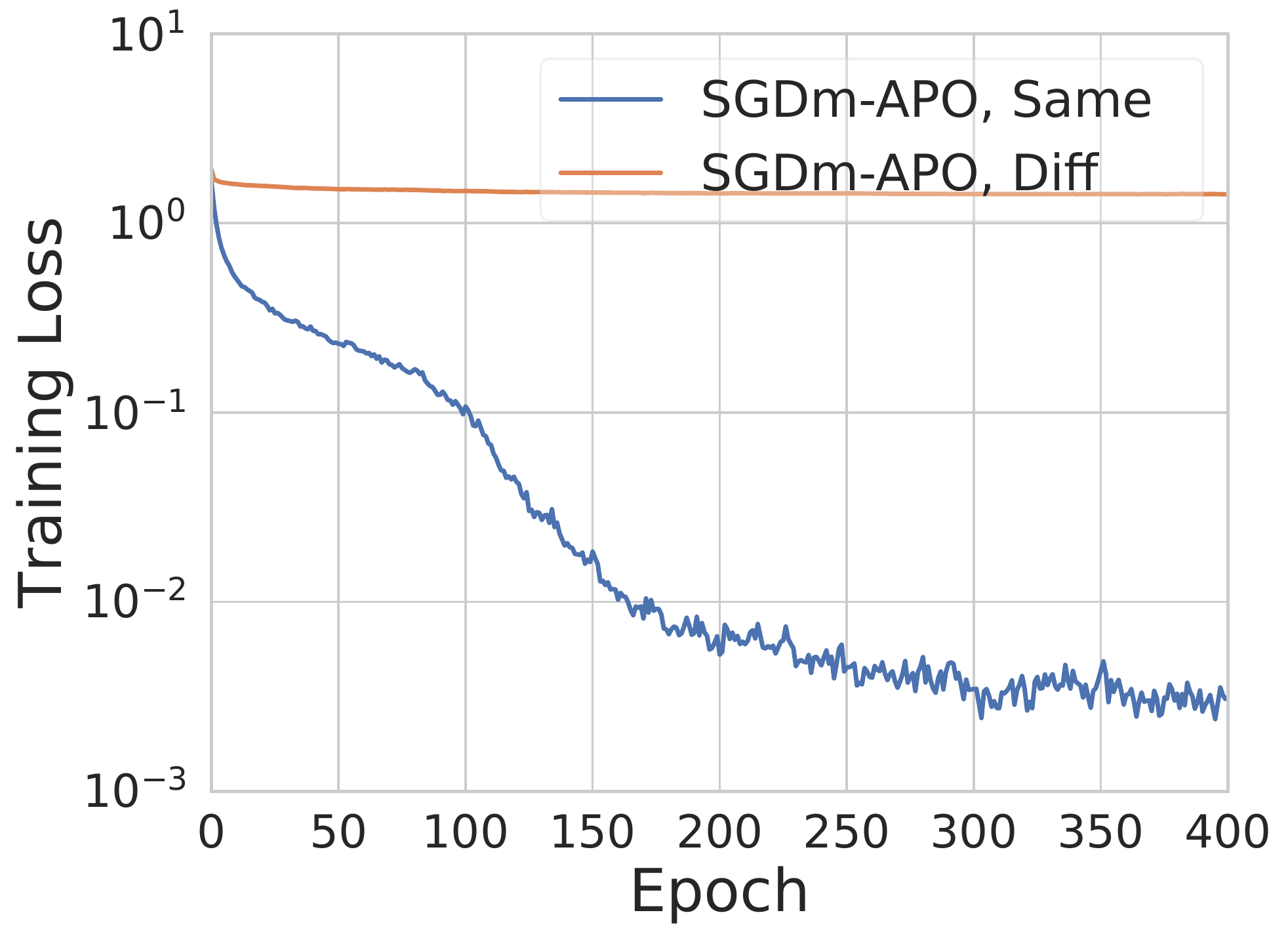}
\includegraphics[width=0.31\linewidth]{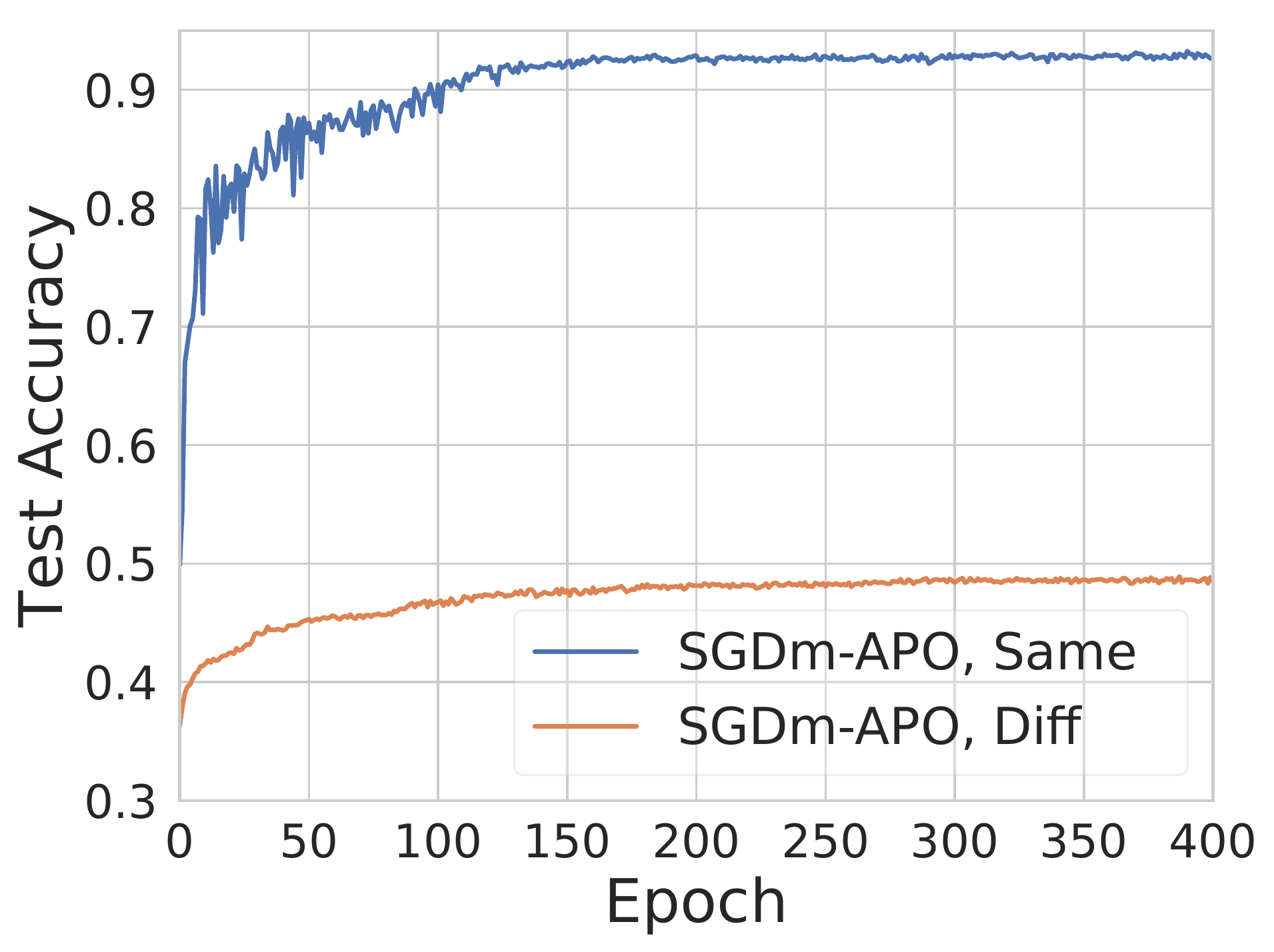}
\includegraphics[width=0.32\linewidth]{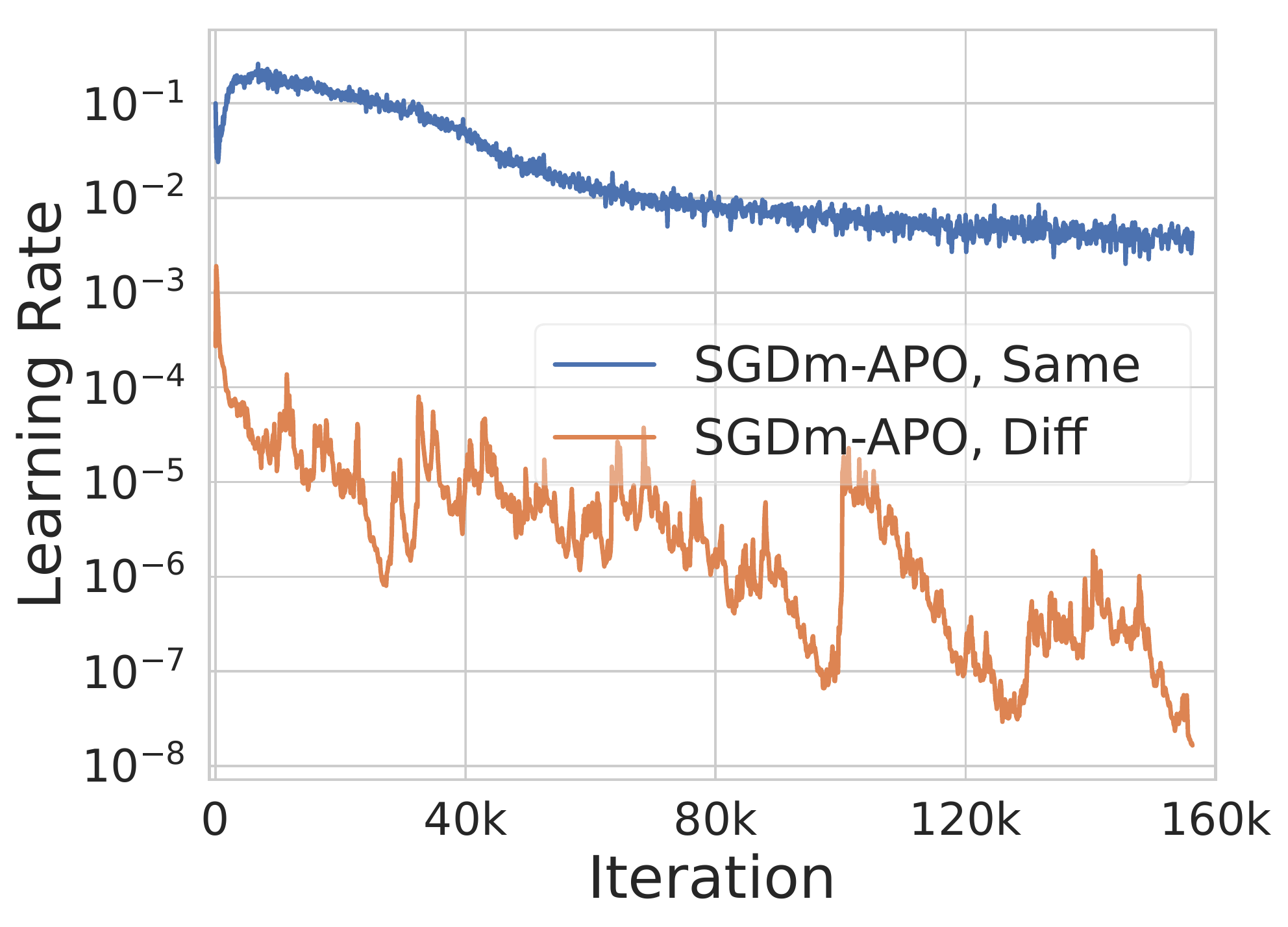}
\caption{\textbf{Left:} Training loss compared between 1) using the same mini-batch for the first term and 2) using a different mini-batch for the first term. \textbf{Right:} Comparison of learning rates for the two conditions.}
\label{fig:same-vs-diff-first-term-sgdm}
\end{figure}

\paragraph{Same Mini-Batch for the FSD Term.}
Another alternative to the meta-objective is to compute the FSD term using the same mini-batch used to compute the loss term, as shown in Eq.~\ref{eq:same-second-term}.
\begin{equation} \label{eq:same-second-term}
    \mathcal{Q}(\boldphi) =
    \cost_{{\color{dkgreen} \batch}}(u(\boldtheta, \boldphi, {\color{dkgreen}\batch}))
    +
    \frac{\lamfsd}{|{\color{dkgreen}\batch}|}\sum_{({\color{dkgreen} \boldxtil}, \cdot) \in {\color{dkgreen} \batch}} \disf(u(\boldtheta, \boldphi, {\color{dkgreen}\batch}), \boldtheta, {\color{dkgreen}\boldxtil})
    +
    \frac{\lamwsd}{2} || u(\boldtheta, \boldphi, {\color{dkgreen} \batch}) - \boldtheta ||_2^2
\end{equation}
Recall that $\disf(u(\boldtheta, \boldphi, {\color{dkgreen} \batch}), \boldtheta, {\color{dkgreen} \boldxtil}) = \rho(f({\color{dkgreen} \boldxtil}, u(\boldtheta, \boldphi, {\color{dkgreen} \batch})), f({\color{dkgreen} \boldxtil}, \boldtheta))$.
The function evaluation $f({\color{dkgreen} \boldxtil}, u(\boldtheta, \boldphi, {\color{dkgreen} \batch}))$ must be performed to compute the loss term; thus, using the same minibatch ${\color{dkgreen} \batch}$ for the FSD term can allow us to reduce computation by re-using this function output in both the loss and FSD computations.
However, with the interpretation of the FSD term in Eq.~\ref{eq:same-second-term} as a Monte Carlo estimate of the expectation $\mathbb{E}_{\tilde{\boldx} \sim \mathcal{D}}[\disf(u(\boldtheta, \boldphi, \batch), \boldtheta, \boldxtil)]$, using the same mini-batch for the loss and dissimilarity would yield a biased estimate.
The effect of using a different vs the same mini-batch to compute the FSD term is shown in Figure~\ref{fig:same-vs-diff-second-term-adam}.

\begin{figure}[h]
\centering
\includegraphics[width=0.32\linewidth]{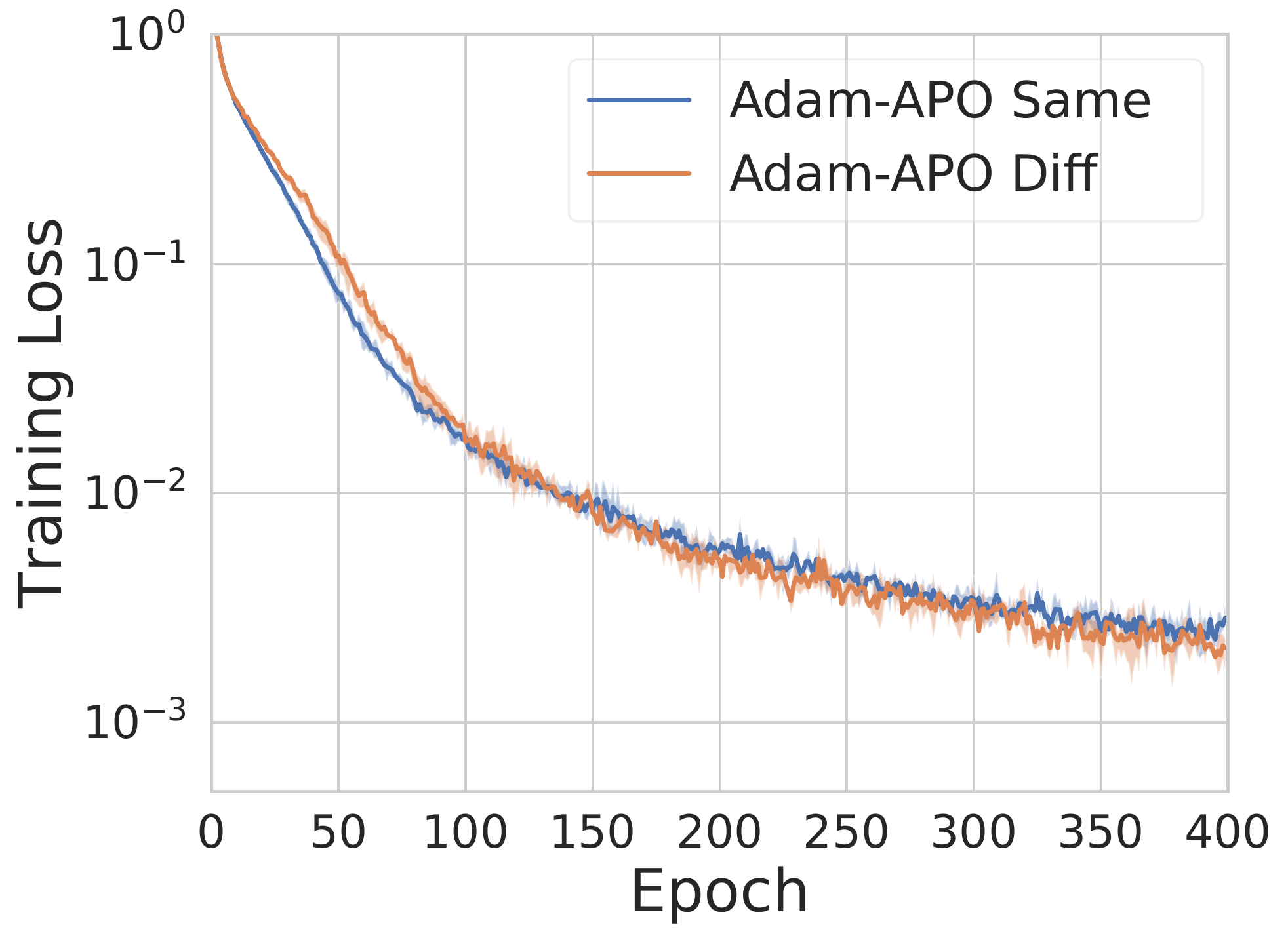}
\includegraphics[width=0.32\linewidth]{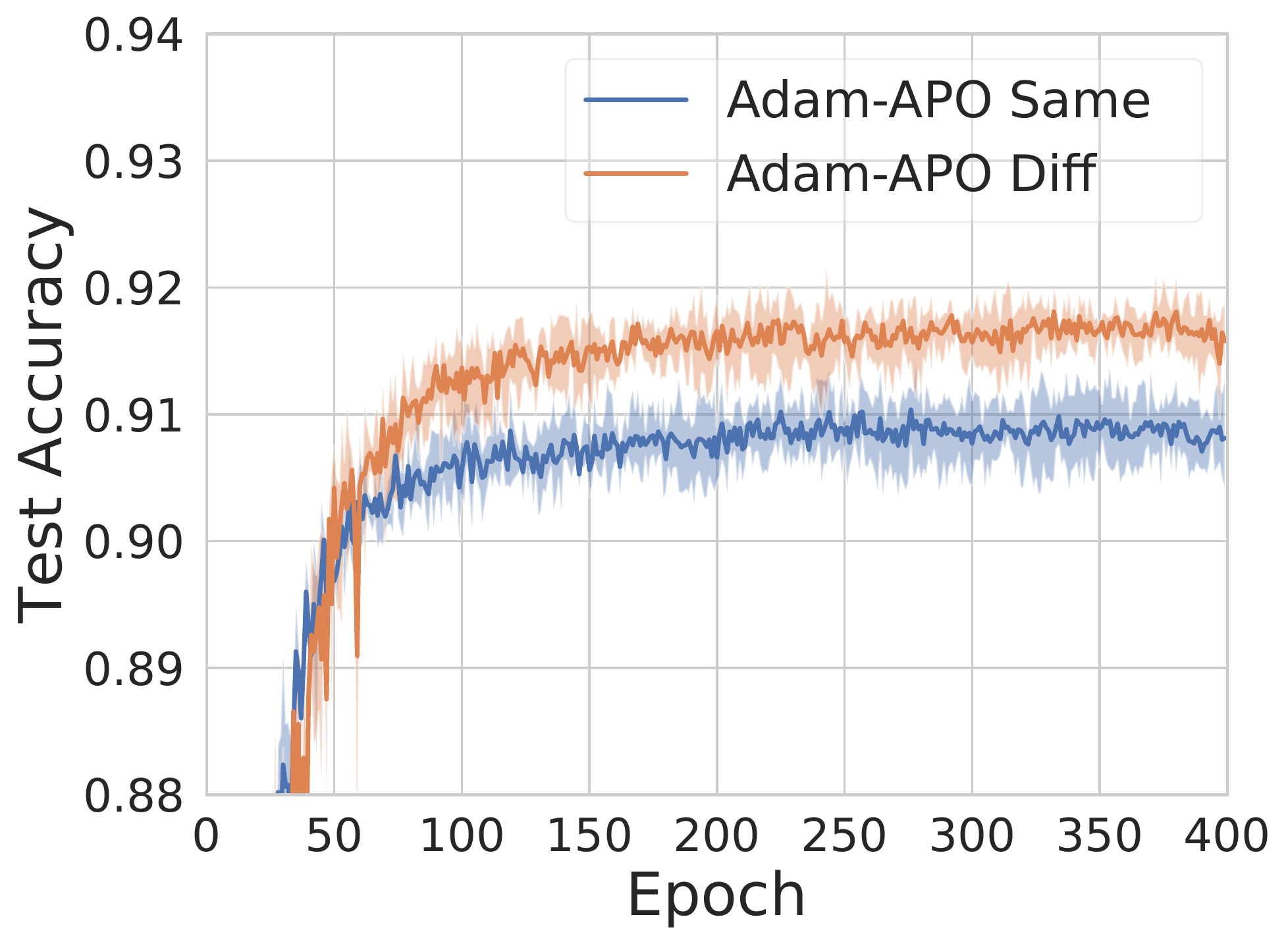}
\includegraphics[width=0.32\linewidth]{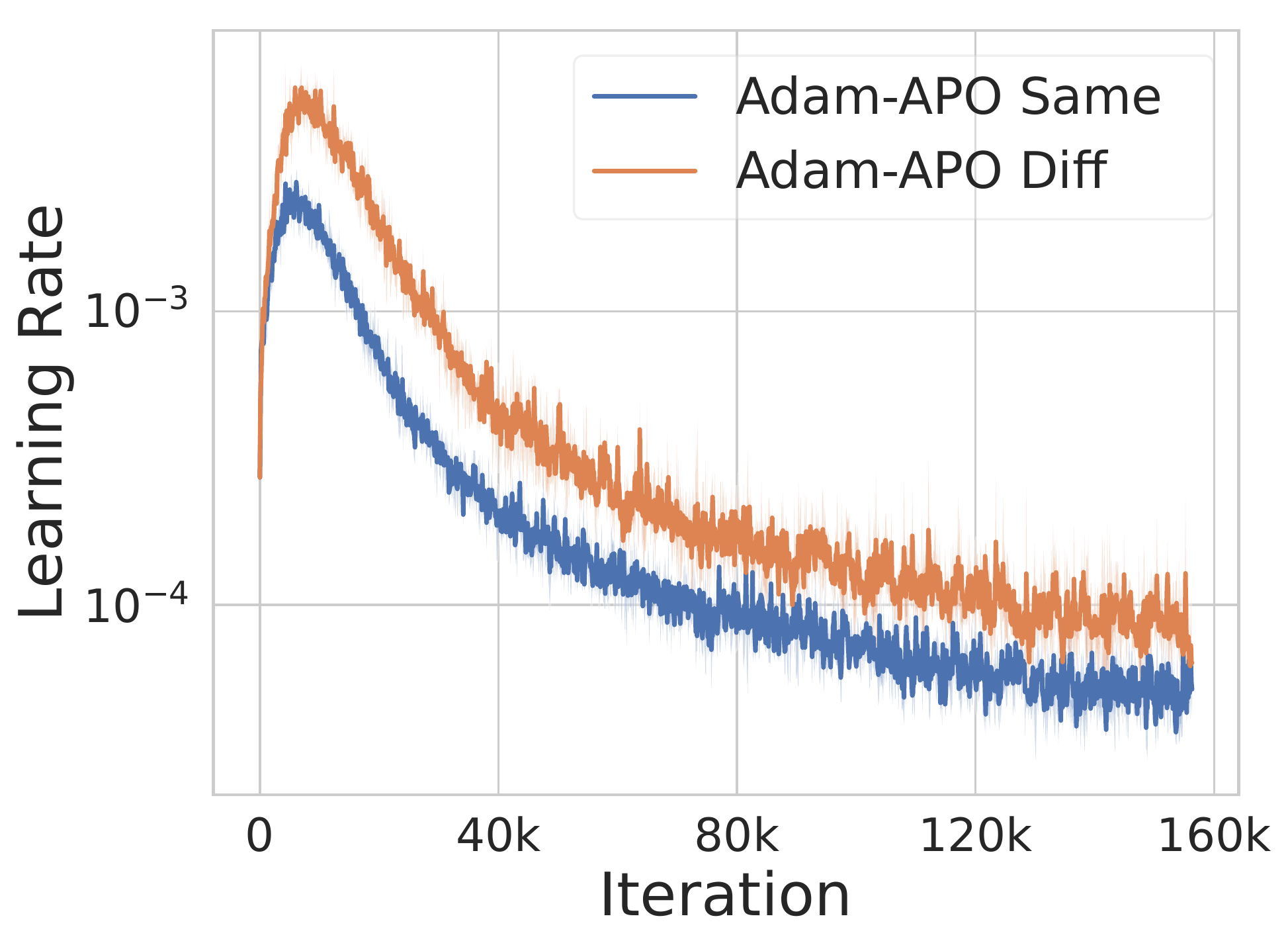}
\caption{\textbf{Left:} Training loss compared between 1) using the same mini-batch for the second term as for the first term of the meta-objective and 2) using a different mini-batch for the second term. \textbf{Right:} Comparison of learning rates for the two conditions.}
\label{fig:same-vs-diff-second-term-adam}
\end{figure}

\section{Approximate Closed-Form Solution for the PPM}
\label{app:approximate-proximal-point}

In Section~\ref{sec:proximal-optimization}, we introduced a general update rule for the stochastic PPM, which is defined as follows: 
\begin{align}
     \vtheta^{(t+1)} &\leftarrow \argmin_{\mathbf{u} \in \mathbb{R}^m} \cost_{\batch} (\mathbf{u})  +
    \lambda_{\text{FSD}} \mathbb{E}_{\tilde{\mathbf{x}} \sim \mathcal{D}} [ \disf(\mathbf{u}, \vtheta^{(t)}, \mathbf{x})]  +
    \lambda_{\text{WSD}} \disw(\mathbf{u}, \vtheta^{(t)}).
    \label{eq:nn-proximal-point-method-2}
\end{align}
We first take the infinitesimal limit by letting $\lambfsd \to \infty$. Then, the optimal update $\mathbf{u}^{\star}$ will stay close to the current parameters $\vtheta^{(t)}$ and we can approximate the loss function with the first-order Taylor approximation. Moreover, since the FSD term $\mathbb{E}_{\tilde{\mathbf{x}} \sim \mathcal{D}} [ \disf(\mathbf{u}, \vtheta^{(t)}, \mathbf{x})]$ is minimized when $\mathbf{u} = \vtheta^{(t)}$, we have $\grad_{\mathbf{u}}\mathbb{E}_{\tilde{\mathbf{x}} \sim \mathcal{D}} [ \disf(\mathbf{u}, \vtheta^{(t)}, \mathbf{x})]|_{\mathbf{u} = \vtheta^{(t)}} = \mathbf{0}$. Hence, the second-order Taylor expansion of the FSD term would be:
\begin{align}
    \mathbb{E}_{\tilde{\mathbf{x}} \sim \mathcal{D}}[ \disf (f(\tilde{\mathbf{x}}, \mathbf{u}), f(\tilde{\mathbf{x}}, \boldsymbol{\theta}^{(t)}))] \approx \frac{1}{2} (\mathbf{u} - \vtheta^{(t)})^{\top} \mathbf{G} (\mathbf{u} - \vtheta^{(t)}),
    \label{eq:app-quad-diss}
\end{align}
where $\mathbf{G} = \nabla^2_{\mathbf{u}} \mathbb{E}_{\tilde{\mathbf{x}} \sim \mathcal{D}} [ \disf(\mathbf{u}, \vtheta^{(t)}, \mathbf{x})]|_{\mathbf{u} = \vtheta^{(t)}}$ is the local Hessian of the FSD term. We further let the weight space discrepancy function to be the squared Euclidean distance $\disw(\mathbf{u}, \vtheta^{(t)}) = \sfrac{1}{2} \|\mathbf{u} - \vtheta^{(t)}\|^2_2$. Combining these insights, we approximate Eqn.~\ref{eq:nn-proximal-point-method-2} by linearzing the loss and taking a quadratic approximation to the FSD term:
\begin{align}
    \vtheta^{(t+1)} \approx \argmin_{\mathbf{u} \in \mathbb{R}^m} \left[\grad_{\vtheta} \cost_{\mathcal{B}}(\vtheta^{(t)})^{\top} (\mathbf{u} - \vtheta^{(t)}) + \frac{\lambfsd}{2} (\mathbf{u} - \vtheta^{(t)})^{\top} \mathbf{G} (\mathbf{u} - \vtheta^{(t)}) + \frac{\lambwsd}{2} \|\mathbf{u} - \vtheta^{(t)} \|^2 \right].
\end{align}
Taking the gradient with respect to $\mathbf{u}$ and setting it equal to $\mathbf{0}$, we have:
\begin{align}
    \grad_{\vtheta} \cost_{\mathcal{B}}(\vtheta^{(t)}) + \lambfsd \mathbf{G} \mathbf{u} - \lambfsd \mathbf{G} \vtheta^{(t)} + \lambwsd \mathbf{u} - \lambwsd \vtheta^{(t)} = \mathbf{0}.
\end{align}
Rearranging the terms, we arrive at the approximate closed-form solution:
\begin{align}
    \mathbf{u}^{\star} \approx \vtheta^{(t)} - (\lambfsd \mathbf{G} + \lambwsd \mathbf{I})^{-1} \grad_{\vtheta} \cost_{\mathcal{B}}(\vtheta^{(t)}).
\end{align}
\section{Optimization Algorithms Derived from the Proximal Objective}
\label{app:opt-alg-derivations}
In this section, we show how several classical optimization algorithms can be derived from the proximal objective introduced in Section~\ref{sec:connection}. Given a mini-batch $\batch$, recall that we consider the following proximal objective for optimizing the model parameters $\boldtheta$:
\begin{align}
    \boldtheta^{(t+1)} &= \argmin_{\boldtheta \in \mathbb{R}^m} \left\{ \cost_{\batch}(\boldtheta) + \lamfsd \mathbb{E}_{\tilde{\boldx}} [\disf(\boldtheta, \vtheta^{(t)}, \tilde{\mathbf{x}})] + \frac{\lamwsd}{2} || \boldtheta - \boldtheta^{(t)} ||_2^2 \right\}.
\end{align}

\begin{table}[h]
    \centering
    \footnotesize
    \begin{tabular}{@{}ccccc@{}}
    \toprule
    \centering
    \textbf{Method} & \textbf{\begin{tabular}[c]{@{}c@{}}Loss Term Approx.\\ $\cost_{\batch}(\boldu)$\end{tabular}} & \textbf{\begin{tabular}[c]{@{}c@{}}FSD Term Approx.\\ $ \disf(\mathbf{u}, \vtheta, \mathbf{x})$\end{tabular}} & \textbf{\begin{tabular}[c]{@{}c@{}}FSD Term\\ Choice of $\disf$\end{tabular}} & \textbf{\begin{tabular}[c]{@{}c@{}}WSD Term\\ $\frac{1}{2} || \mathbf{u} - \boldtheta ||^2$\end{tabular}} \\ \midrule
    \textbf{Gradient Descent}         & $1^{\text{st}}$-order  & - & - & \cmark \\
    \textbf{Newton's Method}          & $2^{\text{nd}}$-order  & - & - & \xmark \\
    \textbf{Damped Newton's Method}   & $2^{\text{nd}}$-order  & - & - & \cmark \\
    \textbf{Natural Gradient}         & $1^{\text{st}}$-order  & $2^{\text{nd}}$-order  & KL & \xmark  \\
    \textbf{Damped Natural Gradient}  & $1^{\text{st}}$-order  & $2^{\text{nd}}$-order  & KL & \cmark  \\
    \textbf{Generalized Gauss-Newton} & $1^{\text{st}}$-order  & $2^{\text{nd}}$-order  & Bregman & \xmark  \\
    \textbf{Damped Generalized Gauss-Newton} & $1^{\text{st}}$-order  & $2^{\text{nd}}$-order  & Bregman & \cmark  \\
    \textbf{Direct Proximal Optimization} & Exact & Exact & Any & \cmark \\
    \bottomrule
    \end{tabular}
\end{table}

\subsection{Gradient Descent}
To derive vanilla gradient descent, we take the first-order Taylor series approximation to the loss term $\cost_{\batch}(\boldtheta)$ about $\boldtheta^{(t)}$, and set $\lamfsd = 0$.
We need to have $\lamwsd > 0$ to ensure that we do not take an infinitely large step.
Thus, each step of gradient descent minimizes an approximation to the proximal objective which we denote by $P^{\text{GD}}_{\boldtheta^{(t)}}(\boldtheta)$:
\begin{align}
    \boldtheta^{(t+1)} &= \argmin_{\boldtheta} P^{\text{GD}}_{\boldtheta^{(t)}}(\boldtheta),
\end{align}
where:
\begin{align}
    P^{\text{GD}}_{\boldtheta^{(t)}}(\boldtheta) = \cost_{\batch}(\boldtheta^{(t)}) + \nabla \cost_{\batch}(\boldtheta^{(t)})^\top (\boldtheta - \boldtheta^{(t)}) + \frac{\lamwsd}{2} || \boldtheta - \boldtheta^{(t)} ||_2^2.
\end{align}
Now, we set the gradient $\nabla_{\boldtheta}P^{\text{GD}}_{\boldtheta^{(t)}}(\boldtheta) = 0$ and solve for $\boldtheta$:
\begin{align}
    \nabla_{\boldtheta} \left( \nabla \cost_{\batch}(\boldtheta^{(t)})^\top (\boldtheta - \boldtheta^{(t)}) + \frac{\lamwsd}{2} || \boldtheta - \boldtheta^{(t)} ||_2^2 \right) &= 0 \\
    \nabla \cost_{\batch}(\boldtheta^{(t)}) + \lamwsd (\boldtheta - \boldtheta^{(t)}) &= 0 \\
    \boldtheta &= \boldtheta^{(t)} - \frac{1}{\lamwsd} \nabla \cost_{\batch}(\boldtheta^{(t)})
\end{align}

\subsection{Newton's Method}

To derive Newton's method, we take the second-order Taylor series approximation to the loss term $\cost_{\batch}(\boldtheta)$ about $\boldtheta^{(t)}$, and set $\lamfsd = 0$.
Below we derive the general form for the \textit{damped} Newton update, which incorporates the WSD term; setting $\lamwsd = 0$ leads to the undamped Newton update as a special case.
Each step of Newton's method minimizes an approximation to the proximal objective which we denote by $P^{\text{Newton}}_{\boldtheta^{(t)}}(\boldtheta)$:
\begin{align}
    \boldtheta^{(t+1)} &= \argmin_{\boldtheta} P^{\text{Newton}}_{\boldtheta^{(t)}}(\boldtheta),
\end{align}
where:
\begin{align}
    P^{\text{Newton}}_{\boldtheta^{(t)}}(\boldtheta) = \cost_{\batch}(\boldtheta^{(t)}) + \nabla \cost_{\batch}(\boldtheta^{(t)})^\top (\boldtheta - \boldtheta^{(t)}) + \frac{1}{2} (\boldtheta - \boldtheta^{(t)})^\top \boldH (\boldtheta - \boldtheta^{(t)}) + \frac{\lamwsd}{2} || \boldtheta - \boldtheta^{(t)} ||_2^2.
    \label{eq:newton-prox}
\end{align}
In Eq.~\ref{eq:newton-prox}, $\boldH = \nabla_{\boldtheta}^2 \cost_{\batch}(\boldtheta)$.
Now, we set the gradient $\nabla_{\boldtheta} P^{\text{Newton}}_{\boldtheta^{(t)}}(\boldtheta) = 0$ and solve for $\boldtheta$:
\begin{align}
    \nabla_{\boldtheta} P^{\text{Newton}}_{\boldtheta^{(t)}}(\boldtheta)
    &=
    0 \\
    \nabla \cost_{\batch}(\boldtheta^{(t)}) + \boldH \boldtheta - \boldH \boldtheta^{(t)} + \lamwsd \boldtheta - \lamwsd \boldtheta^{(t)}
    &=
    0 \\
    (\boldH + \lamwsd \boldI) (\boldtheta - \boldtheta^{(t)})
    &=
    - \nabla \cost_{\batch}(\boldtheta^{(t)}) \\
    \boldtheta &= \boldtheta^{(t)} - (\boldH + \lamwsd \boldI)^{-1} \nabla \cost_{\batch}(\boldtheta^{(t)})
\end{align}

\subsection{Generalized Gauss-Newton Method}
Next, we consider taking the first-order Taylor series approximation to the loss term, and an arbitrary function-space discrepancy $\disf$ approximated by its second-order Taylor series expansion.
We will use the notation $\boldz = f(\boldx, \boldtheta)$ to denote the outputs of a neural network with parameters $\boldtheta$ on inputs $\boldx$.
We have the following proximal objective:
\begin{align}
    P_{\boldtheta^{(t)}}(\boldtheta) = \cost_{\batch}(\boldtheta^{(t)}) + \nabla \cost_{\batch}(\boldtheta^{(t)})^\top (\boldtheta - \boldtheta^{(t)}) + \lamfsd \mathbb{E}_{\boldxtil}[\disf(\boldtheta, \boldtheta^{(t)}, \boldxtil))] + \frac{\lamwsd}{2} || \boldtheta - \boldtheta^{(t)} ||_2^2
\end{align}
Taking the second-order approximation to $\disf$ (with respect to $\boldtheta$), we have:
\begin{align}
    \disf(\boldtheta, \boldtheta^{(t)})
    \approx
    \underbrace{\disf(\boldtheta^{(t)}, \boldtheta^{(t)})}_{=0}
    +
    \underbrace{\nabla_{\boldtheta} \disf(\boldtheta, \boldtheta^{(t)})|_{\boldtheta = \boldtheta^{(t)}}^\top (\boldtheta - \boldtheta^{(t)})}_{=0}
    +
    \frac{1}{2} (\boldtheta - \boldtheta^{(t)})^\top \nabla_{\boldtheta}^2 \disf(\boldtheta, \boldtheta^{(t)})|_{\boldtheta = \boldtheta^{(t)}} (\boldtheta - \boldtheta^{(t)}).
\end{align}
Using the chain rule, we can derive $\nabla_{\boldtheta}^2 \disf(\boldtheta, \boldtheta^{(t)})|_{\boldtheta = \boldtheta^{(t)}}$ as follows:
\begin{align}
    \nabla_{\boldtheta}^2 \disf(\boldtheta, \boldtheta^{(t)})|_{\boldtheta = \boldtheta^{(t)}}
    =
    \underbrace{\left( \frac{\partial \boldz}{\partial \boldtheta} \right)^\top}_{\boldJ_{\boldz \boldtheta}^\top}
    \underbrace{\left( \frac{\partial^2 \rho}{\partial \boldz^2} \right)}_{\boldH_{\rho}}
    \underbrace{\left( \frac{\partial \boldz}{\partial \boldtheta} \right)}_{\boldJ_{\boldz \boldtheta}}
    =
    \boldJ_{\boldz \boldtheta}^\top \nabla_{\boldz}^2 \rho(\boldz, \boldz^{(t)}) \boldJ_{\boldz \boldtheta}
    =
    \boldJ_{\boldz \boldtheta}^\top \boldH_{\rho} \boldJ_{\boldz \boldtheta}
\end{align}
Letting $\boldG$ denote the expectation of the Hessian of the FSD function, $\boldG = \mathbb{E}_{\boldxtil \sim \cD}[\nabla_{\boldtheta}^2 \disf(\boldtheta, \boldtheta^{(t)}, \tilde{\mathbf{x}})]$, we have:
\begin{align}
    P_{\boldtheta^{(t)}}(\boldtheta) = \cost_{\batch}(\boldtheta^{(t)}) + \nabla \cost_{\batch}(\boldtheta^{(t)})^\top (\boldtheta - \boldtheta^{(t)}) + \frac{\lamfsd}{2} (\boldtheta - \boldtheta^{(t)})^\top \boldG (\boldtheta - \boldtheta^{(t)}) + \frac{\lamwsd}{2} || \boldtheta - \boldtheta^{(t)} ||_2^2
\end{align}
If we set the gradient $\nabla_{\boldtheta} P_{\boldtheta^{(t)}}(\boldtheta) = 0$ and solve for $\boldtheta$, we obtain:
\begin{align}
    \boldtheta = \boldtheta^{(t)} - (\lamfsd \boldG + \lamwsd \boldI)^{-1} \nabla \cost_{\batch}(\boldtheta^{(t)})
\end{align}

\section{Optimizing the 1-Step Meta-Objective Recovers Classic Methods}
\label{app:precond-opt}

In this section, we show that when we approximate the loss term and FSD term of our meta-objective $\mathcal{Q}(\boldphi)$ and solve for the analytic $\argmin_{\boldphi} \mathcal{Q}(\boldphi)$, we recover the preconditioners corresponding to classic first- and second-order algorithms, including gradient descent, Gauss-Newton, Generalized Gauss-Newton, and natural gradient.
The preconditioners used by each of these methods are shown in Table~\ref{table:second-order-methods}.
These results parallel those for directly optimizing the parameters $\boldtheta$ using the proximal objective, but optimizing for the meta-parameters $\boldphi = \boldP$.

Throughout this exposition, we use the notation:
\begin{align}
    u(\boldtheta, \boldphi, \batch) = u(\boldtheta, \boldP, \batch) = \boldtheta - \boldP \nabla_{\boldtheta} \cost_{\batch}(\boldtheta) = \boldtheta - \boldP \boldg
\end{align}
where $\boldg = \nabla_{\boldtheta} \cost_{\batch}(\boldtheta)$.
Note that $\boldg$ implicitly depends on the data $\batch$; we use this shorthand to simplify the exposition.
% 
% We have the proximal meta-objective:
% \begin{align}
%     \boldP^\star = \argmin_{\boldP} \left\{ \mathbb{E}_{\batch \sim \cD} \left[ \cost_{\batch}(u(\boldtheta, \boldP, \batch))
%     +
%     \lamfsd \mathbb{E}_{\boldxtil}[\disf(u(\boldtheta, \boldP, \batch), \boldtheta, \boldxtil)]
%     +
%     \frac{\lamwsd}{2} || u(\boldtheta, \boldP, \batch) - \boldtheta ||_2^2 \right] \right\}
% \end{align}

% \begin{theorem*}
%     Assume that the loss is linear and the FSD term is quadratic.
%     Denote the gradient on a mini-batch by $\boldg = \grad_{\boldtheta} \cost_{\batch}(\boldtheta)$, and assume that the second moment matrix $\mathbb{E}_{\batch \sim \cD} \left[ \boldg \boldg^\top \right]$ is non-singular.
%     Then the optimal preconditioning matrix $\pre^{\star}$ minimizing the proximal meta-objective (Eq.~\ref{eq:meta-objective-exact}) is:
%     \begin{align}
%         \pre^{\star} = (\lamfsd \mathbf{G} + \lamwsd \mathbf{I})^{-1}.
%     \end{align}
%     where $\boldG$ is the Hessian of the FSD function.
%     \label{thm:optimal-precond3}
% \end{theorem*}
\begin{theorem*}
    Consider an approximation $\hat{\mathcal{Q}}(\pre)$ to the meta-objective (Eq.~\ref{eq:meta-objective-exact}) where the loss term is linearized around the current weights $\boldtheta$ and the FSD term is replaced by its second-order approximation around $\boldtheta$:
    \begin{align}
    \hat{\mathcal{Q}}(\pre) = & \mathbb{E}_{\mathcal{B} \sim \mathcal{D}} \Big[ \nabla_{\boldtheta} \cost_{\mathcal{B}}(\boldtheta)^{\top} (u(\boldtheta, \pre, \mathcal{B}) - \boldtheta) \nonumber \\
    &\qquad \quad +
    \lambda_{\text{FSD}} (u(\boldtheta, \pre, \mathcal{B}) - \boldtheta)^{\top} \mathbf{G} (u(\boldtheta, \pre, \mathcal{B}) - \boldtheta)
    +
    \frac{\lambda_{\text{WSD}}}{2} \|u(\boldtheta, \pre, \mathcal{B}) -  \vtheta\|^2 \Big].
    \end{align}
    Denote the gradient on a mini-batch as $\boldg = \grad_{\boldtheta} \cost_{\batch}(\boldtheta)$, and assume that the second moment matrix $\mathbb{E}_{\batch \sim \mathcal{D}} \left[ \boldg \boldg^\top \right]$ is non-singular.
    Then, the preconditioning matrix which minimizes $\hat{\mathcal{Q}}$ is given by $\pre^{\star} = (\lamfsd \mathbf{G} + \lamwsd \mathbf{I})^{-1}$, where $\boldG$ denotes the Hessian of the FSD evaluated at $\boldtheta$.
    % \label{thm:optimal-precond}
\end{theorem*}
\begin{proof}
We have the following meta-objective (where the loss term is expanded using its first-order Taylor series approximation):
\begin{align}
    \boldP^\star
    &=
    \argmin_{\boldP} \Big \{ \mathbb{E}_{\batch \sim \cD} \Big[ \cost_{\batch}(\boldtheta) + \underbrace{\nabla_{\boldtheta} \cost_{\batch}(\boldtheta)^\top}_{\boldg} ((\boldtheta - \boldP \boldg) - \boldtheta) \\
    &\qquad\qquad\qquad + \lamfsd \mathbb{E}_{(\boldxtil, \cdot) \sim \cD} \left[\disf(u(\boldtheta, \boldP, \batch), \boldtheta, \boldxtil)\right] + \frac{\lamwsd}{2} || \boldtheta - \boldP \boldg - \boldtheta ||^2 \Big] \Big \}
\end{align}
Let us first derive the second-order Taylor series approximation to the discrepancy term $\disf(u(\boldtheta, \boldphi, \batch), \boldtheta, \boldxtil)$.
To simplify notation, let $d(\boldtheta') \equiv \disf(\boldtheta', \boldtheta, \boldxtil)$ where $\boldtheta$ and the data $\boldxtil$ are implicit.
Then, the second-order expansion of $d$ about $\boldtheta$ is:
\begin{align}
    d(\boldtheta')
    &\approx
    \underbrace{d(\boldtheta)}_{=0} + \underbrace{\nabla d(\boldtheta)^\top (\boldtheta' - \boldtheta)}_{=0} + \frac{1}{2} (\boldtheta' - \boldtheta)^\top \nabla_{\boldtheta}^2 d(\boldtheta) (\boldtheta' - \boldtheta)
\end{align}
We have $d(\boldtheta) = \disf(\boldtheta, \boldtheta, \boldxtil) = \rho(f(\boldxtil, \boldtheta), f(\boldxtil, \boldtheta)) = 0$, and $\nabla d(\boldtheta) = 0$ because $\boldtheta$ is a minimum of $d$, so the first two terms are 0.
Denote the network output on example $\boldxtil$ by $\mathbf{y} = f(\boldxtil, \boldtheta)$.
To compute $\nabla_{\boldtheta}^2 d(\boldtheta)$, we have:
\begin{align}
    \nabla_{\boldtheta}^2 d(\boldtheta) &= \underbrace{\left(\frac{\partial \mathbf{y}}{\partial \boldtheta}\right)^\top}_{\boldJ_{\mathbf{y} \boldtheta}^\top}
    \underbrace{\left(\frac{\partial^2 \rho}{\partial \mathbf{y}^2}\right)}_{\boldH_{\rho}}
    \underbrace{\left(\frac{\partial \mathbf{y}}{\partial \boldtheta} \right)}_{\boldJ_{\mathbf{y} \boldtheta}}
    =
    \boldJ_{\mathbf{y} \boldtheta}^\top \boldH_{\rho} \boldJ_{\mathbf{y} \boldtheta}
\end{align}
To simplify notation, we let $\boldG$ denote the expectation of the Hessian of the FSD function, $\boldG \equiv \mathbb{E}_{\boldxtil \sim \cD}[\nabla_{\boldtheta}^2 d(\boldtheta)$].
Plugging this into the meta-objective, we have:
\begin{align}
    \boldP^\star &= \argmin_{\boldP} \left\{ \mathbb{E}_{\batch \sim \cD} \left[ \cost_{\batch}(\boldtheta) - \boldg^\top \boldP \boldg + \frac{\lamfsd}{2} \boldg^\top \boldP^\top \boldG \boldP \boldg + \frac{\lamwsd}{2} \boldg^\top \boldP^\top \boldP \boldg \right] \right\}
\end{align}
Next, we take the gradient with respect to $\boldP$ and set it to $0$ to solve for $\boldP$:
\begin{align}
    \nabla_{\boldP} \left( \mathbb{E}_{\batch \sim \cD} \left[ \cost_{\batch}(\boldtheta) - \boldg^\top \boldP \boldg + \frac{\lamfsd}{2} \boldg^\top \boldP^\top \boldG \boldP \boldg + \frac{\lamwsd}{2} \boldg^\top \boldP^\top \boldP \boldg \right] \right) &= 0 \\
    \mathbb{E} \left[ - \boldg \boldg^\top + \frac{\lamfsd}{2} (\boldG \boldP \boldg \boldg^\top + \boldG^\top \boldP \boldg \boldg^\top) + \lamwsd \boldP \boldg \boldg^\top \right] &= 0 \\
    \mathbb{E} \left[ - \boldg \boldg^\top + \lamfsd \boldG \boldP \boldg \boldg^\top + \lamwsd \boldP \boldg \boldg^\top \right] &= 0 \\
    (\lamfsd \boldF \boldP + \lamwsd \boldP - \boldI) \mathbb{E} \left[ \boldg \boldg^\top \right] &= 0
\end{align}
Assuming that the second moment matrix $\mathbb{E}_{\batch \sim \cD} [ \boldg \boldg^\top ]$ is non-singular, we have:
\begin{align}
    \lamfsd \boldG \boldP + \lamwsd \boldP - \boldI &= 0 \\
    (\lamfsd \boldG + \lamwsd \boldI) \boldP &= \boldI \\
    \boldP &= (\lamfsd \boldG + \lamwsd \boldI)^{-1}
\end{align}
Thus, $\boldP^\star = (\lamfsd \boldG + \lamwsd \boldI)^{-1}$.
\end{proof}

\section{Optimizing the 1-Step Meta-Objective Yields the KFAC Update}
\label{app:meta-opt-kfac}

In the previous section, we derived the optimal solutions to various approximate proximal objectives, optimizing over an unconstrained preconditioner $\boldP$.
Here, we consider an analogous setup for a structured preconditioner, where $\boldP$ is block-diagonal with the $\ell^{\text{th}}$ block (corresponding to the $\ell^{\text{th}}$ layer in the network) given by $\boldA_\ell \otimes \boldB_\ell$.

\subsection{KFAC Assumptions}
\label{app:kfac-assumptions}

This exposition of the KFAC assumptions is based on~\cite{GrosseNNTDChapter4}.
First we briefly introduce some notation that will be used in the assumptions.
Suppose a layer of a neural network has weights $\mathbf{W}_{\ell}$ and that its input is the activation vector from the previous layer, $\bolda_{\ell-1}$.
Then the output of the layer is a linear transformation of the input, followed by a nonlinear activation function $\sigma$:
\begin{align}
    \bolds_{\ell} &= \mathbf{W}_{\ell} \bolda_{\ell-1} \\
    \bolda_{\ell} &= \sigma(\bolds_{\ell})
\end{align}
where $\bolds_{\ell}$ are the pre-activations.
In the backward pass, we have the following activation and weight derivatives:
\begin{align}
    \mathcal{D} \bolda_{\ell} &= \boldW^\top \mathcal{D} \bolds_{\ell+1} \\
    \mathcal{D} \bolds_{\ell} &= \mathcal{D} \bolda_{\ell} \odot \sigma'(\bolds_{\ell}) \\
    \mathcal{D} \boldW_{\ell} &= \mathcal{D} \bolds_{\ell} \bolda_{\ell-1}^\top
\end{align}

KFAC~\citep{martens2015optimizing} makes the following assumptions on the neural network being optimized:
\begin{itemize}
    \item The layers are independent, such that that the pseudo-derivatives $d w_i$ and $d w_j$ are uncorrelated when $w_i$ and $w_j$ belong to different layers.
    If this assumption is satisfied, then the Fisher information matrix $\boldF$ will be block-diagonal.
    \item The activations $\{\bolda_{\ell} \}$ are independent of the pre-activation pseudo-gradients $\{ \mathcal{D} \bolds_{\ell} \}$. If $\bolda_{\ell-1}$ is independent of $\mathcal{D} \bolds_{\ell}$, then the $\ell^{\text{th}}$ block of the Fisher is:
    \begin{align}
        \hat{\boldG}_{\ell}
        &=
        \boldA^{\text{KFAC}}_{\ell} \otimes \boldB^{\text{KFAC}}_{\ell}
    \end{align}
    where:
    \begin{align}
        \boldA^{\text{KFAC}}_{\ell} &= \mathbb{E}[\bolda_{\ell-1} \bolda_{\ell-1}^\top] \\
        \boldB^{\text{KFAC}}_{\ell} &= \mathbb{E}[\mathcal{D} \bolds_{\ell} \mathcal{D} \bolds_{\ell}^\top]
    \end{align}
\end{itemize}

\subsection{Proof of Corollary~\ref{thm:kfac}}

\begin{corollary*}
Suppose that (1) the assumptions for Theorem~\ref{thm:optimal-precond} are satisfied, (2) the FSD term measures the KL divergence, and (3) $\lamwsd = 0$ and $\lamfsd = 1$.
Moreover, suppose that the parameters $\vtheta$ satisfy the KFAC assumptions listed in Appendix~\ref{app:meta-opt-kfac}.
Then, the optimal solution to the approximate meta-objective recovers the KFAC update, which can be represented using the structured preconditioner in Eq.~\ref{eq:ekfac-param}.
\end{corollary*}
\begin{proof}
If the KFAC assumptions are satisfied, then $\boldF^{-1}$ is block-diagonal with the block corresponding to the $\ell^{\text{th}}$ layer expressed as $\boldG_{\ell} = \boldA_{\ell}^{-1} \otimes \boldB_{\ell}^{-1}$ where $\boldA_{\ell} = \mathbb{E}[\bar{\bolda}_{\ell-1} \bar{\bolda}_{\ell-1}^\top]$ and $\boldB_{\ell} = \mathbb{E}[\mathcal{D} \bolds_{\ell} \mathcal{D} \bolds_{\ell}^\top]$.
By Theorem~\ref{thm:optimal-precond}, the optimal solution to the approximate meta-objective is $\boldP^\star = \boldF^{-1}$ when $\lamwsd = 0$ and $\lamfsd = 1$. Hence $\boldP^\star$ is this block-diagonal matrix, which can be expressed using our structured parameterization (Eq.~\ref{eq:ekfac-param}).
Thus, APO recovers the KFAC update.
\end{proof}

\section{Ablations}
\label{app:labmda-ablation}

\subsection{Ablation Over \texorpdfstring{$\lamwsd$}{lambwsd} and \texorpdfstring{$\lamfsd$}{lambfsd}}

Figure~\ref{fig:meta-fsd} provides an ablation over $\lamwsd$ and $\lamfsd$ for training AlexNet on CIFAR-10 using APO to adapt the preconditioning matrix. We kept the other proximity weight fixed when performing the experiments. We found that APO-Precond is robust in various ranges of proximity weights $\lambda$.

\begin{figure}[H]
\centering
\includegraphics[width=0.40\linewidth]{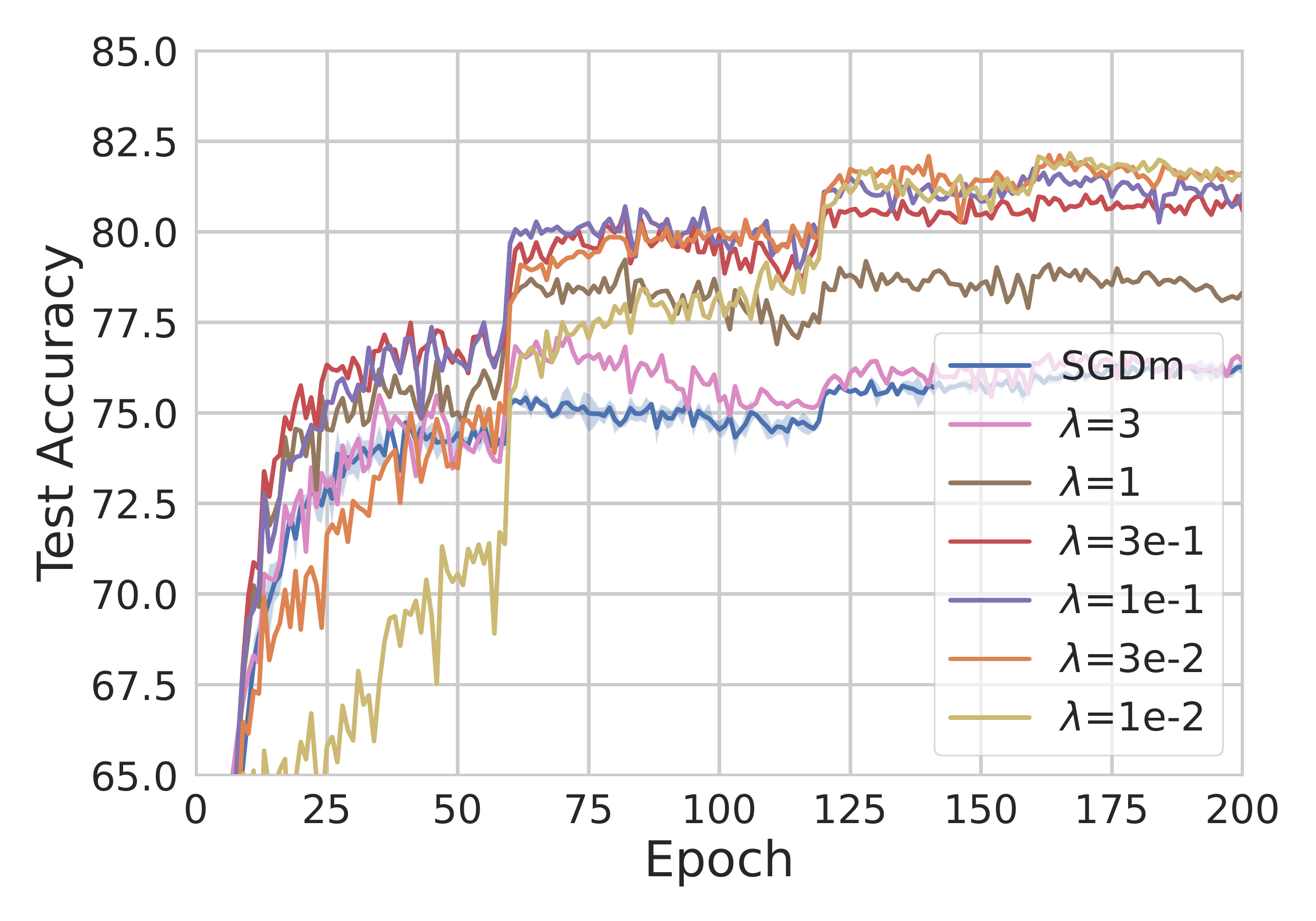}
\includegraphics[width=0.40\linewidth]{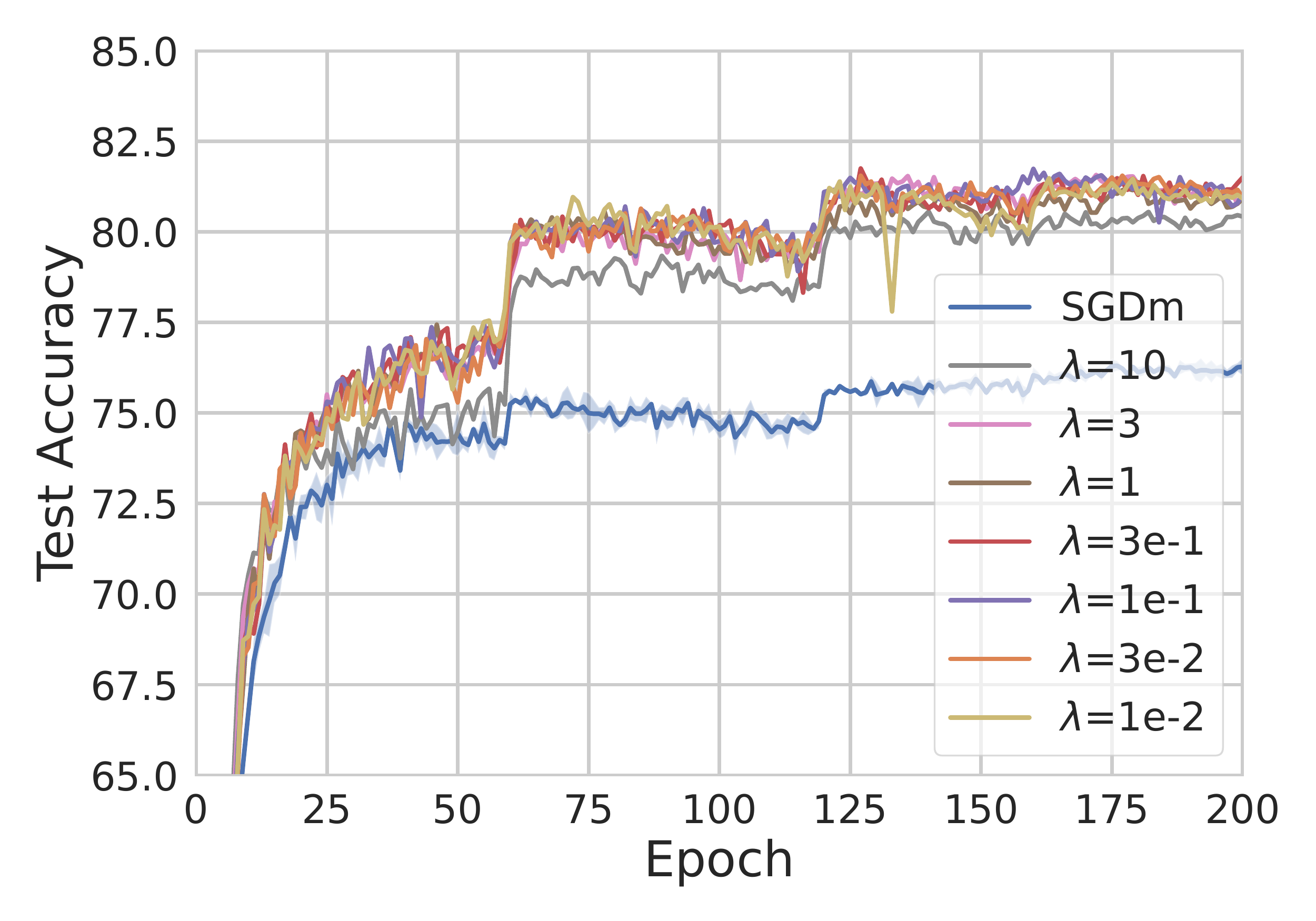}
\vspace{-2mm}
\caption{Ablation over (left) $\lamwsd$ and (right) $\lamfsd$ for training AlexNet on CIFAR-10, using APO-Precond}
\label{fig:meta-fsd}
\end{figure}

\subsection{Robustness to Initial Learning Rate and Meta-Update Interval}
\label{app:robustmeta}
We performed an ablation to evaluate how well APO performs with different meta-update intervals (i.e.~when making more or less frequent updates to the learning rate and preconditioning matrix during training). We used APO to tune the learning rate, and experimented with performing meta-updates once every 10, 20, 50, and 100 base optimization iterations. We trained a ResNet32 model on CIFAR-10, and used SGDm as the base optimizer. The results are shown in Figure~\ref{fig:meta-interval}. We found that APO is robust to the meta-update interval, performing almost identically with respect to training loss, test accuracy, and the adaptive learning rate schedule, for each meta-update interval.

\begin{figure}[h]
    \centering
    \begin{tabular}{ccc}
    \includegraphics[width=0.31\linewidth]{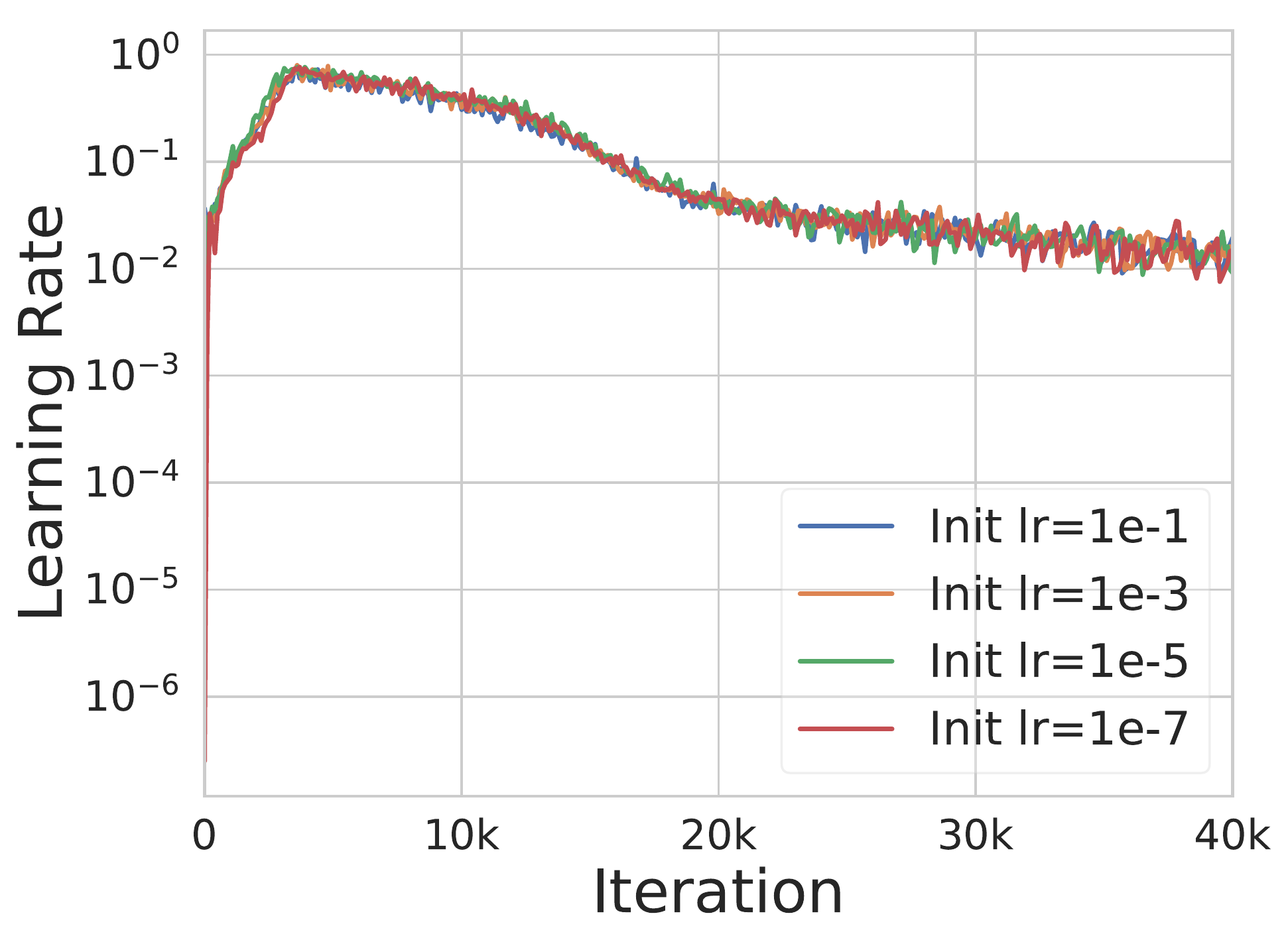} &
    \includegraphics[width=0.31\linewidth]{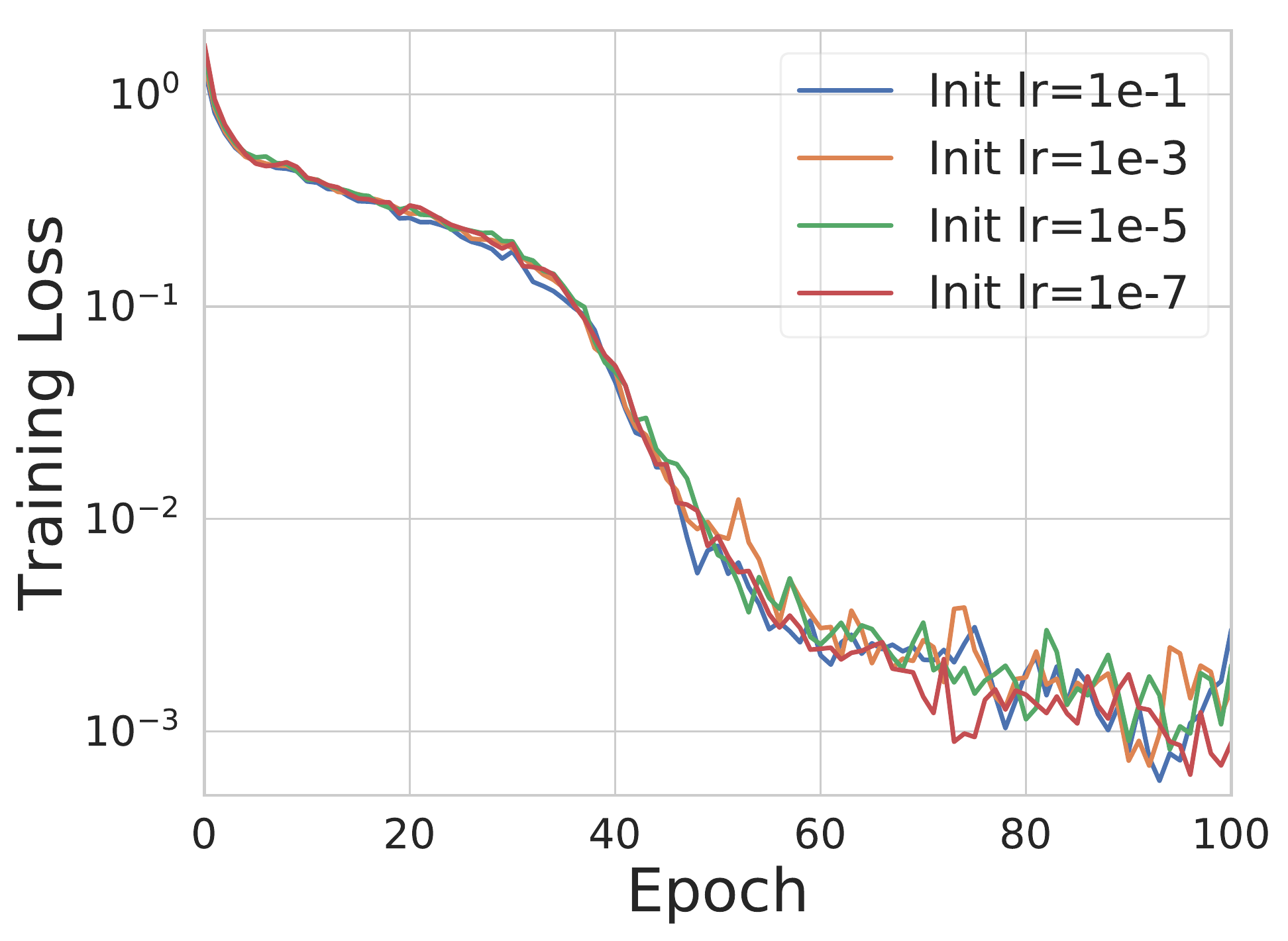} &
    \includegraphics[width=0.31\linewidth]{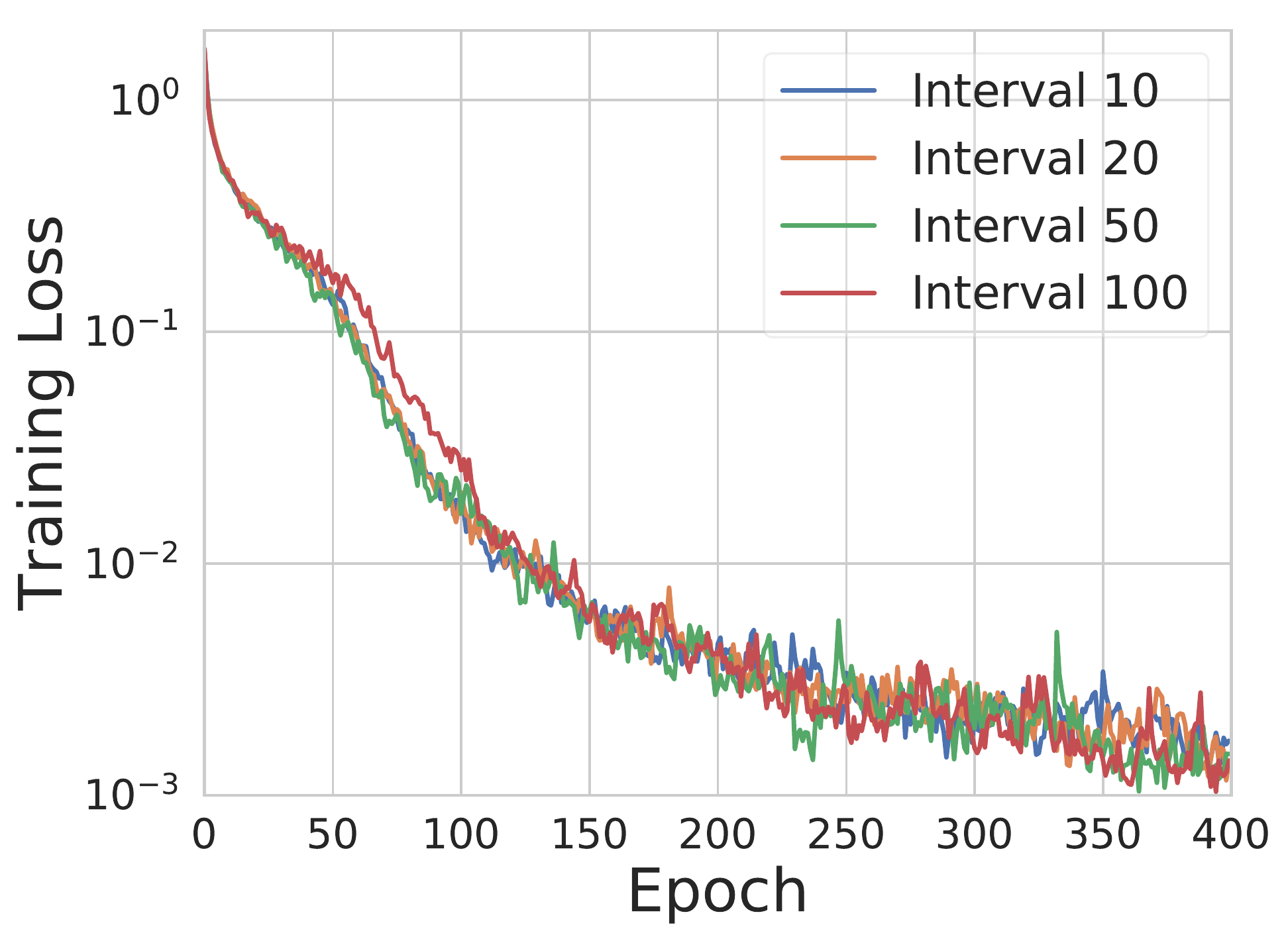} \\
    (a) &
    (b) &
    (c)
    \end{tabular}
    \vspace{-2mm}
    \caption{Robustness to the initial learning rate and meta-update interval. Figures (a) and (b) show that APO achieves almost identical learning rate schedules and training losses for initial learning rates spanning 6 orders of magnitude. Figure (c) shows the training loss for ResNet32 on CIFAR-10 using APO with various meta-update intervals---we observe that APO performs similarly using intervals from 10 to 100, corresponding to computation times $1.3\times$ to $1.03\times$ that of the base optimizer.
    }
    \label{fig:init-lr}
\end{figure}

\begin{figure}[H]
    \centering
    \includegraphics[width=0.31\linewidth]{figures/resnet32_meta_interval/resnet32_meta_interval_train_loss.pdf}
    \includegraphics[width=0.31\linewidth]{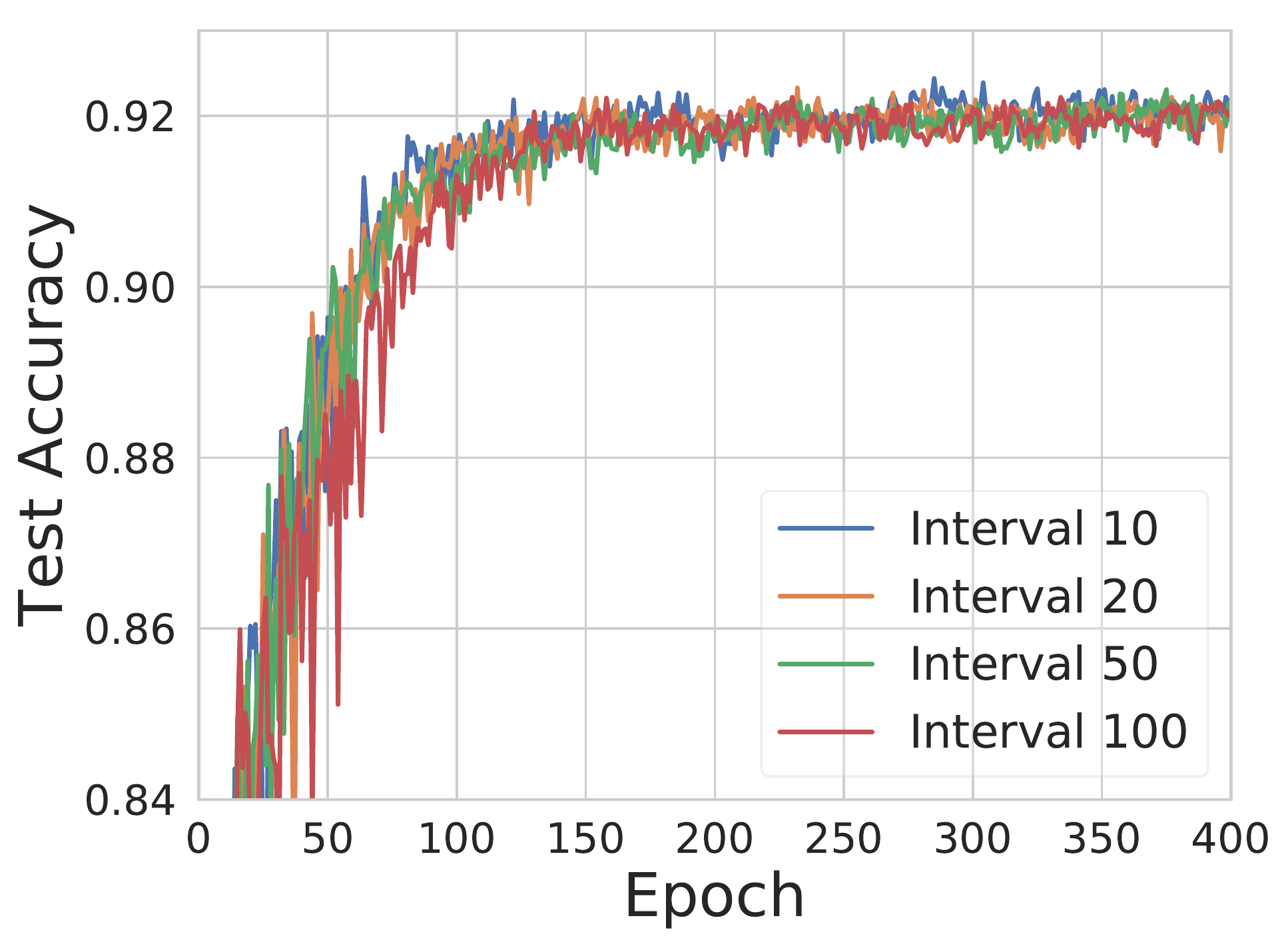}
    \includegraphics[width=0.31\linewidth]{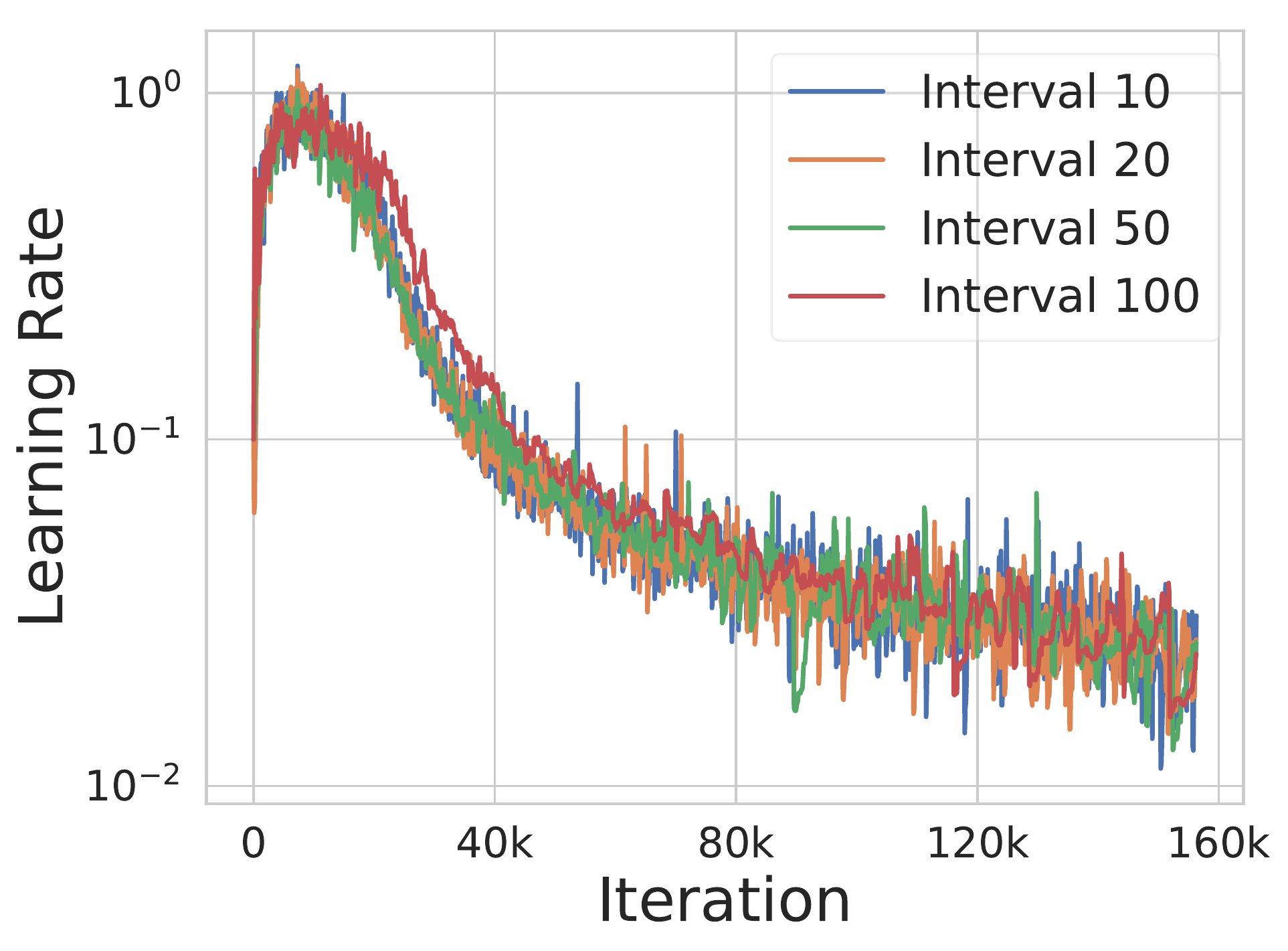}
    \vspace{-2mm}
    \caption{Robustness to the meta update interval. CIFAR-10, SGDm-APO.}
    \label{fig:meta-interval}
\end{figure}

To further examine how meta-update intervals affect the preconditioning adaptation, we trained AlexNet model on CIFAR-10. We used APO to learn the preconditioning matrix and experimented with performing meta-updates once every 10, 20, 50, and 100 base optimization iterations. The results are shown in Figure~\ref{fig:precond-meta-interval}. As in the learning rate adaptation problem, we found that APO performs similarly using intervals from 10 to 100, achieving higher test accuracy than the baseline method.
\begin{figure}[H]
    \centering
    \includegraphics[width=0.41\linewidth]{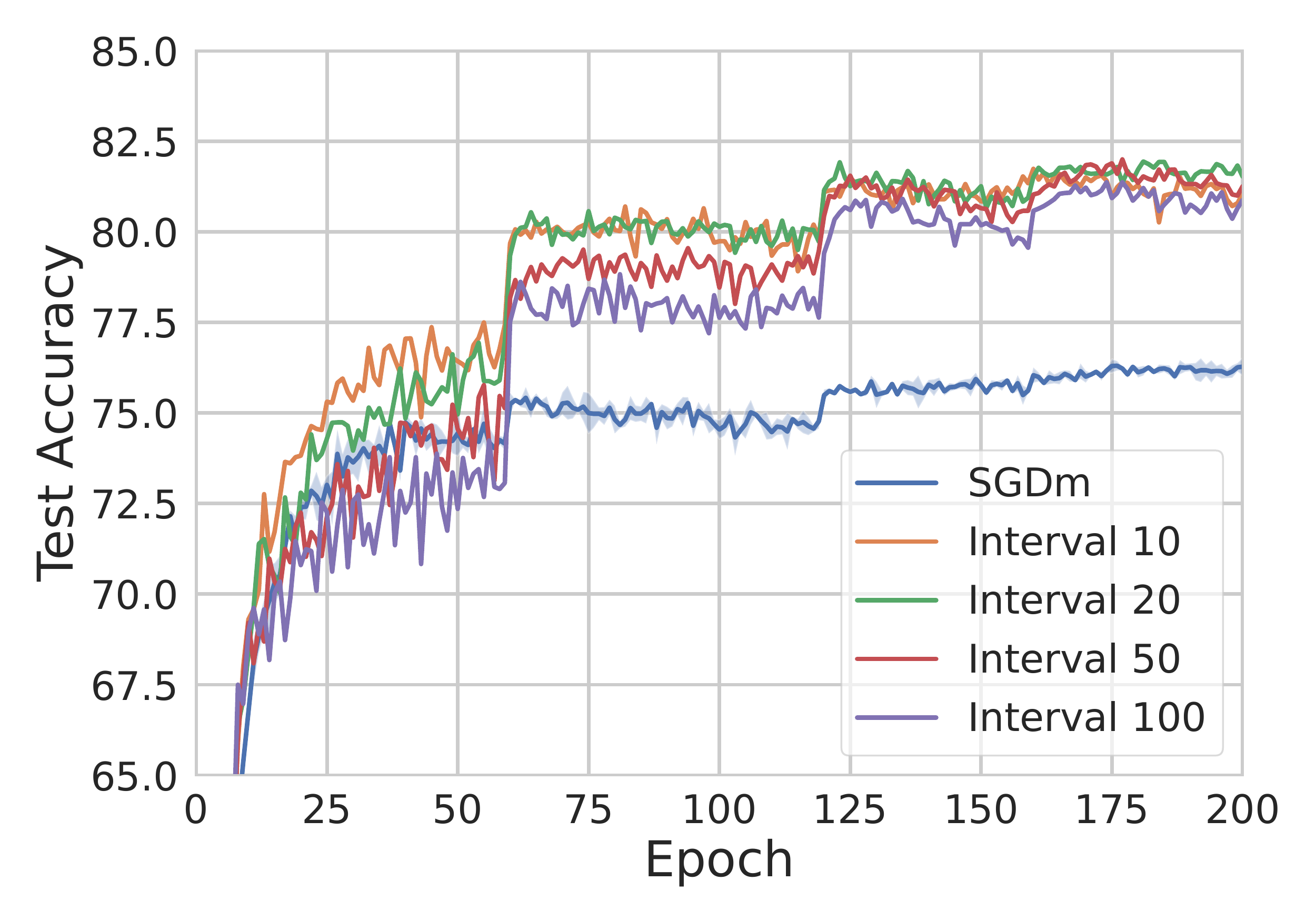}
    \vspace{-2mm}
    \caption{Robustness to the meta update interval on AlexNet. CIFAR-10, APO-Precond.}
    \label{fig:precond-meta-interval}
\end{figure}

\end{document}